\documentclass[lettersize, journal]{IEEEtran}

\usepackage[utf8]{inputenc}
\usepackage[T1]{fontenc}
\usepackage{graphicx}
\usepackage{grffile}
\usepackage{longtable}
\usepackage{wrapfig}
\usepackage{rotating}
\usepackage[normalem]{ulem}
\usepackage{amsmath}
\usepackage{array}
\usepackage{authblk}\usepackage[pagebackref,breaklinks,colorlinks]{hyperref}

\usepackage{dsfont} 

\usepackage{setspace}

\usepackage{bm} 
\usepackage{bbm} 

\usepackage{textcomp}
\usepackage{amssymb}
\usepackage{capt-of}

\usepackage{amsthm}

\usepackage{mathtools}

\usepackage{stmaryrd}

\usepackage{arydshln}

\usepackage{multirow}

\usepackage{algorithm}

\usepackage{algpseudocode}

\usepackage{booktabs}

\usepackage[numbers]{natbib}

\usepackage{subfig}

\newtheorem{theorem}{Theorem}
\newtheorem{corollary}{Corollary}
\newtheorem{lemma}{Lemma}

\newtheorem{assumption}{Assumption}
\newtheorem{remark}{Remark}

\newcommand{\R}{\mathbb{R}}
\newcommand{\Ex}[1]{\mathbb{E}\left[#1\right]}

\DeclarePairedDelimiterX{\norm}[1]{\lVert}{\rVert}{#1}
\DeclarePairedDelimiterX{\abs}[1]{\lvert}{\rvert}{#1}

\usepackage[dvipsnames]{xcolor}

\newcommand{\lbd}{$\lambda$\xspace}

\newcommand{\lb}{\lambda}
\newcommand{\lbhat}{\widehat{\lambda}}

\newcommand{\lbcnf}{\lambda^\mathrm{cnf}}
\newcommand{\lbcnfbar}{\bar{\lambda}^\mathrm{cnf}}
\newcommand{\lbcnfm}{\lbcnf_{-}}
\newcommand{\lbcnfp}{\lbcnf_{+}}
\newcommand{\lbcnfo}{\lbcnf_{*}}
\newcommand{\lbcls}{\lambda^\mathrm{cls}}
\newcommand{\lbclsbar}{\bar{\lambda}^\mathrm{cls}}

\newcommand{\lbclsp}{\lbcls_{+}}
\newcommand{\lbloc}{\lambda^\mathrm{loc}}
\newcommand{\lblocbar}{\bar{\lambda}^\mathrm{loc}}

\newcommand{\lblocp}{\lbloc_{+}}
\newcommand{\lbbul}{\lambda^{\bullet}}
\newcommand{\lbbulbar}{\bar{\lambda}^\bullet}

\newcommand{\lbbulp}{\lbbul_{+}}
\newcommand{\lbbulo}{\lbbul_{*}}

\newcommand{\Lb}{\Lambda}
\newcommand{\Lbcnf}{\Lb^\mathrm{cnf}}
\newcommand{\Lbcls}{\Lb^\mathrm{cls}}
\newcommand{\Lbloc}{\Lb^\mathrm{loc}}
\newcommand{\Lbbul}{\Lb^{\bullet}}

\newcommand{\Lcnf}{L^\mathrm{cnf}}
\newcommand{\Lcls}{L^\mathrm{cls}}
\newcommand{\Lloc}{L^\mathrm{loc}}
\newcommand{\Lbul}{L^\bullet}

\newcommand{\Rcnf}{R^\mathrm{cnf}}
\newcommand{\Rtcnf}{\Tilde{R}^\mathrm{cnf}}
\newcommand{\Rbul}{R^\bullet}
\newcommand{\Rloc}{R^\mathrm{loc}}
\newcommand{\Rcls}{R^\mathrm{cls}}

\newcommand{\alphatot}{\alpha^\mathrm{tot}}
\newcommand{\alphacnf}{\alpha^\mathrm{cnf}}
\newcommand{\alphacls}{\alpha^\mathrm{cls}}
\newcommand{\alphaloc}{\alpha^\mathrm{loc}}
\newcommand{\alphabul}{\alpha^\bullet}

\newcommand{\Gacnf}{\Gamma^\mathrm{cnf}_{\lambda^\mathrm{cnf}}}
\newcommand{\Gacls}{\Gamma^\mathrm{cls}_{\lambda^\mathrm{cnf},\lambda^\mathrm{cls}}}
\newcommand{\Galoc}{\Gamma^\mathrm{loc}_{\lambda^\mathrm{cnf},\lambda^\mathrm{loc}}}

\newcommand{\Xtest}{X_{\mathrm{test}}}
\newcommand{\Ytest}{Y_{\mathrm{test}}}

\newcommand{\Dtrain}{\mathcal{D}_{\mathrm{train}}}

\newcommand{\bleft}{b^\leftarrow}
\newcommand{\btop}{b^\uparrow}
\newcommand{\bright}{b^\rightarrow}
\newcommand{\bbottom}{b^\downarrow}

\newcommand{\bt}{\tilde{b}}
\newcommand{\btleft}{\bt^\leftarrow}
\newcommand{\bttop}{\bt^\uparrow}
\newcommand{\btright}{\bt^\rightarrow}
\newcommand{\btbottom}{\bt^\downarrow}

\newcommand{\bhat}{\hat{b}}
\newcommand{\bhleft}{\bhat^\leftarrow}
\newcommand{\bhtop}{\bhat^\uparrow}
\newcommand{\bhright}{\bhat^\rightarrow}
\newcommand{\bhbottom}{\bhat^\downarrow}

\newcolumntype{R}[1]{>\centering m{#1}}

\RequirePackage{xspace} 
\RequirePackage[nolist,nohyperlinks]{acronym}
\begin{acronym}
    \acro{SeqCRC}{{Sequential Conformal Risk Control}}
    \acro{UQ}{{Uncertainty Quantification}}
    \acro{CP}{{Conformal Prediction}}
    \acro{CRC}{{Conformal Risk Control}}
    \acro{ML}{{Machine Learning}}
    \acro{UQ}{{Uncertainty Quantification}}
    \acro{i.i.d.}{{independent and identically distributed}}
    \acro{CQR}{{Conformalized Quantile Regression}}
    \acro{RCPS}{{Risk-Controlling Prediction Sets}}
    \acro{OD}{{Object Detection}}
    \acro{APS}{{Adaptive Prediction Sets}}
    \acro{LAC}{{Least Ambiguous Set-Valued Classifiers}}
    \acro{BB}{{Bounding Box}}
    \acrodefplural{BB}[BBs]{Bounding Boxes}
\end{acronym}

\newcommand{\seqcrc}{\ac{SeqCRC}\xspace}
\newcommand{\uq}{\ac{UQ}\xspace}
\newcommand{\cp}{\ac{CP}\xspace}
\newcommand{\crc}{\ac{CRC}\xspace}

\newcommand{\od}{\ac{OD}\xspace}

\newcommand{\cpf}{\acf{CP}\xspace}
\newcommand{\crcf}{\acf{CRC}\xspace}
\newcommand{\odf}{\acf{OD}\xspace}

\newtheorem{innercustomthm}{Theorem}
\newenvironment{customthm}[1]
  {\renewcommand\theinnercustomthm{#1}\innercustomthm}
  {\endinnercustomthm}

\def\BibTeX{{\rm B\kern-.05em{\sc i\kern-.025em b}\kern-.08em
    T\kern-.1667em\lower.7ex\hbox{E}\kern-.125emX}}
\usepackage{balance}

\title{Conformal Object Detection by Sequential Risk Control}

\author[1,2]{L\'eo And\'eol}
\author[3]{Luca Mossina}
\author[1]{Adrien Mazoyer}
\author[3,1]{S\'ebastien Gerchinovitz}
\affil[1]{Univ Toulouse, Institut de Math\'ematiques de Toulouse, Toulouse, France} 
\affil[2]{SNCF}
\affil[3]{IRT Saint Exup\'ery}

\begin{document}
\maketitle

\begin{figure*}
    \includegraphics[width=0.95\textwidth]{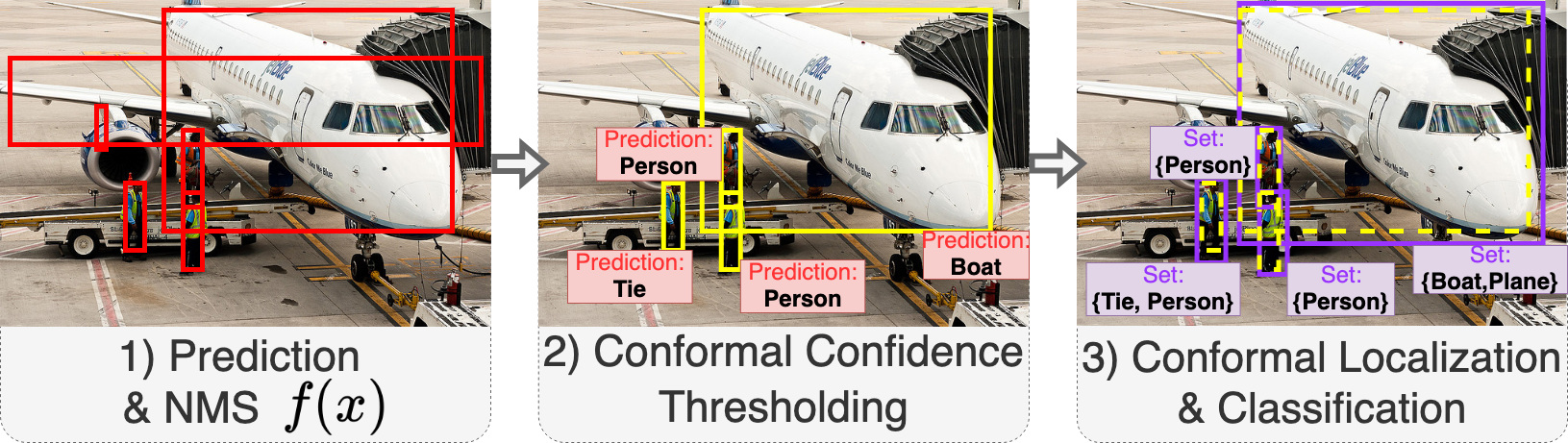}
    \caption{{\bf Conformal Object Detection Pipeline.} 
    Our method corrects the output of a given object detector $f$ with finite-sample probabilistic guarantees at a desired error rate $\alpha$. 
    The figure illustrates the key stages: 
    (1) {\color[HTML]{FF0000}Red} bounding boxes show the initial raw predictions from the object detector $f(x)$ which are then filtered by Non-Maximum Suppression to a smaller set. 
    (2) Conformal Confidence Thresholding applies a statistically grounded \textbf{confidence threshold}, resulting in selected predictions shown with {\color[HTML]{EEEE00}yellow} bounding boxes with a guaranteed risk. 
    (3) Conformal Localization \textbf{corrects} the previous dashed {\color[HTML]{EEEE00}yellow} boxes to build the final statistically valid {\color[HTML]{9933FF}purple} bounding boxes. 
    For each localized object, Conformal Classification constructs a \textbf{prediction set of labels} (e.g., "Set: \{Boat, Plane\}") from the model's predictions, guaranteed to contain the true class with high probability. 
    These steps together control a user-defined risk in localization and classification at a predetermined level $\alphatot$.
    }
    \label{fig:first_fig}
\end{figure*}

\begin{abstract}
    Recent advances in object detectors have led to their adoption for industrial uses.
    However, their deployment in safety-critical applications is hindered by the inherent lack of reliability of neural networks and the complex structure of object detection models. 
    To address these challenges, we turn to Conformal Prediction, a post-hoc predictive uncertainty quantification procedure 
    with statistical guarantees that are valid for any dataset size, without requiring prior knowledge on the model or data distribution.
    Our contribution is manifold. First, we formally define the problem of Conformal Object Detection (COD). We introduce a novel method, \emph{Sequential Conformal Risk Control (SeqCRC)}, that extends the statistical guarantees of
    Conformal Risk Control to two sequential tasks with two parameters, as required in the COD setting.
    Then, we present old and new loss functions and prediction sets suited to applying SeqCRC to different cases and certification requirements.
    Finally, we present a {\emph{conformal toolkit}} for replication and further exploration of our method.
    Using this toolkit, we perform extensive experiments that validate our approach and emphasize trade-offs and other practical consequences.
\end{abstract}
\let\thefootnote\relax\footnotetext{This work has been submitted to the IEEE for possible publication. Copyright may be transferred without notice, after which this version may no longer be accessible.}

\section{Introduction}
\acf{OD} has emerged as a pivotal task in computer vision, which aims at localizing and classifying objects within images or video sequences.
It has had a remarkable impact in diverse fields, ranging from autonomous vehicles~\cite{feng2022autonomous} to medical imaging~\cite{yang2021artificial}.
However, there is an increasing demand to certify predictions of \od models within safety-critical systems~\cite{
Jenn_2020_identifying, mamalet_2021_white, 
Alecu_2022_can, Poretschkin_2023_guideline}. 

In this light, while most \uq methods are Bayesian~\cite{miller2018dropout,Harakeh_2020_bayesod,sharifuzzaman2024bayes} or heuristic~\cite{schubert2021metadetect, meyer2020learning, zhang2024harnessing}, our approach emphasizes distribution-free \uq methods~\cite{Vovk_2005_algorithmic}. 
These methods neither assume prior knowledge of the data distribution nor are restricted to a particular class of predictive models.
They are statistically valid even with a finite number of examples (i.e. non-asymptotic).  
When combining object detection with these methodologies, most notably \ac{CP}, we take an \emph{a posteriori} strategy.
This post-hoc approach is theoretically sound and computationally lightweight, allowing for rapid integration into real-world systems. 
Existing literature on conformal prediction for OD~\cite{deGrancey_2022_object,Andéol_2023_confident, Andéol_2023_conformal} and on tolerance regions for OD~\cite{Li_2022_towards, Blot_2024_efficient} typically targets specific steps of the OD pipeline or individual model families. 
In contrast, our framework covers the \emph{entire detection pipeline} within a single unified approach that applies across models and tasks, while providing statistical guarantees.
It is applicable to virtually any object detectors, and readily deployable in existing systems, ensuring reliability without a comprehensive overhaul of established frameworks, and at very little cost.

After discussing related works and formally defining the setup (Section~\ref{sec:od-uq}), we make the following set of contributions:
\begin{enumerate}
    \item In Section~\ref{sec:cod}, we introduce a novel approach called \emph{Sequential Conformal Risk Control (SeqCRC)}, which enables control over two task-specific losses parameterized by two sequentially selected parameters---a key requirement for general OD tasks. This allows the user to reliably tune a confidence threshold and to compute reliable ``error margins'' both for the localization and classification tasks.
    \item In Section~\ref{sec:app-pred-sets-losses}, we review several state-of-the-art conformal loss functions and introduce new ones, discussing their relevance and the trade-offs they entail. We also examine the crucial choice of metric for matching predictions with ground truths.
    \item In Section~\ref{sec:experiment}, we empirically validate the correctness, efficiency, and broad applicability of our framework, via the first extensive experiments of Conformal \od in the literature.
    \item Finally, we provide a comprehensive, open-source codebase\footnote{
        \url{https://github.com/leoandeol/cods}}---a repository of models and conformal procedures---that facilitates the development of new guarantees, loss functions, and methodologies. 
\end{enumerate}

In the supplemental material, we provide a proof of our theoretical guarantee, more detailed pseudocodes, as well as additional remarks on modeling choices and experiments.

\subsection{Related Works}
\label{sec:related_works}

We now discuss several approaches to UQ for OD, with or without statistical guarantees. The first set of references below is closest to our paper.

\textit{Conformal prediction for OD.} Conformal prediction \cite{Vovk_2005_algorithmic, Angelopoulos_2022_gentle} and variants  \cite{Park_2020_PAC, Bates_2021_RCPS,angelopoulos2021learn} are natural post-processing methods that offer statistically valid uncertainty estimates. 
Some applications to object detection have been proposed recently, notably through conformal prediction and tolerance regions.
These two frameworks build and calibrate post-hoc prediction sets with statistical guarantees, although slightly different ones.
To the best of our knowledge, most such applications to OD, e.g., \cite{deGrancey_2022_object, Andéol_2023_conformal, Andéol_2023_confident, mukama2024copula}, only consider the localization task, i.e., predicting the coordinates of bounding boxes around objects on an image.
Two exceptions are~\cite{timans2024adaptive} and~\cite{Li_2022_towards}.
The former provides a guarantee on the classification task, while also producing conformal corrections for localization that are conditional on the class of the predicted object.
The other work \cite{Li_2022_towards} builds all-encompassing sets (including confidence, classification and localization) that are specifically tailored to Faster R-CNN models.
However, to the best of our knowledge, no work offers a holistic framework to build sets. That is, prediction sets that would be valid for virtually all families of OD models, with statistical guarantees on both the number of predicted objects and their localization and classification accuracies.

A similar approach to ours, with two parameters to tune sequentially, was recently and independently introduced in \cite{xu2024twoStageCRC} for rank retrieval. 
However, their theoretical guarantee is either asymptotic or requires two data splits for finite samples, whereas our method provides a finite-sample guarantee using a single data split only.
The framework of LTT~\cite{angelopoulos2021learn} provides another method to build prediction sets, which can depend on more than one parameter, using a  multiple testing approach. 
This framework has recently been applied to language modeling~\cite{quach2023conformal} as well as to rank retrieval~\cite{xu2024twoStageCRC}. 
It has also been applied to \od for medical data \cite{Blot_2024_efficient}, for the problem of simultaneously controlling precision and recall.

\textit{Other approaches to UQ for OD.}
Leaving statistical guarantees aside, let us mention other useful methods to compute uncertainty estimates for OD.
First, several heuristic approaches exist. MetaDetect~\cite{schubert2021metadetect} trains an auxiliary network on top of the object detector to estimate predictive uncertainty from handpicked metrics. It does so by learning to discriminate false positives from true positives (in terms of Intersection over Union, IoU), and to predict IoU.
The method of \cite{kuppers2022confidence} allows to calibrate the confidence score of predicted bounded boxes, in a position-aware manner. The resulting confidence of bounding boxes in all regions of the image is then improved in terms of calibration error. Estimating uncertainty from the gradient of the loss (with regards to the weights) has also been studied in~\cite{riedlinger2023gradient} and a comparison of the two approaches is discussed in~\cite{riedlinger2022uncertainty}. 
Notably, both approaches offer significantly more insights into a bounding box’s likelihood of containing an object than the native confidence score.

Another widely used approach is Bayesian inference, which models uncertainty through a \textit{posterior} distribution over parameters.
This distribution cannot be computed analytically and needs to be approximated. It is mainly done using methods such as Monte-Carlo Dropout~\cite{neal2012bayesian,gal2015dropout} or Variational Inference~\cite{hinton1993keeping,kingma2013auto}, but alternatives exist~\cite{wilson2020bayesian,li2016preconditioned}. 
A first application of Bayesian Neural Networks to OD~\cite{miller2018dropout} demonstrated significantly increased precision and recall, and a good estimation label uncertainty in the open set detection setting. 
Another approach named BayesOD~\cite{Harakeh_2020_bayesod}, applicable to anchor-based networks, allows to replace NMS with a Bayesian fusion: after clustering largely overlapping predictions, the clusters are merged into a single prediction with an uncertainty estimate on its coordinates. 
This allows incorporating information from less confident overlapping predictions, rather than simply deleting them. 
Other, non Bayesian improvements to the NMS have also been proposed~\cite{bodla2017soft,hosang2017learning}.

All these works are complementary to ours. Indeed, even if they lack theoretical guarantees, they provide uncertainty estimates that could be used as \emph{inputs} to our method, in order to construct more adaptive prediction sets. 
Our approach enables the integration of standard uncertainty metrics to adjust predictions in an uncertainty-aware manner, for example, by substantially enlarging prediction sets for highly uncertain cases, while leaving confident ones unchanged, instead of applying a uniform scaling factor.

\section{Setting: Object Detection with Uncertainty Quantification}
\label{sec:od-uq}
Both \acf{OD} and \acf{UQ} for \od have long been dominated by heuristic methods. As we aim to bring guaranteed uncertainty quantification to this field, we start by formally introducing the setting. Our new SeqCRC method and theoretical analysis are presented in Section~\ref{sec:cod}, while a large selection of OD-specific modeling choices are discussed in  Section~\ref{sec:app-pred-sets-losses}. 

\subsection{Problem Setup (without UQ)} \label{ssec:pb_setup}

\begin{figure}
    \centering
    \includegraphics[width=\columnwidth]{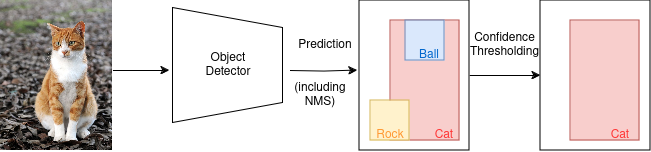}
    \caption{Abstract diagram representing the lens through which we study object detectors. As seen, we consider the NMS to have already been conducted post-prediction.
    }
    \label{fig:diag-od}
\end{figure}

Object detection on images is a complex task, as it involves both localizing and classifying an a priori unknown number of objects for each input image.  
To that end, most models (e.g. \cite{Redmon_2016_YOLO,Jocher_2020_YOLOv5_by_Ultralytics, Jocher_2023_YOLO_v8, ge2021yolox}) predict a large amount of bounding boxes with an additional confidence component.
This score is further used to filter out most predicted boxes using a threshold to keep the most relevant bounding boxes. 
While existing models compute this score in different ways, and some may use additional scores~\cite{ren2017fasterRCNN-tpami}, we only consider a single score, irrespective of its construction, to provide a holistic setting and method.

{\bf Formal Problem Setup.} 
We denote by $\mathcal{X}$ the input (image) space. Each object on an image is modeled as an element of  $\mathcal{B}\times\mathcal{C}$, where $\mathcal{B}=\mathbb{R}^{4}_{+}$ is the bounding box space and $\mathcal{C}=\left\{1,\cdots,K\right\}$ is the label space ($K$ is the number of labels).
Each location vector $b=(\bleft,\btop,\bright,\bbottom)
\in\mathcal{B}$ 
contains the coordinates of the top-left and bottom-right corners of its corresponding bounding box, respectively $(\bleft, \btop)$ and $(\bright, \bbottom)$. 
We assume that $\bleft\leq \bright$ and $\btop \leq \bbottom$.\footnote{The horizontal axis points to the right while the vertical axis heads downwards.}

Throughout this paper, we will refer to inclusion of a bounding box $b$ by another $\bt$, denoted $b \subseteq \bt$, to mean that we simultaneously have $\bleft \geq \btleft, \btop \geq \bttop, \bright \leq \btright, \bbottom \leq \btbottom$. 
We use the notation $b \not\subseteq \bt$ whenever one or more of these inequalities do not hold true. 

Commonly in \od, the output for a predicted object is not a single label, but an estimated probability distribution on $\mathcal{C}=\{1,\cdots,K\}$, such as a softmax output. 
Further recall that we work with a predictor $f$ that outputs a large, fixed number $N^\text{pred}$ of objects. Each of these objects is composed of a bounding box, with an additional classification score (softmax output) and a confidence score in $[0,1]$. 
Such predictors are therefore functions $f:\mathcal{X}\rightarrow\left(\mathcal{B}\times\Sigma_{K-1}\times[0,1]\right)^{N^\text{pred}}$, where $\Sigma_{K-1}$ denotes the $(K-1)$-simplex.\footnote{That is, $\Sigma_{K-1}=\{ c \in [0,1]^{K} : \sum_{i=1}^{K}c_i =1$\}.}

Importantly, a recurring post-processing step for object detectors is Non-Maximum Suppression (NMS), which consists in removing redundancies and thus significantly lowering the amount of predicted boxes.  
Its output lies in the reduced space $\left(\mathcal{B}\times\Sigma_{K-1}\times[0,1]\right)^{N^\mathrm{nms}}$ where the number $N^\mathrm{nms}$ depends on $x$ and is typically much smaller than $N^\text{pred}$.

Similarly, each image $x$ is assigned a true complex label $y \in (\mathcal{B} \times \mathcal{C})^{|y|}$ given by the list (or sequence) of all annotated ground-truth objects on $x$ (we write $|y|$ for the number of objects). We will sometimes use $\mathcal{Y}$ to denote the set of all such complex labels $y$.

{\bf Instances of the Problem.}
Most of the commonly used models in the literature fit in the setting presented above. 
Popular examples are YOLO~\cite{Redmon_2016_YOLO,Jocher_2020_YOLOv5_by_Ultralytics,Jocher_2023_YOLO_v8} for single-stage methods, and Faster R-CNN~\cite{ren2017fasterRCNN-tpami} for two-stage methods. 
Other examples include the more recent transformer-based models such as DETR~\cite{Carion_2020_detr} and its variants~\cite{zheng2023less,zhao2024ms}, as well as the diffusion-inspired model DiffusionDet~\cite{Chen_2023_DiffusionDet}.

\subsection{Adding UQ with three parameters} 
\label{sec:setting-UQ}
In this work, we study an approach to \uq in \od that has gained renewed interest over the past decade: Conformal Prediction~\cite{Angelopoulos_2022_gentle}. 
We consider a post-processing approach that corrects predictions, such as adding a margin on a bounding box or outputting multiple classes instead of the most likely one.
This is the conformal approach, and such corrected bounding boxes, or sets of classes, are called ``prediction sets''.
Section \ref{sec:related_works} discusses earlier works following this approach. Put briefly, our conformal prediction method for OD is a simple, generic, and statistically-grounded way to tune three parameters:
\begin{itemize}
    \item a confidence parameter $\lbcnf\in\Lbcnf$ used to define a threshold below which predicted objects are ignored;
    \item a localization parameter $\lbloc\in\Lbloc$ used to correct the size of predicted bounding boxes (by, e.g., adding a margin in pixels to these), and
    \item a classification parameter $\lbcls\in\Lbcls$ used to build the set of most likely classes for any predicted object (via, e.g., a threshold on the softmax scores).
\end{itemize}

Figure \ref{fig:first_fig} illustrates the usage of these three parameters.
After Non-Maximum Suppression (which is considered as part of the OD model itself), the thresholding step (using~$\lbcnf$) allows us to reduce the list of predicted objects to a smaller list of objects $\Gacnf$ that are more likely to correspond to true objects.
Then, from these predictions, we carry out two tasks in parallel. For localization, we construct a list $\Galoc$ of bounding boxes in $\mathcal{B}$ obtained by expanding the bounding boxes from $\Gacnf$ (with an additive or multiplicative margin).
For classification, we construct a list $\Gacls$ of subsets of $\{1,\ldots, K\}$---one subset for each predicted object, with the labels predicted as most likely.

In the rest of the paper, we show how to construct the prediction sets $\Gacnf, \Galoc, \Gacls$ and how to tune the parameters $\lbcnf,\lbloc,\lbcls$ so that the resulting predictions are statistically valid (Section~\ref{sec:cod}) and practically meaningful (Section~\ref{sec:app-pred-sets-losses}). A comprehensive list of experiments is discussed in Section~\ref{sec:experiment}.

\section{Conformal Object Detection}
\label{sec:cod}

Neural networks are often regarded as black boxes due to their complexity. There is an important research effort in deep learning theory for generalization bounds~\cite{anthony2009neural,bartlett2021deep,grohs2022mathematical}. 
However, due to the complexity of neural networks, these bounds are often loose and therefore not yet suitable for industrial requirements. In particular, in Object Detection, there is a lack of a theoretical framework.
This section introduces a post-hoc method for constructing model-agnostic, distribution-free prediction sets guaranteed to contain the ground truth with large probability.
This method pertains to an approach known as Conformal Prediction and first introduced by Vovk et al. (e.g., \cite{Vovk_2005_algorithmic} and references therein). In particular, one algorithmic variant called Split (or Inductive) Conformal Prediction~\cite{Papadopoulos_2002_inductive} is well suited to post-hoc applications on large models.
This simple approach allows us to bypass generalization error bounds of neural networks, via a simple yet statistically rigorous ``testing'' viewpoint.
For an introduction to the subject we refer the reader to, e.g., \cite{Angelopoulos_2022_gentle,Fontana_2023_conformal_review, Vovk_2005_algorithmic,Angelopoulos_2024_theoretical_CP}.

In this paper, we study an extension of Split Conformal Prediction called \emph{\crcf}~\cite{Angelopoulos_2022_CRC} and that adds flexibility for the user---a necessary feature for many applications.
Next we start with some useful background on \crc, which only applies to prediction sets indexed by a single (scalar) parameter. Then, in Section~\ref{sec:seqcrc-theory}, we design and analyze an extension of \crc suited to our Object Detection setting, with three dependent scalar parameters.

\subsection{Background: Conformal Risk Control}
\label{sec:background-crc}
In this section we recall how conformal risk control proceeds in the post-hoc setting,\footnote{By \emph{post-hoc} we refer to a procedure conducted on top of a pretrained model's predictions.} for a general supervised learning task such as classification or regression (we momentarily work with generic input and label spaces $\mathcal{X}$ and $\mathcal{Y}$). 
Assume that we are given a deterministic predictor $f$.\footnote{The predictor is in fact required to be independent from the test and the calibration data. Roughly speaking, this corresponds to a setting where the model $f$ is trained and validated on a different split or pretrained.}
To construct a prediction set, we require a calibration dataset $\mathcal{D}_{\text{cal}}=(X_i, Y_i)_{i=1}^{n}\in(\mathcal{X}\times\mathcal{Y})^{n}$ of $n$ examples and a test point $\Xtest \in \mathcal{X}$.
We then build a \emph{prediction set} $\Gamma_{\lambda}(\Xtest)$\footnote{
    Equivalently, prediction sets can be expressed in terms of a \emph{non-conformity score}. See \cite{Gupta_2022_nested} for details.
}
with the goal of controlling a loss $\ell(Y_\mathrm{test},\Gamma_\lambda (X_\mathrm{test}))$ below a user-defined threshold $\alpha$.
For example, in multiclass classification where $\mathcal{Y} = \{1,\ldots,K\}$, we can work with $\Gamma_{\lambda}(\Xtest) \subset \mathcal{Y}$ and $\ell(y,S)=\mathbbm{1}_{y\notin S}$. Multiple other examples exist, with non-binary losses or with prediction sets $\Gamma_{\lambda}(\Xtest)$ that are not necessarily subsets of~$\mathcal{Y}$ as, e.g., for the localization task (in that case, the prediction set is the predicted box after application of a margin, see Section~\ref{sec:seqcrc-od-localization}).
For notational simplicity, we denote the random losses\footnote{
    These are referred to as \textit{random losses} because they are functions of random variables, hence random variables themselves.
} by
$L_{i}(\lambda)=\ell(Y_i, \Gamma_{\lambda}(X_i))$, for $i=1,...,n$. We also denote the random loss of the test sample by $L_\mathrm{test}(\lambda)=\ell(Y_\mathrm{test}, \Gamma_{\lambda}(X_\mathrm{test}))$.

Typically, such prediction sets $\Gamma_\lambda(x)$ are parameterized by a real number \lbd that controls their sizes.
Conformal methods are data-driven procedures for computing some
$\lbhat$
whose associated prediction sets $\Gamma_{\lbhat}(\Xtest)$ are both statistically valid and as ``small'' as possible.
(As will become clear later, $\lbhat$ depends on the calibration data, and thus $\Gamma_{\lbhat}(\Xtest)$ depends on both $\Xtest$ and the entire calibration dataset.)

We now provide more formal details on the assumptions, the CRC method, and its statistical guarantee. Throughout the paper, the recurring
assumptions on the model and data
are the following.

\begin{assumption}
    \label{ass:cp-independence}
    The calibration and test data are observations of $n+1$ random variables $(X_1, Y_1),\dots,(X_{n},Y_{n}),(X_\mathrm{test}, Y_\mathrm{test})$ that are independent and identically distributed. Moreover, the predictor $f$ is deterministic.
\end{assumption}

Note that in practice, $f$ is typically the result of a training step, using another dataset $\Dtrain$. 
In that case, the appropriate 
assumption is that the $n$ calibration points together with the test sample be independent and identically distributed conditionally on $\Dtrain$.
This implies that the \crc procedure could use calibration and test data from a different distribution than seen in training, so that one may reuse pre-trained models on a new specific task~\cite{deGrancey_2022_object,Andéol_2023_confident}.

The CRC method recalled next only handles a \emph{scalar} parameter $\lb \in \Lb$, where $\Lb \subset \R$. We also suppose that $\Lambda$ is compact (for simplicity), as well as the following.

\begin{assumption}
\label{ass:loss}
The random loss functions 
$L_i: \Lambda \rightarrow \R$ defined by $L_{i}(\lambda)=\ell(Y_i, \Gamma_{\lambda}(X_i))$ 
are non-increasing and right-continuous in $\lambda$, take their values in $\left[0, B\right]$
and satisfy $L_i(\bar{\lambda})=0$ almost surely, where $\bar{\lambda}=\max\Lambda$.
\end{assumption}

The monotonicity assumption on the $L_i$ is immediately satisfied when the prediction sets are nested (i.e. $\Gamma_\lambda(x) \subset \Gamma_{\lambda^\prime}(x)$ for all $\lambda \leq \lambda^\prime\in\Lambda$ and $x\in\mathcal{X}$) and when the loss function $\ell(y,S)$ is non-increasing in $S$ (i.e., $\ell(y,S)\geq\ell(y,S^\prime)$ for all $S\subset S^\prime$). 
The earliest methods in Conformal Prediction correspond to the special case $\ell(y,S)=\mathbbm{1}_{y\notin S}$. Another example is the Recall of False Negative Rate loss. 
In Section~\ref{sec:app-pred-sets-losses} we will use both examples for classification and localization tasks respectively, as well as other ones.

The objective is then to correctly tune the parameter $\lambda$ in order to control the behavior of $L_\mathrm{test}(\lambda)=\ell(Y_\mathrm{test},\Gamma_\lambda (X_\mathrm{test}))$ with respect to a user-defined threshold $\alpha$. 
The CRC method of~\cite{Angelopoulos_2022_CRC} computes
\begin{equation}
    \lbhat = \inf \left\{ \lambda \in \Lambda : \frac{1}{n+1}\sum\limits_{i=1}^{n}L_i(\lambda) + \frac{B}{n+1}\leq \alpha \right\}.
    \label{eq:crc-inf-lambda}
\end{equation}

Put differently, CRC finds the smallest $\lambda$ (i.e., the one leading to the smallest prediction set) such that the average loss $\frac{1}{n+1}\sum_{i=1}^{n}L_i(\lambda) + \frac{L_{\textrm{test}}(\lambda)}{n+1}$ over the $n+1$ calibration/test points is below the desired error rate~$\alpha$, when anticipating a worst-case value of $B$ for the unseen test loss~$L_{\textrm{test}}(\lambda)$.

\begin{theorem}[Theorems 1 and 2 in \cite{Angelopoulos_2022_CRC}]\label{th:CRC-thm}
    Under Assumptions~\ref{ass:cp-independence} and~\ref{ass:loss}, for any $\alpha\geq\frac{B}{n+1}$, the parameter
    $\hat{\lambda}$ defined in \eqref{eq:crc-inf-lambda} satisfies
    \vspace{-0.1em}
    \begin{equation} \label{eq:crc}
        \alpha - \frac{2B}{n+1}\leq\Ex{L_{\mathrm{test}}(\hat{\lambda})}\leq\alpha,
    \end{equation}
    where the lower bound holds under an additional continuity-type assumption on the random losses.
\end{theorem}

Importantly, the above expectation is over both $\mathcal{D}_{\text{cal}}$ and $(\Xtest, \Ytest)$: roughly speaking, it is a mean computed over all possible joint observations of $(X_1,Y_1),\dots,(X_n,Y_n), (\Xtest,\Ytest)$. There is no guarantee that \eqref{eq:crc} holds for all possible test instances, or for all calibration data sets. 
In particular, when $\ell(y,S)=\mathbbm{1}_{y\notin S}$, the event $\{\Ytest\notin\Gamma_{\lambda}(\Xtest)\}$ might occur because of an atypical test image $\Xtest$, label $\Ytest$, or of atypical calibration points $(X_i, Y_i)$ in $\mathcal{D}_\text{cal}$. 
Ideally, one may hope for a guarantee that holds conditionally on each test instance, but such test-conditionality is impossible to obtain without further assumptions, as shown in~\cite{Vovk_2012_Conditional}.

\emph{Other Conformal Procedures.}
Multiple extensions exist to solve variations of this problem.
For example, with the CLCP method of~\cite{Wang_2023_clcp}, one can replace the guarantee on the mean (expectation) with a guarantee on the median or quantiles of the loss.
We may also want to obtain guarantees that hold not only on average over all classes, but on average per class
\cite{Vovk_2012_Conditional, Sadinle_2019_lac_Least_Ambiguous}. 
This can be useful to reduce over-correction or under-correction of some target classes.
Furthermore, Conformal Prediction has been studied in other settings, such as Full \cp \cite{Vovk_2005_algorithmic}, which trains the predictor and calibrates prediction sets on the same data, or the jackknife (leave-one-out) and cross-validation variants \cite{Vovk_2013_cross_conformal,barber2021predictive, Kim_2020_jackknife_bootstrap}.

\subsection{Sequential Conformal Risk Control with 1+2 parameters}
\label{sec:seqcrc-theory}

In the classical setting recalled in Section~\ref{sec:background-crc}, one can control a loss on a prediction set indexed by a single parameter~$\lambda$. 
In all the sequel, we work within the object detection setting of Section \ref{sec:od-uq}. The calibration and test data correspond to random pairs $(X_i,Y_i)$, $i \in \{1,\ldots,n\} \cup \{\text{test}\}$, where $X_i \in \mathcal{X}$ is an image and $Y_i \in \mathcal{Y}$ denotes its true list of bounding boxes and classes. More importantly, we now consider three parameters: $\lbcnf$ for confidence thresholding, $\lbloc$ for the localization task, and $\lbcls$ for the classification task.
The localization and classification tasks are dependent on the confidence task, and existing theory for conformal prediction cannot handle this situation. 
Therefore we introduce \emph{\seqcrc}, 
a new method that extends the CRC algorithm to handle 1+2 sequentially dependent parameters, which is well suited for \od tasks.\\[-0.2cm]

\subsubsection{Notation and assumptions}
\label{sec:SeqCRCnotationassumptions}
We denote by $L_i^{j}$, with $j\in\{\mathrm{cnf}, \mathrm{loc}, \mathrm{cls}\}$, the loss functions for the confidence ($\Lcnf$), localization ($\Lloc$) and classification ($\Lcls$) tasks. 
For each task $j\in\{\mathrm{cnf}, \mathrm{loc}, \mathrm{cls}\}$ we also write $\lb^j\in\Lb^j$ for the associated parameter and $\alpha^j$ for the desired error rate.
We assume that the $\Lb^j \subset \R$ are compact, and set $\bar{\lambda}^j = \max \Lambda^j$. (Our examples in Section~\ref{sec:app-pred-sets-losses}  are of the form $\Lb^j = \bigl[0,\bar{\lambda}^j\bigr]$.)

Our procedure proceeds in two steps:
\begin{itemize}
    \item First is the \textit{confidence} step, with the loss $\Lcnf_i(\lbcnf)$ depending only on a single parameter (as in Section~\ref{sec:background-crc}). 
    \item Then, two independent steps are carried out in parallel: \textit{localization} and \textit{classification}. They depend on the first one,  and their associated losses are thus denoted by $\Lbul_i(\lbcnf, \lb^\bullet)$, where $\bullet\in\{\mathrm{loc},\mathrm{cls}\}$.
\end{itemize}

Though not written explicitly in this section, the random loss functions $\Lcnf_i: \Lbcnf \rightarrow \R$ and $\Lbul_i: \Lbcnf\times\Lambda^\bullet \rightarrow \R$ are obtained through task-specific losses and predictions sets defined in Section~\ref{sec:app-pred-sets-losses}, roughly of the form:\footnote{
    Though not written explicitly below, the loss $\Lbul_i(\lbcnf,\lbbul)$ also depends on a matching between ground-truth and predicted boxes. 
    See Section~\ref{sec:level-of-guarantee-and-matching} for details.
}
\[
\Lcnf_i(\lbcnf)=\ell^{\mathrm{cnf}}\bigl(Y_i, \Gamma^{\mathrm{cnf}}_{\lbcnf}(X_i)\bigr)
\]
and
\[
\Lbul_i(\lbcnf,\lbbul)=\ell^{\bullet}\bigl(Y_i, \Gamma^{\bullet}_{\lbcnf,\lbbul}(X_i)) \;.
\]

For the analysis, we will assume the following two-step counterpart of Assumption~\ref{ass:loss}. Recall that the $\Lb^j \subset \R$ are compact, and that $\bar{\lambda}^j = \max \Lambda^j$. 
\begin{assumption}
\label{ass:loss-seq}
The random loss functions $\Lcnf_i: \Lbcnf \rightarrow \R$ and $\Lbul_i: \Lbcnf\times\Lambda^\bullet \rightarrow \R$ are non-increasing and right-continuous in all of their parameters,\footnote{i.e., non-increasing and right-continuous in each of the parameters, all the other parameters being fixed. (This should hold true for every sample path.)} and take their values in $\left[0, B^\mathrm{cnf}\right]$ and $\left[0, B^\bullet\right]$ respectively.
Also assume that for 
any $i \in \{1,\cdots,n\}$, we have
$\Lbul_i(\lbcnfbar, \lbbulbar)=0$ almost surely.    
\end{assumption}

In general, we do not assume that $\Lcnf_i(\lbcnfbar)=0$ almost surely (this is not necessarily the case in Section~\ref{sec:app-pred-sets-losses}), though this additional property implies additional guarantees, as shown in Theorem~\ref{th:odcrc} below. \\

\subsubsection{The SeqCRC method}
\label{sec:seqcrc}
We now introduce our method; see Steps~1 and~2 below. We use the following notation for the empirical risks related to the confidence, localization, and classification tasks:
\begin{align*}
          \Rcnf_n(\lbcnf) & = \frac{1}{n}\sum_{i=1}^n \Lcnf_i(\lbcnf)\,,\\
    \Rbul_n(\lbcnf,\lbbul) & = \frac{1}{n}\sum_{i=1}^n \Lbul_i(\lbcnf, \lbbul)   \,.
\end{align*}

\noindent
We also use a (slightly) conservative variant of $\Rcnf_n(\lbcnf)$ that anticipates the other two tasks:
\begin{align}
\label{eq:risk_cnf_n}
    \Rtcnf_n(\lbcnf) & = \max \Big\{\Rcnf_n(\lbcnf), \Rloc_n(\lbcnf,\lblocbar),\\
    & \qquad\qquad\Rcls_n(\lbcnf,\lbclsbar)  \big\}\,, \nonumber
\end{align}
where $\lblocbar=\max \Lambda^\text{loc}$ and $\lbclsbar=\max \Lambda^\text{cls}$. Note that $\Rtcnf_n(\lbcnf)  \in [0,\Tilde{B}^\mathrm{cnf}]$ with $\Tilde{B}^\mathrm{cnf}=\max \{ B^\mathrm{cnf}, B^\mathrm{loc}, B^\mathrm{cls} \}$.

\ \\
\textbf{Step 1:} For the confidence step, we compute: 
\begin{align}
\label{eq:seqcrc-lbcnf-plus}
    \lbcnfp & = \inf \left\{ \lbcnf \!\in\! \Lbcnf \!:\!
     \frac{n \Rtcnf_n(\lbcnf)}{n+1} \!+\! \frac{\Tilde{B}^\mathrm{cnf}}{n+1} 
    \!\leq\! \alphacnf\! \right\}, \\            
\label{eq:seqcrc-lbcnf-min}
    \lbcnfm & = \inf \left\{ \lbcnf \!\in\! \Lbcnf \!:\! \frac{n\Rtcnf_n(\lbcnf)}{n+1} \!+\! \frac{0}{n+1} \!\leq\! \alphacnf\! \right\},                 
\end{align}
where we use the convention $\inf \varnothing = \max \Lambda^\text{cnf} = \lbcnfbar$.

Importantly, contrary to the CRC algorithm, we build two estimators instead of one: $\lbcnfp$ is used on test images for confidence thresholding, while $\lbcnfm$ is used in the second step~\eqref{eq:seqcrc-lbj} below. 
The reason we work with $\Rtcnf(\lbcnf)$ instead of $\Rcnf(\lbcnf)$ is to ensure that this second step is feasible, as will be clear from the analysis.\\

\noindent
\textbf{Step 2:}  For the second step (localization and classification), our SeqCRC method computes:
\begin{align}
\label{eq:seqcrc-lbj}
    \lbbulp & = \inf \left\{ \lbbul \in \Lbbul : \frac{n\Rbul_n(\lbcnfm, \lbbul)}{n+1} + \frac{B^\bullet}{n+1} \leq \alphabul\right\}.
\end{align}

\subsubsection{Theoretical analysis of SeqCRC}
We now state our main theoretical result. It implies a probabilistic guarantee for the localization and classification tasks, and an additional guarantee for the confidence task under an additional assumption on the confidence loss function. Recall the placeholder $\bullet\in\{\mathrm{loc},\mathrm{cls}\}$ and that $\bar{\lambda}^\mathrm{cnf} = \max \Lambda^\mathrm{cnf}$.
\begin{theorem}
    \label{th:odcrc}
    Suppose that Assumptions \ref{ass:cp-independence} and \ref{ass:loss-seq}
    hold true.\footnote{\label{ft:iidlosses}To be more rigorous, in addition to Assumption~\ref{ass:cp-independence}, we should assume that $(L^\text{cnf}_i, L^\text{loc}_i, L^\text{cls}_i)$ is fully determined (in a measurable way) by $(X_i,Y_i)$, as in the examples of Section~\ref{sec:app-pred-sets-losses}. 
    This implies the following key property: that the random $3$-tuples $(L^\text{cnf}_1, L^\text{loc}_1, L^\text{cls}_1), \ldots, (L^\text{cnf}_{n+1}, L^\text{loc}_{n+1}, L^\text{cls}_{n+1})$ are independent and identically distributed.} 
    Let $\alphacnf \geq 0$ and $\alphabul \geq \alphacnf + \frac{B^\bullet}{n+1}$. 
    Then, $\lb^\mathrm{cnf}_{\pm}$ and 
        $\lbbulp$ introduced in \eqref{eq:seqcrc-lbcnf-plus}, \eqref{eq:seqcrc-lbcnf-min} and \eqref{eq:seqcrc-lbj} are well defined, and
    \begin{equation}
        \label{eq:thm}
        \Ex{L_\mathrm{test}^\bullet(\lbcnfp, \lbbulp)} \leq \alphabul \;.
    \end{equation}
    Under the additional assumption that $\Lcnf_i(\lbcnfbar)\leq\alpha^\mathrm{cnf}$ almost surely for all $i \in \{1,\ldots,n\} \cup \{\mathrm{test}\}$, we also have
    \begin{equation}
        \label{eq:thm_cnf}
        \Ex{L_\mathrm{test}^\mathrm{cnf}(\lbcnfp)} \leq \alphacnf \;.
    \end{equation}
\end{theorem}

\begin{proof}
    See Appendix~\ref{app:proof}.
\end{proof}
The proof follows from similar arguments to the (single-step) CRC algorithm of \cite{Angelopoulos_2022_CRC}, with some additional subtleties due to the sequential nature of the problem. 
In particular, the modified empirical risk $\Rtcnf_n$ \eqref{eq:risk_cnf_n} in the first step is useful to ensure that the set in \eqref{eq:seqcrc-lbj} in nonempty, and the ``optimistic'' estimator $\lbcnfm$  proves key for symmetry purposes.

This guarantee is similar to that in Theorem 3 by \cite{xu2024twoStageCRC} (independent work, with another targeted application), but our SeqCRC method achieves a finite-sample guarantee with a \textbf{single data split}. In contrast, their approach relies on standard \crc estimators, thus requiring an additional data split to get a finite-sample guarantee for the second stage. Our method avoids this by introducing a corrected estimator, $\lbbulp$, which is based on the optimistic estimator $\lbcnfm$ from the first stage.

The bound in \eqref{eq:thm} implies two individual guarantees, for the localization and classification tasks respectively. 
In practice, we usually want a joint guarantee on the two tasks: an example is given by the following immediate corollary (we use $\max\{a,b\} \leq a+b$ for $a,b \geq 0$).

\begin{corollary}
\label{corollary_sum}
    Under the same assumptions as in Theorem \ref{th:odcrc}, 
    \begin{equation}
        \Ex{\max\big(\Lloc_\mathrm{test}(\lbcnfp, \lbloc_{+}), \Lcls_\mathrm{test}(\lbcnfp, \lbcls_{+})\big)}\leq \alphatot ,
    \end{equation}
    where $\alphatot = \alphaloc + \alphacls$.
\end{corollary}
In other words, it is possible to get a global guarantee at a level $\alphatot$ by splitting it between the localization and classification tasks. The same logic is applicable with three terms, if need be to include the confidence loss in the final guarantee (at the cost of a larger total error $\alphatot=\alphacnf+\alphaloc+\alphacls$). For instance, simultaneously guaranteeing a sufficient number of predictions while ensuring their localization and classification.
Importantly, the choices of $\alphaloc$ and $\alphacls$ given a global $\alphatot$ can be decided according to the difficulty and practical importance of each task, as long as they are \emph{not} done in a data-dependent way;
that is, one cannot choose the values of $\alpha$ \textit{after} observing the calibration data.

We now comment on the assumptions from a practical standpoint.

{\bf On the \emph{vanishing loss} assumption.}
In Assumption~\ref{ass:loss-seq} the two conditions $\Lloc_i(\lbcnfbar,\lblocbar)=0$ and $\Lcls_i(\lbcnfbar,\lbclsbar)=0$ on the localization and classification losses have practical implications. 
When applied to our examples in Section~\ref{sec:app-pred-sets-losses}, they correspond to assuming that, after Non-Maximum Suppression, there is at least one predicted object per image if the underlying matching is non-injective ({i.e., multiple ground-truth objects can be matched to one predicted object}), or at least as many predictions as ground truths if it is injective (e.g., via Hungarian matching, one can match every ground-truth object to its own prediction).

{\bf On the \emph{monotone loss} assumption.}
Though intuitive, the assumption that the loss functions are monotone can fail in practice for $\Lloc_i$ and $\Lcls_i$. 
While it may not appear obvious from the loss definitions in the next section, the matching required for their computation is the root of this issue (see Remarks~\ref{rmk:nonmonotone-loc} and~\ref{rmk:nonmonotone-cls} for details).
This implies that the localization and classification loss functions are typically non-monotone in the $\lbcnf$ parameter.

To avoid this (small) gap between theory and practice, in Section~\ref{sec:seqcrc-od-algos}, we propose a simple \emph{monotonization trick}: within \eqref{eq:seqcrc-lbcnf-plus}, \eqref{eq:seqcrc-lbcnf-min}, and \eqref{eq:seqcrc-lbj} we replace $\Lloc_i$ and $\Lcls_i$ with the smallest upper bounds that are provably monotone, as done similarly in~\cite{Angelopoulos_2022_CRC}. These modified loss functions can be computed on the fly, and preserve the guarantees \eqref{eq:thm} and \eqref{eq:thm_cnf}. Nevertheless, this comes at a cost, as these conservative losses lead to larger parameters $\lblocp,\lbclsp$ and therefore larger prediction sets, as the \emph{gap} to monotonicity of the loss increases. This gap should therefore be as small and localized as possible, especially for values below $\alpha$.

\section{Application-tailored Prediction Sets \& Losses }
\label{sec:app-pred-sets-losses}

The general method presented in the previous section can be instantiated with various losses and prediction sets to meet specific practical requirements.
This section presents a selection of such sets and losses and compares their features.
They are introduced for confidence, localization and classification tasks respectively in Sections~\ref{sec:seqcrc-od-confidence}, \ref{sec:seqcrc-od-localization} and \ref{sec:seqcrc-od-classification}.
Additional losses and prediction sets are discussed in Appendix~\ref{app:additional-losses}.
Finally, in Section~\ref{sec:seqcrc-od-algos} we present algorithms with their high-level pseudo-codes both for calibration (where we monotonize irregular losses on the fly, in order to meet the assumptions behind SeqCRC) and for inference.

Throughout this section, we work with losses that are $[0,1]$-valued, i.e., $B^\mathrm{cnf}=1$ and $B^\bullet=1$.
Moreover we define prediction sets and losses for a deterministic example consisting in an image $x$ and its ground truth $y$. 
(We thus drop the index $i$ in $L_i$, except when multiple images are considered, e.g., in the algorithms' pseudo-code.)\\

Next, we start by discussing essential preliminary considerations on two topics that are specific to object detection: the level of the statistical guarantee and matching.

\subsection{Level of Guarantee \& Matching}
\label{sec:level-of-guarantee-and-matching}
Although object detection might naively be interpreted as a combination of two simple problems---classification and regression---there are modeling and algorithmic subtleties that are key to make SeqCRC's theoretical guarantee (Theorem~\ref{th:odcrc}) meaningful in OD practice.

First, the expectations in \eqref{eq:thm} and \eqref{eq:thm_cnf} are averages over ``instances'', whose definition is ambiguous in the \od context. 
We find two alternatives:

\begin{enumerate}
    \item \textit{Object-level guarantee (instances = objects)}: 
    Each distinct annotated object in an image is considered an instance. 
    For a dataset of $n$ images, where image $X_i$ contains $\abs{Y_i}$ objects, the total number of instances is $\sum_{i=1}^{n} \abs{Y_i}$. 
    \item \textit{Image-level guarantee (instances = images)}: 
    Each image and its associated set of objects in the dataset is considered a single instance. 
    For a dataset of $n$ images, this results in $n$ instances. 
\end{enumerate}

\noindent
The two settings are strongly connected, and while the former is closer to the usual approach in \od, we focus on the more general \emph{image-level} one in the core of our article. We argue that the image-level setting allows for more flexibility in the design of losses, and in particular to account for interactions between objects on an image (more details in Appendix~\ref{app:box_lvl}).

Second,
in order to compute losses on calibration data and metrics on test data,
we need to define a \textit{rule for matching the predictions of the model $f$ to ground-truth annotations}:
in our setting, an image can have multiple objects (boxes) and in practice, it is not trivial to determine which predicted box is associated with which ground-truth box, if any at all.
This matching (ground-truth $\rightarrow$ prediction) does not need to be injective (but may be, such as using a Hungarian algorithm), meaning that multiple ground-truth boxes can be matched to the same predicted box. 

Hereon,
we assume that every image $x$ produces a matching $\pi_x$ that associates each ground-truth $(b_j,c_j)$ box/class pair with its nearest prediction $(b_{\pi_x(j)}, c_{\pi_x(j)})$ with respect to a ``distance'' function $d$.\footnote{We use the term ``distance'' even if $d$ is not necessarily a metric from a mathematical viewpoint.}

The choice of $d$ is critical, as it shapes the prediction sets obtained using our conformal procedure. 
A satisfactory distance, which accurately connects true and predicted objects, is hard to design: there is a trade-off between minimizing spatial distance, and distance in the classification space. Indeed, for a given ground truth, the closest prediction spatially can be very incorrect in classification, and the closest prediction in classification can be far spatially.
For localization, following the seminal work of \cite{deGrancey_2022_object}, we use the asymmetric signed Hausdorff distance (see also \cite{rucklidge1997efficientlyLocatingObjectsHausdorff}):
$$
d_\mathrm{haus}(b, \hat{b})=\max \left\{\bhleft - \bleft, \bhtop - \btop, \bright - \bhright, \bbottom - \bhbottom \right\}.
$$

Intuitively, it corresponds to the smallest margin (in pixels) to add to any of the predicted coordinates of $\hat{b}$ such that the extended box covers the ground-truth box entirely, aligning with localization objectives.

For classification, we propose using the nonconformity score of \cite{Sadinle_2019_lac_Least_Ambiguous} as a distance, which we refer to as LAC, from their \textit{Least Ambiguous Classifier}:
$$
d_\mathrm{LAC}(c, \hat{c})=1-\hat{c}_c,
$$
where $c$ is the integer true-class label and $\hat{c}$ a probability vector (e.g. softmax output). It measures how far the predicted probability $\hat{c}$ for the true class $c$ is from being $1$.
We therefore propose to merge the two, balanced by a parameter $\tau\in(0,1)$, resulting in:
\begin{equation} \label{eq:match_dist_mix}
d_\mathrm{mix}\left((b,c),(\hat{b},\hat{c}) \right) 
= \tau d_\mathrm{LAC}(c, \hat{c}) + (1-\tau)  d_\mathrm{haus}(b, \hat{b}).
\end{equation}

The distance functions $d_\mathrm{LAC}$ and $d_\mathrm{haus}$ are natural candidates since they capture identical or similar phenomena compared to the localization and classification losses used in the next sections. Other distance choices are possible, but not all of them are relevant. 
For instance, we also consider the \textit{Generalized IoU} (Intersection over Union) distance function, following a common practice in \od \cite{rezatofighi2019generalized}:
$$
d_\mathrm{GIoU}(b,\hat{b})=1 - \frac{\mathrm{area}(b\cap\hat{b})}{\mathrm{area}(b\cup\hat{b})}+\frac{\mathrm{area}(\tilde{b}\backslash(b\cup\hat{b}))}{\mathrm{area}(\tilde{b})},
$$
where $\tilde{b}$ is the smallest convex hull enclosing both $b$ and $\hat{b}$. As we will see in Section~\ref{sec:benchmark-results}, this distance fails in practice compared to the other propositions.\\

We now address the confidence, localization and classification tasks. Recall that throughout this section we work with a generic example $(x,y)$, which can belong to calibration or test data. The notation $x$ stands for the input image, while $y$ denotes the list (or sequence) of all annotated ground-truth objects on $x$, which are elements of $\mathcal{B} \times \mathcal{C}$ (bounding box and class, see Section~\ref{ssec:pb_setup}).

\subsection{First Step: Confidence}
\label{sec:seqcrc-od-confidence}
For the first step introduced in Sec.~\ref{sec:seqcrc}, and as (2) in Fig.~\ref{fig:first_fig} we define our prediction set and loss as follows.

\paragraph{Prediction set}
For an input image $x$, the output of our OD model $f$ followed by NMS is a set of $N^\mathrm{nms}$ objects (see Sec. \ref{ssec:pb_setup} and Fig. \ref{fig:diag-od}) with confidence scores $o(x)_1,\dots,o(x)_{N^\mathrm{nms}}$ in $[0,1]$.
In the object detection literature, trained models are evaluated using mean Average Precision (mAP), which is computed by generating precision-recall curves for each class with a range of thresholds on the confidence scores.
For deployment, users must choose a fixed threshold (e.g., 0.5); this threshold determines how many predictions are excluded from the final set of detections.
We offer there to calibrate this threshold, in such a way that we obtain a statistical guarantee on a specific notion of error important for the user.

More formally, let $\lbcnf\in\Lbcnf=[0,1]$. We start by identifying objects that satisfy
\[
    \mathcal{I} = \bigl\{k\in\{1,\cdots,N^\mathrm{nms}\}:o(x)_{[k]}\geq 1-\lbcnf\bigr\} \;,
\]
where $o(x)_{[k]}$ is the $k$-th largest among the $\{o(x)_k\}_k$. Then, for any image $x$, the (post-processed) prediction set $\Gacnf(x)$ consists of all objects with sufficiently high confidence scores:
\begin{equation}
\Gacnf(x)= \bigl\{(\bhat(x)_{[k]}, \hat{c}(x)_{[k]})\bigr\}_{k\in\mathcal{I}} \;.
\label{eq:cnf-pred-set-sequence}
\end{equation}

\paragraph{Losses}
Since the loss $\Lcnf$, which controls the amount of bounding boxes,
comes before the localization step, which controls their size,
it follows that $\Lcnf$ controls a trade-off between the amount and size of bounding boxes through the definition of $\Lcnf$ and the choice of $\alphacnf$. This choice should be made depending on operational and downstream objectives or requirements. We provide several examples.

A first natural approach could be to force the number $\abs{\Gacnf(x)}$ of predictions to be close to the number $\abs{y}$ of true boxes.
However, this is not feasible because the loss would not be non-increasing in $\lbcnf$. 
Therefore, we introduce the ``box-count-threshold'' loss,
that controls this quantity in a one-sided manner, such that at least as many predictions as the number of true objects are kept:
\[
\Lcnf_{\mathrm{box-count-threshold}}(\lbcnf) = \begin{cases}
    0 &\text{ if } \abs{\Gacnf(x)} \geq \abs{y},\\
    1 &\text{ otherwise}.
\end{cases}
\]

This loss is strict: if the number of predictions is even slightly smaller than the number of ground-truth objects, the penalty is maximal ($=1$).
The ``box-count-recall'' loss is a relaxed version:
\[
    \Lcnf_{\mathrm{box-count-recall}}(\lbcnf) = \frac{\left( \abs{y} - \abs{\Gacnf(x)}\right)_{+}}{\abs{y}},
\]
where $(\cdot)_+ = \max(0,\cdot)$.

These losses do not account for the locations of the predictions and ground truths. We introduce and discuss additional examples of losses in Appendix~\ref{app:additional-losses}.

\subsection{Second Step: Localization}
\label{sec:seqcrc-od-localization}
\begin{figure}
    \centering
    \includegraphics[width=\columnwidth]{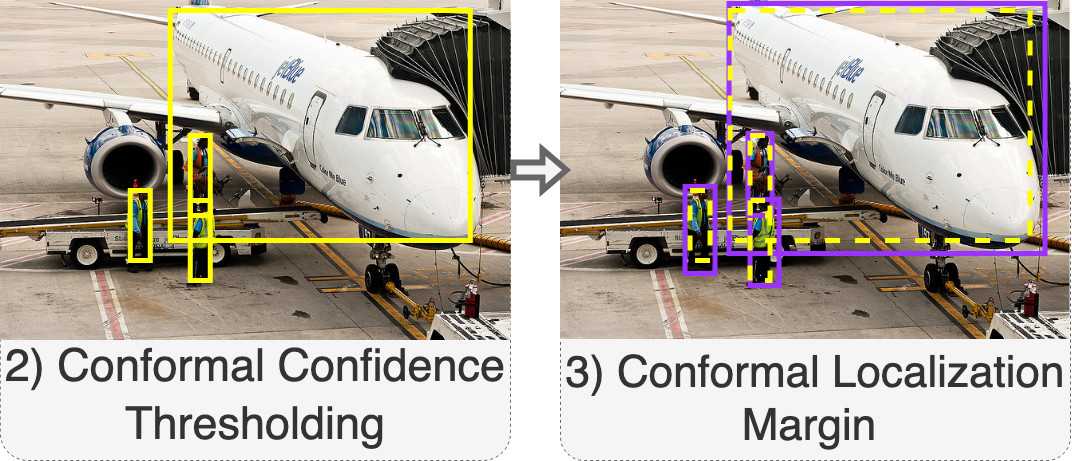}
    \caption{\textbf{Conformal Localization Step.} The localization step, one of the two elements of (3) in Fig.~\ref{fig:first_fig}, consists in applying corrective margins to each of the predicted bounding boxes.}
    \label{fig:submain-fig-loc}
\end{figure}
For the second step introduced in Sec.~\ref{sec:seqcrc} and illustrated in Fig.~\ref{fig:submain-fig-loc}, 
recall that by ``localization'' we refer to how well a bounding box contains the ground-truth bounding box of an object, regardless of its class. 
Similarly to the previous confidence step, we work with sequences of bounding boxes (multiple objects per image).

\paragraph{Prediction sets}
Roughly speaking, we alter the boxes obtained after the confidence step by modifying their coordinates.
For an image $x$, the ``prediction set'' $\Galoc(x)$ is a sequence of boxes with the same size as $\Gacnf(x)$; it contains the same bounding boxes as $\Gacnf(x)$ but adjusted as in
Eqs.~\eqref{eq:galoc_additive} or~\eqref{eq:galoc_multip} below.
 
More precisely, let $\bhat_k =(\bhleft,\bhtop,\bhright,\bhbottom)$ denote the $k$-th box in  $\Gacnf(x)$.
(We ignore the second element of the tuple in~\eqref{eq:cnf-pred-set-sequence}, which is treated in Section~\ref{sec:seqcrc-od-classification} below.)
Following \cite{deGrancey_2022_object,Andéol_2023_conformal}, a natural way to adjust the predicted box $\bhat_k$ is to add an ``additive'' margin (in pixels) around it. This new box is the $k$-th element $\Galoc(x)_k$ of our prediction set $\Galoc(x)$:
\begin{equation}
\label{eq:galoc_additive}
\Galoc(x)_k = \bhat_k+ \begin{pmatrix} -\lbloc&-\lbloc&\lbloc&\lbloc \end{pmatrix},
\end{equation}
where $\lbloc\in\Lbloc\subset\mathbb{R}_{+}$, usually bounded by the size of the image.
However, all bounding boxes will be expanded by an equal margin, which can be disproportional to their size: predictions for small objects may suffer very large corrections and vice versa. 
To mitigate this issue, one can instead
build a ``multiplicative'' margin as in \cite{deGrancey_2022_object,Andéol_2023_conformal}: one \textit{multiplies} the parameter $\lbloc$ by factors $\hat{w}$ and $\hat{h}$, respectively corresponding to the height and width of the predicted box: 
\begin{equation}
\label{eq:galoc_multip}
    \Galoc(x)_k = \bhat_k+\lbloc\begin{pmatrix} -\hat{w}_k &-\hat{h}_k &\hat{w}_k & \hat{h}_k \end{pmatrix},
\end{equation}
where $\lbloc\in\Lbloc\subset\mathbb{R}_{+}$, also bounded by the size of the image.

Note that, as already done for regression tasks (in, e.g., \cite{Papadopoulos_2002_inductive,Bostrom2017,Gupta_2022_nested}), we could replace these width and height factors by any notion of local uncertainty to improve the adaptivity of our conformal correction (see Appendix~\ref{app:additional-losses}). 
It is also possible to obtain corrections that depend on the class associated with the box \cite{timans2024adaptive}.

\paragraph{Losses}
For a given image $x$ with ground-truth labeling $y$, let $b_j$ denote the $j$-th bounding box in $y$. Following Section~\ref{sec:level-of-guarantee-and-matching}, in order to define the losses, we compute a matching $\pi_x$ that maps indices $j$ of true bounding boxes $b_j$ to indices $k$ of predicted boxes $\bhat_k$ in $\Gacnf(x)$. The predicted box paired to a true box $b_j$ is thus denoted by $\bhat_{\pi_x(j)}$.

A natural loss introduced by \cite{deGrancey_2022_object} consists in checking whether a proportion at least $\tau\in[0,1]$ of true boxes $b_j$ on an image $x$ are entirely covered by their matched predictions $\bhat_{\pi_x(j)}$ after resizing the latter as in~\eqref{eq:galoc_additive} or~\eqref{eq:galoc_multip}. Writing $\bhat_{k}^{\lbloc} = \Galoc(x)_k$ for simplicity, this reads:
\[
\Lloc_\mathrm{thr}(\lbcnf, \lbloc) = \begin{cases}
    0 \text{ if } \dfrac{\abs{\{b_j \in y,  b_j \subseteq \bhat_{\pi_x(j)}^{\lbloc}\}}}{\abs{y}} \geq \tau \\[2ex]
    1 \text{ otherwise}.
\end{cases}
\]
This has two drawbacks. 
First, the choice of $\tau$ must be done before looking at the data, and may significantly impact the result (a small change of $\tau$ may cause a large variation in the loss). Second, it penalizes without distinction boxes that fail to be covered by a single pixel, or by a significant margin. 

For this purpose, following \cite{ Andéol_2023_conformal,Andéol_2023_confident}, we also study two relaxations of the previous loss. 
The first consists in directly controlling the previous proportion, and is therefore referred to by \emph{box-wise recall}:
\begin{equation}
    \Lloc_{\rm box}(\lbcnf, \lbloc) = 1-\frac{\abs{\{b_j \in y, b_j \subseteq\bhat_{\pi_x(j)}^{\lbloc}\}}}{\abs{y}} \;.
\label{eq:od-loss-proportion}
\end{equation}
However, the inclusion condition remains restrictive. 
We can further relax the loss using the average proportion of the covered areas of the true boxes $b_j$, via the \emph{pixel-wise} loss:
\[
\Lloc_{\rm pix}(\lbcnf, \lbloc) = 1-\frac{1}{\abs{y}}\sum_{b_j \in y}\frac{\mathrm{area}(b_j \cap \bhat_{\pi_x(j)}^{\lbloc})}{\mathrm{area}(b_j)} \;,
\]
where $\mathrm{area}(b)=(\bright - \bleft)(\bbottom - \btop)$ is the area of a bounding box, and $b_j \cap \bhat$ is the (possibly degenerate) bounding box obtained by intersecting $b_j$ and $\bhat$.

\begin{remark}[On the monotonicity assumption]
\label{rmk:nonmonotone-loc}
In Assumption~\ref{ass:loss-seq} we require that losses be non-increasing in the parameters $\lbloc$ and $\lbcnf$. While the three examples $\Lloc_\mathrm{thr},\Lloc_{\rm box},\Lloc_{\rm pix}$ are easily seen to be non-increasing with respect to the additive/multiplicative margin $\lbloc$, the dependency in $\lbcnf$ is not necessarily monotone.
\begin{figure}
    \centering
    \includegraphics[width=\columnwidth]{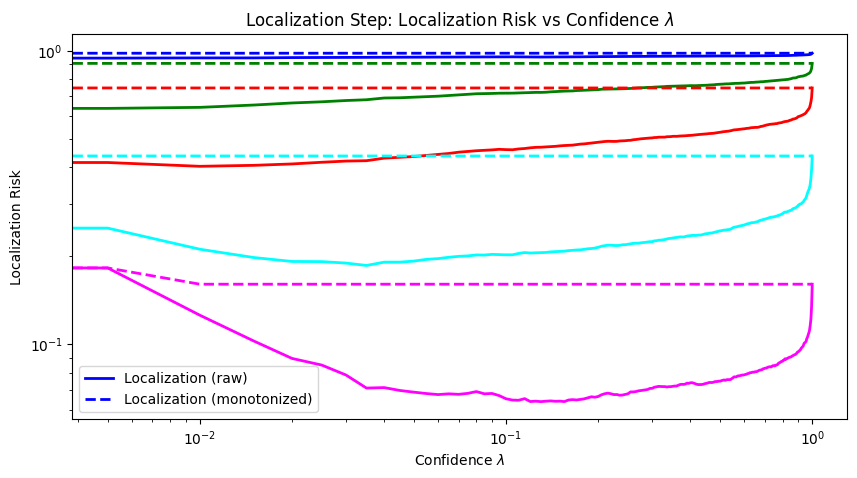}
    \caption{\textbf{Monotonization of the Localization Loss.} The plot represents the box-wise recall risk on calibration data as a function of the confidence parameter $\lbcnf$, that is, the box-wise recall loss (Eq.~\eqref{eq:od-loss-proportion}) averaged over the calibration set. The various colors (blue, green, red, cyan, magenta) correspond to different fixed values of $\lbloc$ (0, 125, 250, 375, 500). The solid lines are the raw risks, and the dotted ones are the risks after loss monotonization (cf. Section~\ref{sec:seqcrc-od-algos}). While we could expect raw losses (and thus raw risks) to be non-increasing, they are usually not, due to (partially) diverging goals between the localization loss and the distance $d$ used in the matching algorithm.}
    \label{fig:monotonicity_loc}
\end{figure}
See Fig.~\ref{fig:monotonicity_loc} for an illustration.
The reason is that, as $\lbcnf$ increases, more and more boxes are added in $\Gacnf(x)$, so that the matching $\pi_x$ between true boxes $b_j$ and boxes $\bhat_k$ in $\Gacnf(x)$ changes. This change is guided by the ``distance'' function $d$ used to compute $\pi_x$, which is only vaguely related to the three loss examples above. 

In Section~\ref{sec:seqcrc-od-algos} and Appendix~\ref{app:algo} we will introduce a \emph{monotonization trick} to fix this issue.
\end{remark}

\subsection{Classification}
\label{sec:seqcrc-od-classification}
We now address the classification task, as illustrated in Fig.~\ref{fig:submain-fig-cls}.
We adapt two well-known CP methods for classification to our setting: the \acf{LAC} of \cite{Sadinle_2019_lac_Least_Ambiguous} and that of \cite{Romano_2020_APS}, which we refer to as \acf{APS}.
However, other methods from the literature could be applied \cite{Angelopoulos_2021_RAPS,luo2024trustworthy,luo2024weighted,melki2024penalized}.

\paragraph{{Prediction sets}} 
Our prediction sets for the classification task are also defined element-wise (for every index~$k$). 
We let $\hat{c}_k$ denote the $k$-th classification prediction within $\Gacnf(x)$, ignoring the previously handled location component. Recall from Section~\ref{ssec:pb_setup} that $\hat{c}_k$ is a probability vector (e.g., a softmax output) over the label space $\mathcal{C} = \{1,\cdots,K\}$.

Letting $\lbcls\in\Lbcls=[0,1]$, the LAC prediction set~\cite{Sadinle_2019_lac_Least_Ambiguous} is defined by
\begin{equation}
\label{eq:Gacls-LAC}
        \Gacls(x)_k = \left\{ \kappa \in \mathcal{C} : \hat{c}_k(\kappa) \geq 1-\lbcls \right\}, 
\end{equation}
where $\hat{c}_{k}(\kappa)$ is the softmax score of the $k$-th prediction for the class $\kappa$. 
Put differently, all sufficiently likely classes (according to the softmax prediction) are predicted.

\begin{figure}
    \centering
    \includegraphics[width=\columnwidth]{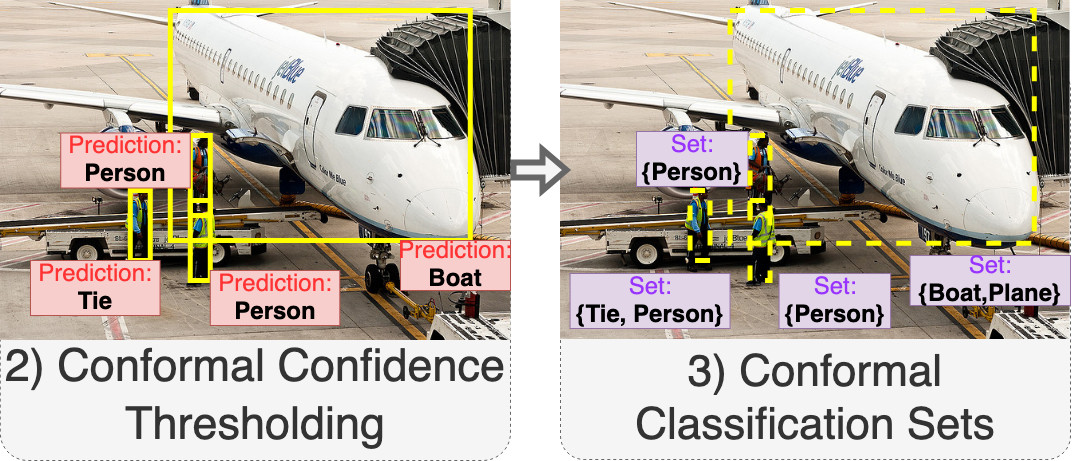}
    \caption{\textbf{Conformal Classification Step.} The classification step, one of the two elements of (3) in Fig.~\ref{fig:first_fig}, consists in a predicting a set of labels for each object, with the guarantee to contain the ground truth with high probability.}
    \label{fig:submain-fig-cls}
\end{figure}

The APS variant by \cite{Romano_2020_APS} consists in retaining the most likely classes too, but in a way that the cumulative softmax probability is above a given threshold.
More precisely, one first ranks the softmax scores in decreasing order: let $\kappa_{[1]},\cdots,\kappa_{[K]}\in\mathcal{C}$ be such that $\hat{c}_k(\kappa_{[1]}) \geq \cdots \geq \hat{c}_k(\kappa_{[K]})$. (The order can be different for different values of $k$.) 
The APS prediction set is then defined by\footnote{When $\lbcls=1$, we set $\Gacls(x)_k = \mathcal{C}$ by convention.}
\begin{equation}
\label{eq:Gacls-APS}
    \Gacls(x)_{k} = \left\{\kappa_{[1]},\dots,\kappa_{[\hat{m}(\lbcls)]}\right\},
\end{equation}    
where
\begin{equation*}
    \hat{m}(\lbcls) = \min\left\{m: \sum_{l=1}^{m} \hat{c}_{k}(\kappa_{[l]}) > \lbcls \right\}.
\end{equation*}

\paragraph{Losses}
Let $c_j \in \mathcal{C}$ denote the class label of the $j$-th object in the ground truth $y$.
We evaluate a binary loss for each true object, by checking whether its class $c_j$ belongs to the associated prediction set component $\Gacls(x)_{\pi_x(j)}$. We then aggregate these loss values. In our experiments we compute an average over all true objects, but some alternatives are discussed in Appendix~\ref{app:box_lvl}. The average version reads:
\begin{equation}
\label{eq:classif_avg_aggregation}
    L^{\mathrm{cls}}(\lbcnf, \lbcls) = \frac{1}{\abs{y}}\sum_{c_j\in y}\mathbbm{1}_{ c_j \notin \Gacls(x)_{\pi_x(j)}} \;.
\end{equation}

\begin{remark}
\label{rmk:nonmonotone-cls}
Similarly to Remark~\ref{rmk:nonmonotone-loc}, for the two prediction sets defined in \eqref{eq:Gacls-LAC} and \eqref{eq:Gacls-APS}, the loss $L^{\mathrm{cls}}(\lbcnf, \lbcls)$ is non-increasing in $\lbcls$ but not necessarily in $\lbcnf$.
\end{remark}

\subsection{Algorithm}
\label{sec:seqcrc-od-algos}
In this section we introduce---in a general form---the algorithms used to compute the prediction sets and tune their parameters. 
They can be used for any choices of prediction set and losses compatible with our method.

\begin{algorithm}
\caption{SeqCRC: Calibration of Parameters}
\label{alg:seqcrc-calibration}
\setstretch{1.2}
\begin{algorithmic}[1]
\Require Calibration dataset $\mathcal{D}_\text{cal} = \{(X_i, Y_i)\}_{i=1}^n$, object detector $f$ (including NMS), matching procedure $\pi$, desired error rates $\alphacnf, \alphaloc, \alphacls$.
\Ensure Calibrated parameters $\lbcnf_+, \lbcnf_-, \lbloc_+, \lbcls_+$

\Statex \textit{\underline{Step 0}: Definitions}
\State Define prediction sets $\Gacnf, \Galoc, \Gacls$
    \label{alg:def_psets}
\State Define losses $L_i^{\mathrm{cnf}}, L_i^{\mathrm{loc}}, L_i^{\mathrm{cls}}$ for confidence, localization, and classification
    \label{alg:def_losses}

\Statex \textit{\underline{Step 1}: Compute predictions}
\For{$i = 1$ to $n$}
    \State $\hat{Y}_i \gets f(X_i)$
\EndFor

\Statex \textit{\underline{Step 2}: Calibrate confidence parameters}
\State $\lbcnf_+ \gets$ Solve Eq.~\eqref{eq:seqcrc-lbcnf-plus}
using $(X_i,Y_i, \hat{Y}_i)_{i=1}^n$,
    the sets $\Gacnf$, $\Galoc$ and $\Gacls$ 
    ,
    and the losses
    $L_i^{\mathrm{cnf}}$, $L_i^{\mathrm{loc}}$, 
    $L_i^{\mathrm{cls}}$ 
    using Algorithm~\ref{alg:seqcrc-monotonizing-step1} (with the monotonization trick)
    \label{alg:line_lbcnf_plus}
\State $\lbcnf_- \gets$ Solve Eq.~\eqref{eq:seqcrc-lbcnf-min} using $(X_i,Y_i, \hat{Y}_i)_{i=1}^n$,
    the sets $\Gacnf$, $\Galoc$ and $\Gacls$
    ,
    and the losses
    $L_i^{\mathrm{cnf}}$, $L_i^{\mathrm{loc}}$, 
    $L_i^{\mathrm{cls}}$ 
    using Algorithm~\ref{alg:seqcrc-monotonizing-step1} (with the monotonization trick)

\Statex \textit{\underline{Step 3}: Compute matchings (see Section~\ref{sec:level-of-guarantee-and-matching})}
\For{$i = 1$ to $n$}
    \State $\hat{Y}_i \gets f(X_i)$
    \State $\pi_x \gets \pi(Y_i, \hat{Y}_i)$ 
\EndFor

\Statex \textit{\underline{Step 4}: Calibrate localization parameter}

\State $\lbloc_+ \gets$ Solve Eq.~\eqref{eq:seqcrc-lbj} 
using $(X_i,Y_i, \hat{Y}_i)_{i=1}^n$, $L_i^{\mathrm{loc}}$ and $\Galoc$ using Algorithm~\ref{alg:seqcrc-monotonizing-step2} (with the monotonization trick) \label{alg:calib-line-loc}

\Statex \textit{\underline{Step 5}: Calibrate classification parameter}
\State $\lbcls_+ \gets$ Solve Eq.~\eqref{eq:seqcrc-lbj} 
using $(X_i,Y_i, \hat{Y}_i)_{i=1}^n$, $L_i^{\mathrm{cls}}$ and $\Gacls$ using Algorithm~\ref{alg:seqcrc-monotonizing-step2} (with the monotonization trick)\label{alg:calib-line-cls}

\end{algorithmic}
\end{algorithm}

In Algorithm~\ref{alg:seqcrc-calibration} we present a high-level pseudo-code to compute the parameters $\lbcnfp,\lbcnfm,\lblocp,\lbclsp$ defined in \eqref{eq:seqcrc-lbcnf-plus}, \eqref{eq:seqcrc-lbcnf-min} and \eqref{eq:seqcrc-lbj}, while meeting the monotonicity requirements on the loss functions $\Lloc$ and $\Lcls$. (Recall from Section~\ref{sec:app-pred-sets-losses}, Remarks~\ref{rmk:nonmonotone-loc} and~\ref{rmk:nonmonotone-cls}, that these losses are not necessarily monotone in $\lbcnf$.) As mentioned in Section~\ref{sec:seqcrc-theory}, this is achieved via a \emph{monotonization trick} (e.g., \cite{Angelopoulos_2022_CRC}), i.e., by replacing $\Lloc_i(\lbcnf,\lbloc)$ and $\Lcls_i(\lbcnf,\lbcls)$ with the smallest upper bounds that are provably monotone, namely, $\sup_{\lambda' \geq \lbcnf} \Lloc_i(\lambda',\lbloc)$ and $\sup_{\lambda' \geq \lbcnf} \Lcls_i(\lambda',\lbcls)$ respectively.

More specifically, Algorithm~\ref{alg:seqcrc-calibration} uses two optimization subroutines detailed in Appendix~\ref{app:algo}: Algorithm~\ref{alg:seqcrc-monotonizing-step1} for the confidence step, and Algorithm~\ref{alg:seqcrc-monotonizing-step2} for the localization/classification step. These two algorithms ``monotonize'' the underlying losses \emph{on the fly} (while optimizing the risk), to improve computational and memory efficiency.

\textbf{Requirements}
In these algorithms, we assume the user has already trained an object detector $f$. We further suppose that they have chosen losses, prediction sets, error levels for all three steps, and a matching procedure $\pi$, as described in Sections~\ref{sec:level-of-guarantee-and-matching} through \ref{sec:seqcrc-od-classification}.

\textbf{Edge cases} 
We must also handle two particular edge cases: an image may contain no ground-truth objects (i.e., $y$ is empty), or no predictions after the confidence thresholding as in Eq.~\eqref{eq:cnf-pred-set-sequence}.
In the first case, we set the loss to zero, and in the second case, to $B$ (in this section, for all losses we have $B=1$).

Algorithm~\ref{alg:seqcrc-inference} contains the procedure to follow for an inference on a test image. Using the same prediction set definitions as in calibration, and the parameters $\lbcnf_+, \lbloc_+, \lbcls_+$ obtained from it, it builds prediction sets that satisfy the conformal guarantee of Theorem~\ref{th:odcrc}. 
The process is considerably simpler than calibration, as it solely involves constructing the prediction sets. This procedure is computationally inexpensive for all the prediction sets and losses introduced and can be regarded as negligible when compared to the cost of a forward pass of~$f$. 
Importantly, only the $\lambda_+ $ values are used during inference.

\begin{algorithm}[t]
\caption{SeqCRC: Inference}
\label{alg:seqcrc-inference}
\setstretch{1.2}
\begin{algorithmic}[1]
\Require Test image $X_\mathrm{test}$, object detector $f$ (including NMS), calibrated parameters $\lbcnf_+, \lbloc_+, \lbcls_+$ and defined prediction sets $\Gacnf, \Galoc, \Gacls$
\Ensure Prediction sets $\Gamma^\mathrm{cnf}, \Gamma^\mathrm{loc}, \Gamma^\mathrm{cls}$

\Statex \textit{\underline{Step 1}: Confidence filtering}
\State $\hat{Y}_\mathrm{test} \gets f(X_\mathrm{test})$ 
\State $\Gamma^\mathrm{cnf} \gets ((\bhat_k, \hat{c}_k, \hat{o}_k) \in \hat{Y}_\mathrm{test} : \hat{o}_k \geq 1 - \lbcnf_+)$ (Eq.~\ref{eq:cnf-pred-set-sequence})

\Statex \textit{\underline{Step 2}: Localization and classification prediction sets}
\For{$k$ in $1\dots\abs{\Gamma^\mathrm{cnf}_{\lbcnfp}(\Xtest)}$}
    \State Compute $\Gamma^\mathrm{loc}\gets\Gamma^\mathrm{loc}_{\lbcnfp,\lblocp}(\Xtest)_k$ as chosen in Section~\ref{sec:seqcrc-od-localization}
    \State Compute $\Gamma^\mathrm{cls}\gets\Gamma^\mathrm{cls}_{\lbcnfp, \lbclsp}(\Xtest)_k$ as chosen in Section~\ref{sec:seqcrc-od-classification}
\EndFor

\end{algorithmic}
\end{algorithm}

\section{Experiments}
\label{sec:experiment}

In this section, we present a comprehensive set of experiments designed to assess the performances of our SeqCRC approach to \acf{OD}. 
These experiments include large-scale evaluations that result in a thorough benchmark, encompassing all matchings, losses, and prediction sets introduced in Sections~\ref{sec:level-of-guarantee-and-matching} through \ref{sec:seqcrc-od-classification}. 

Our goal is to provide insight into the practical interpretation and implications of our conformal method, highlighting both its strengths and limitations, while clarifying the practical process leading to statistical guarantees.

\begin{table}
\centering
\caption{Comparison of set sizes and risks for all losses and prediction sets from Section~\ref{sec:app-pred-sets-losses}, using DETR-101 as a reference model. The results are split into four categories: Confidence (Cnf.) and Localization (Loc.) losses, as well as Localization (Loc.) and Classification (Cls.) Prediction Sets (Pred. Sets). The risk levels are $\alpha^\mathrm{cnf}=0.02$ for the confidence task, $\alphaloc=0.05$ for localization, and $\alphacls=0.05$ for classification. 
Thus, we apply the SeqCRC method
at global level $\alphatot=\alphaloc+\alphacls=0.1$ (see Corollary~\ref{corollary_sum}).
For each setting, the results are shown using optimal choices for all remaining parameters.}
\label{tab:combined_losses_pred_sets_detr_01}
\begin{tabular}{lllrR{1cm}R{1cm}}
&& Setting & Set Size & Task Risk & Global Risk \tabularnewline
\midrule
\multirow{5}{*}{\rotatebox[origin=c]{90}{Loss}}& \multirow{2}{*}{\rotatebox[origin=c]{90}{Cnf.}} 
& box\_count\_threshold & 25.588 & 0.022 & 0.086 \tabularnewline
&& box\_count\_recall & 17.778 & 0.019 & 0.085 \tabularnewline
\cmidrule(r){2-6}
&\multirow{3}{*}{\rotatebox[origin=c]{90}{Loc.}}
& thresholded & 1.552 & 0.046 & 0.097 \tabularnewline
&& boxwise & 1.504 & 0.049 & 0.097 \tabularnewline
&& pixelwise & 1.043 & 0.047 & 0.096 \tabularnewline
\midrule
\multirow{4}{*}{\rotatebox[origin=c]{90}{Pred. Sets}}&
\multirow{2}{*}{\rotatebox[origin=c]{90}{Loc.}} 
& additive & 1.047 & 0.052 & 0.100 \tabularnewline
&& multiplicative & 1.043 & 0.047 & 0.096 \tabularnewline
\cmidrule(r){2-6}
&\multirow{2}{*}{\rotatebox[origin=c]{90}{Cls.}} 
& aps & 1.007 & 0.050 & 0.082 \tabularnewline
&& lac & 0.994 & 0.051 & 0.087 \tabularnewline
\end{tabular}
\end{table}

\begin{figure}[t]
    \centering
    \includegraphics[width=.7\columnwidth]{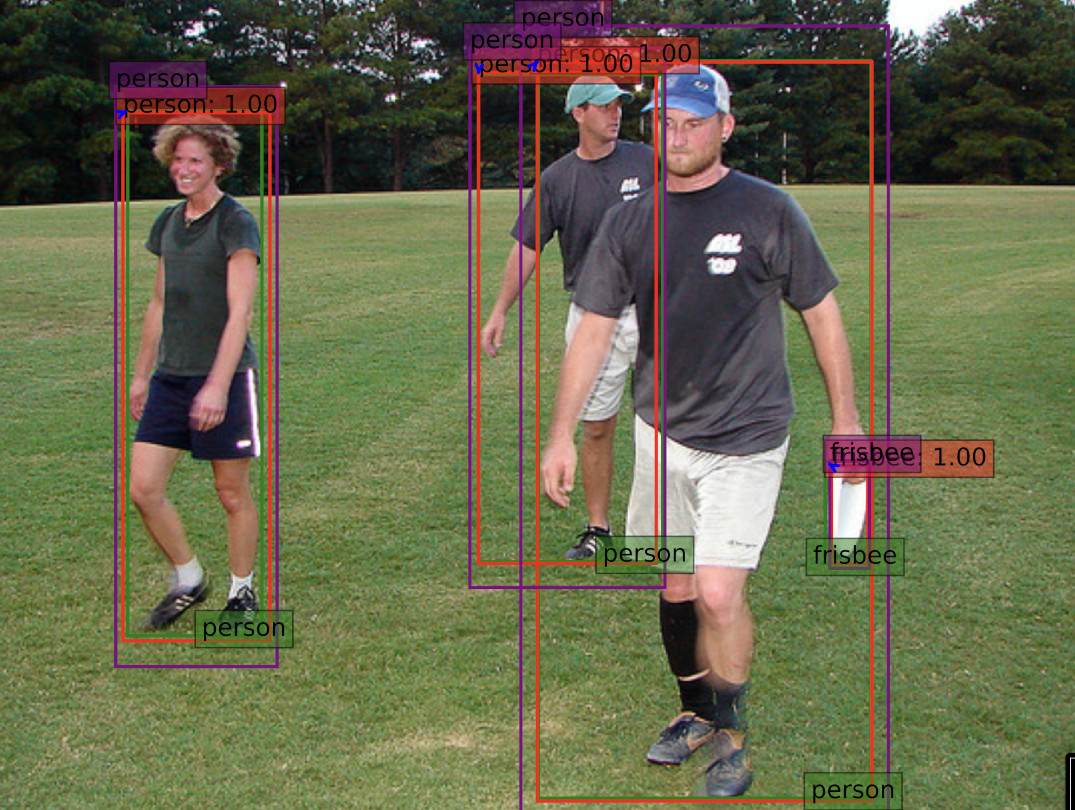}
    \caption{Example of a successful conformal prediction. 
    The model's {\color[HTML]{FF0000}predictions} are close to the {\color[HTML]{008001}ground truths}, and the {\color[HTML]{800080}conformal predictions} for the localization and classification tasks both include their respective ground truths.
    }
    \label{fig:frisbee}
\end{figure}

\subsection{Experimental Setup}
\label{sec:experimental-setup}
In these experiments, we use two reference models: YOLOv8~\cite{Jocher_2023_YOLO_v8} (more precisely, YOLOv8x\footnote{\url{https://github.com/ultralytics/ultralytics}}, 68M parameters), a recent improvement in a long line of single-stage detectors including \cite{Redmon_2016_YOLO,Jocher_2020_YOLOv5_by_Ultralytics}, and DETR~\cite{Carion_2020_detr} (DETR-101\footnote{\url{https://github.com/facebookresearch/detr}}, 60M parameters), a well-known transformer-based predictor.
It is important to note that while all results here are based on these models,
the method we present is \textbf{model-agnostic} and can be applied straightforwardly to virtually any other model (following our setting in Section~\ref{sec:od-uq}). The choice of these models emphasizes the broad applicability of our new framework, as well as the potential advantages or shortcomings of some families of models.

We employ the MS-COCO dataset~\cite{Lin_2014_COCO}, a standard benchmark in the field. 
For our calibration and inference procedures, we use the validation split of this dataset. 
This approach is common when an annotated dataset that is independent of training and model selection is unavailable~\cite{Angelopoulos_2022_CRC,angelopoulos2021learn}.
While potentially breaking Assumption~\ref{ass:cp-independence}, the practical availability constraints often
make the validation split the best choice
for conformal calibration and evaluation. 
The validation data is then split into two equal segments of $n=2500$ instances each, dedicated to calibration and inference respectively. 

\subsubsection{Methods}
We conduct experiments on MS-COCO for both models (YOLO8 and DETR) and for all combinations of losses, distance functions (matchings), and prediction sets introduced in Section~\ref{sec:app-pred-sets-losses}. 
For the matching ``mix'' referring to distance \eqref{eq:match_dist_mix}, we set $\tau=0.25$, providing a fair balance between localization and classification distances, according to preliminary experiments.
For all experiments, the IoU threshold for the NMS is set to $0.5$.
We focus on the setting of Corollary~\ref{corollary_sum} that offers a joint guarantee on the localization and classification tasks (hence $\alphatot = \alphacls + \alphaloc$ here). 
We consider two global error levels : $\alphatot=0.1$ and $\alphatot=0.2$. For the former we set $\alphacnf=0.02$, $\alphaloc=0.05$ and $\alphacls=0.05$ for confidence, localization and classification respectively. Similarly for the latter we set $\alphacnf=0.03$, $\alphaloc=0.1$, $\alphacls=0.1$. 
Finally, to mitigate the pathological second-step loss values towards $\lbcnf=1$ (i.e. a null confidence threshold $1-\lbcnf$) as illustrated in Fig.~\ref{fig:monotonicity_loc}, we filter out predictions with confidence lower than $10^{-3}$. This threshold is chosen independently of the data, for all experiments. (This thus corresponds to refining Steps 1 and 2 in Fig.~\ref{fig:first_fig}, and does not break the guarantees of Theorem~\ref{th:odcrc}.) This leads to much lighter correction when applying the monotonization trick and therefore to smaller set sizes.

\subsubsection{Metrics} \label{sec:experiment_metrics}

For evaluation, we follow a common practice in Conformal Prediction, that of reporting the two metrics of \textbf{Empirical Risk} and \textbf{Set Size}. 
We examine their relationships with \od metrics in Appendix~\ref{app:further_experiments}. 
Throughout this section, we focus on the test data sequence $ (X_{\mathrm{test},i}, Y_{\mathrm{test},i})_{i=1}^{n_\mathrm{test}} $.

The \emph{empirical risk} (or simply \emph{risk}) measures the average loss on test data. For binary losses, this equals $1$ minus the coverage, a standard metric in Conformal Prediction.
(In the case of a binary loss, it corresponds to $1$ minus the coverage, one of the main metrics in Conformal Prediction.)
For each task $j \in \{ \mathrm{cnf,loc,cls}\}$, the {empirical risk} that we measure is
\[
j\text{-Risk:}\quad \left\{\begin{array}{ll}
     \frac{1}{n_{\mathrm{test}}} \sum\limits_{i=1}^{n_{\mathrm{test}}}L^j_\mathrm{test,i}(\lbcnfp) & \mathrm{if~} j=\mathrm{cnf} \\
     \frac{1}{n_{\mathrm{test}}} \sum\limits_{i=1}^{n_{\mathrm{test}}}L^j_\mathrm{test,i}(\lbcnfp,\lambda^j_+) & \mathrm{if~} j \in \{ \mathrm{loc,cls}\}
\end{array}\right.
\label{eq:metric-risk}
\]
The $j$-Risk can be compared against the target value $\alpha^j$, for each task $j \in \{ \mathrm{cnf,loc,cls}\}$ individually. 
However, a \textbf{Global Risk} can be computed by
\[
\frac{1}{n_{\mathrm{test}}} \sum_{i=1}^{n_{\mathrm{test}}}\max\left\{\Lloc_{\mathrm{test},i}(\lbcnfp, \lblocp), \Lcls_{\mathrm{test},i}(\lbcnfp, \lbclsp)\right\}.
\label{eq:metric-global-risk}
\]
This corresponds to the joint risk controlled in Corollary~\ref{corollary_sum} (therefore compared against $\alphatot$), and it is our main target. As mentioned after the corollary, one may include the confidence risk if desired; in our experiments, we focus on localization and classification only.

Secondly, for each task (including confidence), we measure the mean set size, averaged over images. As the sets take various shapes depending on the tasks, we introduce three different metrics, tailored for interpretability. 

First, set sizes of confidence sets are defined by
\[
\text{Confidence Set Size:}\quad
\frac{1}{n_{\mathrm{test}}}\sum_{i=1}^{n_{\mathrm{test}}}\abs{\Gamma^\mathrm{cnf}_{\lbcnfp}(X_{\mathrm{test},i})}.
\]
This metric measures the average number of boxes output by our procedure per image.
Then, for localization, the metric we introduce is inspired from \cite{Andéol_2023_confident}; it is given by the ratio of the conformal box' area over that of its original prediction (referred to as \emph{Stretch} in the source). More precisely:
\[
    \text{Localization Set Size:}\quad
    \frac{1}{n_{\mathrm{test}}}\sum_{i=1}^{n_{\mathrm{test}}}\frac{1}{n_{\mathrm{test},i}}\sum_{\bhat_k, \bhat_{k}^{\lblocp}}  \sqrt{\frac{\mathrm{area}(\bhat_{k}^{\lblocp})}{\mathrm{area}(\bhat_k)}},
\]
where $n_{\mathrm{test},i} = \abs{\Gamma^\mathrm{cnf}_{\lbcnfp}(X_{\mathrm{test},i})}$ is the (final) number of predicted boxes on the $i$-th test image, and where $\bhat_k$ and $\bhat_{k}^{\lblocp}$ are the $k$-th elements of $\Gamma^\mathrm{cnf}_{\lbcnfp}(X_{\mathrm{test},i})$ and $\Gamma^\mathrm{loc}_{\lbcnfp,\lblocp}(X_{\mathrm{test},i})$ respectively---in other words, the same $k$-th predicted box before and after altering its size with the $\lblocp$ parameter (e.g. an additive margin).

Finally, the set size for classification prediction sets is defined (similarly as above) by averaging over two levels: within a test image, and then over the test set;
\[
\text{Classification Set Size:}\quad
\frac{1}{n_\mathrm{test}}\sum_{i=1}^{n_\mathrm{test}}\frac{1}{n_{\mathrm{test},i}}\sum_{\hat{c}_{k}^{\lbclsp}} \abs{\hat{c}_{k}^{\lbclsp}},
\]
where $\hat{c}_{k}^{\lbclsp}=\Gamma^\mathrm{cls}_{\lbcnfp,\lbclsp}(X_{\mathrm{test},i})_k$ denotes the classification prediction set for the $k$-th prediction on image $X_{\mathrm{test},i}$.

\begin{table}
\centering
\caption{Effect of target $\alphatot$ level on set sizes for Confidence (Cnf.), Localization (Loc.) and Classification (Cls.) for four matching functions: the Generalized IoU (GIoU), the asymmetric signed Hausdorff divergence (Hausdorff), the Least-Ambiguous Classifier non-conformity score (LAC), and a mixture of the last two (Mix), the best performing overall.}
\label{tab:alpha_effect_sets_per_loss}
\begin{tabular}{lrrrr}
Matching & Target $\alphatot$ & Cnf. Size & Loc. Size & Cls. Size \\
\midrule
\multirow{2}{*}{GIoU} & 0.1 & 17.778 & 28.241 & 44.471 \\
& 0.2 & 14.046 & 23.690 & 32.335 \\
\midrule
\multirow{2}{*}{Hausdorff} & 0.1 & 25.588 & 1.043 & 41.846 \\
& 0.2 & 14.046 & 0.999 & 22.035 \\
\midrule
\multirow{2}{*}{LAC} & 0.1 & 25.588 & 14.147 & 0.994 \\
& 0.2 & 22.657 & 7.786 & 0.653 \\
\midrule
\multirow{2}{*}{Mix} & 0.1 & 25.588 & 1.334 & 8.228 \\
& 0.2 & 22.657 & 1.018 & 0.931 \\
\end{tabular}
\end{table}

\begin{table}
\centering
\caption{Comparison of set sizes for all losses and prediction sets from Section~\ref{sec:app-pred-sets-losses} for two models: DETR-101 and YOLOv8x. The comparison is done for a global error level $\alphatot=0.1$. The set size reported are the ones of the associated task in its part of the table.}
\label{tab:detr_yolo_set_sizes}
\begin{tabular}{lllrrrr}
&& Metric & \multicolumn{2}{c}{DETR} & \multicolumn{2}{c}{YOLOv8}\\
&&          & Risk & Set Size & Risk & Set Size\\
\midrule
\multirow{5}{*}{\rotatebox[origin=c]{90}{Loss}}&\multirow{2}{*}{\rotatebox[origin=c]{90}{Cnf.}} 
& {\scriptsize box\_count\_threshold} & 0.022 & 25.588 & 0.012 & 18.855 \\
&& {\scriptsize box\_count\_recall} & 0.019 & 17.778 & 0.019 & 11.710 \\
\cmidrule(r){2-7}
&\multirow{3}{*}{\rotatebox[origin=c]{90}{Loc.}} 
& {\scriptsize thresholded} & 0.046 & 1.552 & 0.047 & 5.451 \\
&& {\scriptsize boxwise} & 0.049 & 1.504 & 0.046 & 5.225 \\
&& {\scriptsize pixelwise} & 0.047 & 1.043 & 0.049 & 3.867 \\
\midrule
\multirow{4}{*}{\rotatebox[origin=c]{90}{Pred. Sets}}&\multirow{2}{*}{\rotatebox[origin=c]{90}{Loc.}} 
& {\scriptsize additive} & 0.052 & 1.047 & 0.049 & 3.867 \\
&& {\scriptsize multiplicative} & 0.047 & 1.043 & 0.052 & 4.181 \\
\cmidrule(r){2-7}
&\multirow{2}{*}{\rotatebox[origin=c]{90}{Cls.}} 
& {\scriptsize aps} & 0.050 & 1.007 & 0.043 & 0.996 \\
&& {\scriptsize lac} & 0.051 & 0.994 & 0.049 & 0.717 \\
\end{tabular}
\end{table}

\subsection{COD Toolkit}

To emphasize the reproducibility and reusability of our work, we introduce a comprehensive toolkit, 
which we refer to as the ``Conformal Object Detection Toolkit,''\footnote{
    \url{https://github.com/leoandeol/cods}.
}
or \emph{COD} for short. 
This open-source repository includes several widely used object detectors, including, but not limited to, YOLO and DETR.

It includes the procedures outlined in this paper, multiple loss functions, and a suite of visualization tools for clear understanding and interpretation of results.
Figure~\ref{fig:frisbee} shows an example, where the ground truth, prediction, and corrected bounding boxes are colored green, red, and purple, respectively.
The class label, top-1 predicted label, and conformal sets are also displayed.

The COD toolkit is implemented in Python and PyTorch, compatible with both CPU and GPU.
Finally, it supports the integration of new procedures, prediction sets, and loss functions, enabling extensive benchmarks and further comparisons in future works.

\begin{figure}
    \centering
    \includegraphics[width=.7\columnwidth]{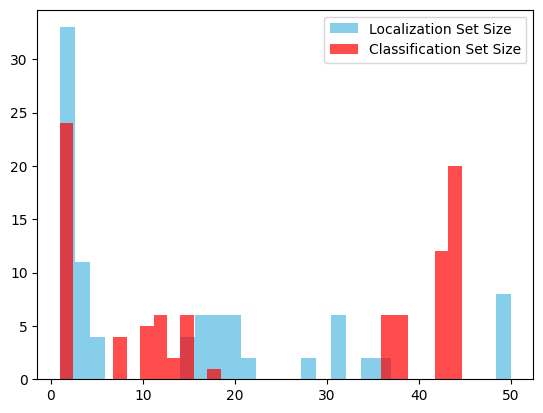}
    \caption{Histogram of set sizes for all combinations of losses, prediction sets and matching functions for $\alphatot=0.1$, using the detector DETR. In {\color[HTML]{87ceeb}blue}
    the conformal set sizes in Localization, and in {\color[HTML]{ff4c4c}red}, the classification ones (see Section \ref{sec:experiment_metrics}). While the two have different units, their scale is similar, and, in distribution, both display a heavy tail. A large share of experiment settings lead to dramatically worse performance (= set size) at inference.
    }
    \label{fig:set-size-histogram}
\end{figure}

\subsection{Benchmark \& Results}
\label{sec:benchmark-results}

The benchmark consists of a large-scale study of all combinations of the losses, prediction sets, matching functions and more (see Section~\ref{sec:experimental-setup}).
This results in the largest benchmark in Conformal Object Detection, with hundreds of individual runs. Key results are presented in several tables and figures in this section, while complete results are provided in Appendix~\ref{app:complete_benchmark}.

\textbf{Impact of Component Choice}.
We begin by examining Table~\ref{tab:combined_losses_pred_sets_detr_01}. 
For each loss and prediction set introduced in Section~\ref{sec:app-pred-sets-losses}, we present the best experimental result (i.e., the reported experiment being the best combination of the remaining parameters). 
It does not include a section for both classification losses, and confidence prediction sets, as only one of each is introduced in Section~\ref{sec:app-pred-sets-losses}.
The table is organized as follows: it is split into two, with the upper part comparing different losses, and the lower part about prediction sets. These are then subdivided into tasks, where we compare losses or prediction sets specific to each task.
For instance, the first line reads: for the box\_count\_threshold confidence loss, and the best remaining hyperparameters (other losses, matching, etc), the confidence set size is 25.588, with an empirical test risk of 0.022. The global (joint localization and classification) empirical risk is 0.086.
The experiments are conducted with DETR-101, in the $\alphatot=0.1$ setting.
For virtually all tasks and global risk levels, the risk is controlled, i.e., below its corresponding risk level $\alphatot$ or $\alpha^j$. A few experiments showcase a risk slightly above their respective $\alpha$, likely due to randomness from sampling. 
We also observe that the set sizes are sufficiently small and therefore practically informative for each loss. 
For example, in classification, when the classes in the set are semantically close, or induce a similar response, a set size of $3$ can be considered informative. For example, a set containing pedestrian, animal and cyclist for autonomous driving.
The choice of loss function significantly affects the resulting set sizes, particularly for the losses presented here. 
The ``pixelwise'' loss, a relaxation of the previous two, produces smaller corrections at the cost of only slightly weaker guarantees (e.g., a guarantee on containing most of the ground truth rather than all of it), confirming results from \cite{Andéol_2023_confident}. Similarly, the relaxed confidence loss ``box\_count\_recall'' leads to the same conclusion.

However, these losses are compared using the best possible remaining parameters.
While presented set sizes are reasonably small, this is not true for all settings. 
Figure~\ref{fig:set-size-histogram} shows the histogram of set sizes for localization and classification across numerous experiments, where each element represents a unique combination of losses, prediction sets, matching functions, and models. 
The histogram reveals a heavy tail: although a significant proportion of experiments produce informative set sizes (leftmost bar), a notable share results in excessively large set sizes, including bounding boxes as large as the image and conformal classification sets with dozens of classes.
This implies, that for a desired error level $\alphatot$, the choice of losses can lead to either informative, or overly large --unusable-- sets.

{\textbf{Matching Functions \& Error Levels}.}
We now look at Table~\ref{tab:alpha_effect_sets_per_loss}.
It compares, for each 4 proposed matching functions in Section~\ref{sec:level-of-guarantee-and-matching}, and two risk levels $\alphatot$, the size metrics of the three tasks, for the best performing experiment in this setting (i.e. best losses and prediction sets for each $\alphatot$ and matching function).
We observe that the confident set size is constant for all matching functions for $\alphatot=0.1$.
This is due, in our setting, to the \emph{maximum} of the three risks in Eq.~\eqref{eq:seqcrc-lbcnf-plus} being always the confidence risk. Indeed, when $\alphacnf$ is not too high, at least a box is selected on most images and therefore by definition, the localization and classification risks in this step are null (due to $\bar{\lambda}$).
However, for $\alphatot=0.2$, we observe two confident set sizes, which is due to the best experiment (in the sense of localization and classification set sizes) relying on different confidence losses depending on the matching.
We further observe that localization and classification set sizes are respectively minimized with the ``Hausdorff'' and ``LAC metrics'', while ``Mix'' produces an interesting trade-off of the two. However, the GIoU matching, while more \od-inspired, suffers from poor performance across the board, due to the lack of connection to the subsequent losses, that use the matching. Finally, in the ``Mix'' setting, a large jump in classification set size can be noted. Indeed, as observed in other experiments in Appendix~\ref{app:complete_benchmark}, the impact of a risk level $\alpha$ to its associated task's resulting set size is not smooth, it can greatly vary for a small change in target risk. 

{\textbf{Model-Agnosticism}.}
We now look at Table~\ref{tab:detr_yolo_set_sizes} which compares set sizes and empirical risks for DETR-101 and YOLOv8x. It is built similarly to Table~\ref{tab:combined_losses_pred_sets_detr_01}, introduced previously.
We observe that for both model, the risk is controlled for all experiments, that is below the corresponding risk level $\alphatot$ or $\alpha^j$.\footnote{Using APS results in a looser bound, as discussed in \cite{Angelopoulos_2021_RAPS}.} This result is a direct consequence of the model-agnosticism of conformal prediction and therefore of SeqCRC -- it is applicable regardless of the performance or size of the model. However, the set size depends on the quality of the underlying model. 
Furthermore, there is a trade-off in terms of number, and size of (conformal) bounding boxes. 
We observe here that using YOLOv8 tends to lead to much fewer selected bounding boxes, but with much larger corrections in localization: with the same hyperparameters used for calibration, the resulting trade-off differs.

\subsection{Discussion}
\label{sec:exp-discussion}

{\textbf{Risk Control}.}
In all our experiments, the risk meets or is very close to the target value, which empirically confirms Theorem~\ref{th:odcrc}. 
However, the global risk from Corollary~\ref{corollary_sum} is often unnecessarily conservative (see Table~\ref{tab:combined_losses_pred_sets_detr_01} and Appendix~\ref{app:complete_benchmark}), suggesting that its correction could be refined.

{\textbf{Tradeoff: Guarantee Strength versus Set Size}.}
As shown in Tables~\ref{tab:combined_losses_pred_sets_detr_01} and~\ref{tab:detr_yolo_set_sizes}, relaxed losses such as the “pixelwise” localization loss yield much smaller prediction sets, but at the cost of weaker guarantees. Unlike the stricter losses, which ensure full object coverage, the pixelwise variant only guarantees that a proportion of the ground-truth area is covered (i.e., missing a few pixels is not heavily penalized).
This tradeoff is often acceptable, but in some safety-critical applications (e.g., \cite{zouzou2025robustvisionbasedrunwaydetection}) full-object guarantees remain necessary.

{\textbf{Strengths and Limitations in Practice}.}
Qualitatively, many examples (e.g., Fig.~\ref{fig:frisbee}) show accurate object counts, predictions that are fully contained in conformal boxes, and correct singleton classification sets.
Limitations, however, also appear: in Fig.~\ref{fig:zebras}, while the two zebras are well predicted, an additional large box includes both animals without ground-truth support, which reveals two issues.

First, our approach guarantees a large enough \emph{recall} but only empirically limits false positives (via the confidence thresholding step).
In fact, precision is a non-monotone function of its parameters, and monotonizing it using our approach could lead to too large risks for a solution to exist.

Second, MS-COCO annotations are inconsistent: individual objects may be labeled alongside a larger enclosing box, and sometimes include unlabeled instances (see Appendix~\ref{app:further_experiments}). Models trained on such data reproduce this variability, generating group boxes unpredictably.

Finally, we recall that conformal prediction brings a guarantee on the loss of the prediction \emph{with regards to} the ground truth. 
If the ground truth is incorrect or inconsistent (e.g., due to annotation choices or labeling errors), and the model in fact predicts the object more accurately than the labels, then conformal prediction must impose very large corrections in order to achieve small risk. These corrections could inflate the prediction sets to the point of becoming practically unusable.

\begin{figure}
    \centering
    \includegraphics[width=.7\columnwidth]{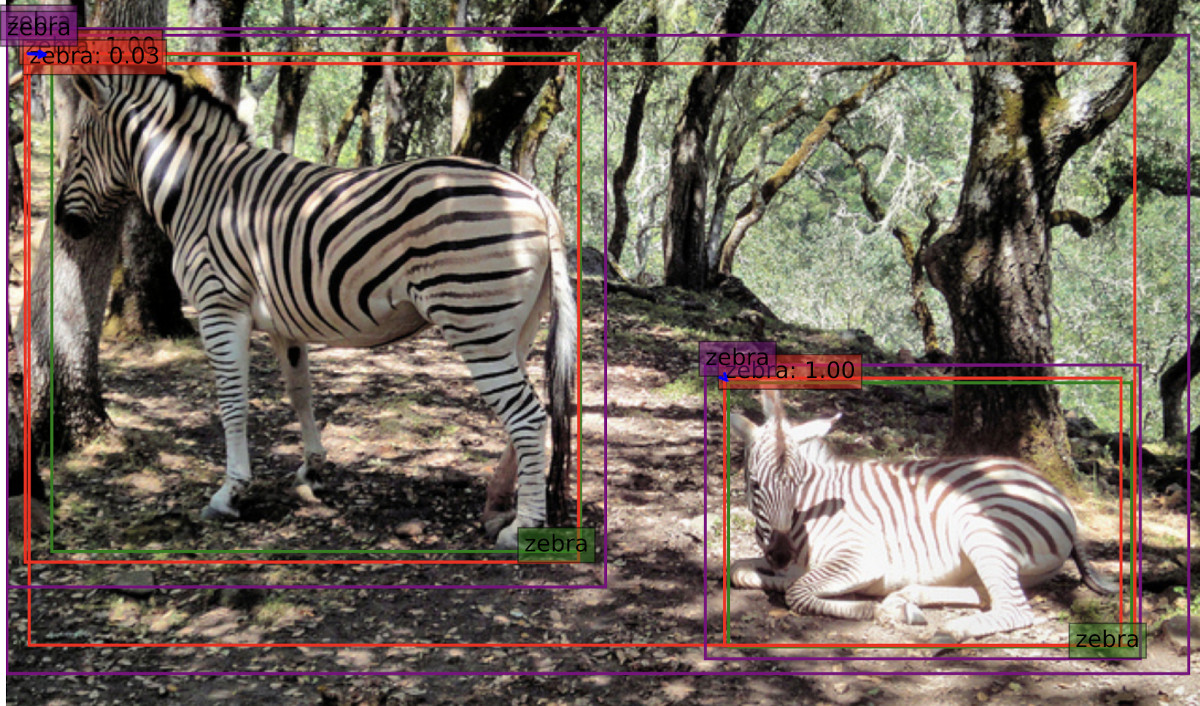}
    \caption{Example of an annotation/model artifact. 
    Beyond the two correctly detected zebras, an additional (larger) prediction includes both animals, but does not correspond to any ground truth. 
    We argue this is due to artifacts in the annotation of MS-COCO, which are reproduced by the pretrained model.}
    \label{fig:zebras}
\end{figure}

\section{Conclusion \& Future Works}

In this work, we address the critical challenge of building trustworthy object detection (OD) systems. We focus on a family of statistically valid uncertainty quantification methods: conformal risk control. Though this topic had previously been studied in some restricted OD cases (for a single model, or a sub-task of \od), we provide the first framework to build guaranteed prediction sets for all common tasks in OD.

Our main contribution is the new statistical method \emph{Sequential Conformal Risk Control} (SeqCRC). 
Unlike previous works using \crc, SeqCRC allows to control risks associated with two sequential tasks using two \emph{dependent} parameters. In OD this provides a statistically sound way to tune a confidence threshold as well as classification and localization ``error margins''. 

We study various ways to measure prediction errors (conformal loss functions), both old and new, and discuss their properties as well as other key hyperparameters (e.g., matching distance). We test our framework through extensive experiments, which confirm that our SeqCRC method effectively controls different risks and produces small enough set sizes.

To support the research community, we provide an open-source codebase that we used for our experiments. It includes models and procedures, enabling development of new techniques, error metrics, and associated guarantees. We hope this codebase and our extensive experiments can be used for future benchmarks in Conformal Object Detection.

Our analysis identifies key directions for future research.
For instance, designing new confidence loss functions that better balance valuable guarantees with smaller (more informative)  prediction sets would be key in practice. Other examples mentioned in Section~\ref{sec:exp-discussion} are about extending the framework to control false positives (precision), or improving its robustness against inconsistencies in ground-truth annotations. 

\section*{Acknowledgements}

Our work has benefitted from the AI Cluster ANITI and the research program DEEL. ANITI is funded by the France 2030 program under the Grant agreement n°ANR-23-IACL-0002. DEEL is an integrative program of the AI Cluster ANITI, designed and operated jointly with IRT Saint Exupéry, with the financial support from its industrial and academic partners and the France 2030 program under the Grant agreement n°ANR-10-AIRT-01.

A. Mazoyer has been supported by the project ROMEO (ANR-21-ASIA-0001) from the ASTRID joint program of the French National Research Agency (ANR) and the French Innovation and Defence Agency (AID)

{
    \bibliographystyle{ieeetr}
    \bibliography{bibliography}

\begin{thebibliography}{10}

\bibitem{feng2022autonomous}
D.~Feng, A.~Harakeh, S.~L. Waslander, and K.~Dietmayer, ``A review and comparative study on probabilistic object detection in autonomous driving,'' {\em IEEE Transactions on Intelligent Transportation Systems}, vol.~23, no.~8, pp.~9961--9980, 2022.

\bibitem{yang2021artificial}
R.~Yang and Y.~Yu, ``Artificial convolutional neural network in object detection and semantic segmentation for medical imaging analysis,'' {\em Frontiers in oncology}, vol.~11, p.~638182, 2021.

\bibitem{Jenn_2020_identifying}
E.~Jenn, A.~Albore, F.~Mamalet, G.~Flandin, C.~Gabreau, H.~Delseny, A.~Gauffriau, H.~Bonnin, L.~Alecu, J.~Pirard, {\em et~al.}, ``Identifying challenges to the certification of machine learning for safety critical systems,'' in {\em European congress on embedded real time systems (ERTS 2020)}, 2020.

\bibitem{mamalet_2021_white}
F.~Mamalet, E.~Jenn, G.~Flandin, H.~Delseny, C.~Gabreau, A.~Gauffriau, B.~Beaudouin, L.~Ponsolle, L.~Alecu, H.~Bonnin, B.~Beltran, D.~Duchel, J.-B. Ginestet, A.~Hervieu, S.~Pasquet, K.~Delmas, C.~Pagetti, J.-M. Gabriel, C.~Chapdelaine, S.~Picard, M.~Damour, C.~Cappi, L.~Gard{\`e}s, F.~D. Grancey, B.~Lefevre, S.~Gerchinovitz, and A.~Albore, ``{White Paper Machine Learning in Certified Systems},'' research report, {IRT Saint Exup{\'e}ry ; ANITI}, Mar. 2021.

\bibitem{Alecu_2022_can}
L.~Alecu, H.~Bonnin, T.~Fel, L.~Gardes, S.~Gerchinovitz, L.~Ponsolle, F.~Mamalet, {\'E}.~Jenn, V.~Mussot, C.~Cappi, K.~Delmas, and B.~Lefevre, ``{Can we reconcile safety objectives with machine learning performances?},'' in {\em {ERTS}}, June 2022.

\bibitem{Poretschkin_2023_guideline}
M.~Poretschkin, A.~Schmitz, M.~Akila, L.~Adilova, D.~Becker, A.~B. Cremers, D.~Hecker, S.~Houben, M.~Mock, J.~Rosenzweig, J.~Sicking, E.~Schulz, A.~Voss, and S.~Wrobel, ``{Guideline for Trustworthy Artificial Intelligence -- AI Assessment Catalog},'' 2023.

\bibitem{miller2018dropout}
D.~Miller, L.~Nicholson, F.~Dayoub, and N.~S{\"u}nderhauf, ``Dropout sampling for robust object detection in open-set conditions,'' in {\em IEEE International Conference on Robotics and Automation}, pp.~3243--3249, 2018.

\bibitem{Harakeh_2020_bayesod}
A.~Harakeh, M.~Smart, and S.~L. Waslander, ``Bayesod: A bayesian approach for uncertainty estimation in deep object detectors,'' in {\em Proceedings of ICRA}, 2020.

\bibitem{sharifuzzaman2024bayes}
S.~A. Sharifuzzaman, J.~Tanveer, Y.~Chen, J.~H. Chan, H.~S. Kim, K.~D. Kallu, and S.~Ahmed, ``Bayes r-cnn: An uncertainty-aware bayesian approach to object detection in remote sensing imagery for enhanced scene interpretation,'' {\em Remote Sensing}, vol.~16, no.~13, p.~2405, 2024.

\bibitem{schubert2021metadetect}
M.~Schubert, K.~Kahl, and M.~Rottmann, ``Metadetect: Uncertainty quantification and prediction quality estimates for object detection,'' in {\em 2021 International Joint Conference on Neural Networks}, pp.~1--10, 2021.

\bibitem{meyer2020learning}
G.~P. Meyer and N.~Thakurdesai, ``Learning an uncertainty-aware object detector for autonomous driving,'' in {\em 2020 IEEE/RSJ International Conference on Intelligent Robots and Systems (IROS)}, pp.~10521--10527, IEEE, 2020.

\bibitem{zhang2024harnessing}
R.~Zhang, H.~Zhang, H.~Yu, and Z.~Zheng, ``Harnessing uncertainty-aware bounding boxes for unsupervised 3d object detection,'' {\em arXiv preprint arXiv:2408.00619}, 2024.

\bibitem{Vovk_2005_algorithmic}
V.~Vovk, A.~Gammerman, and G.~Shafer, {\em Algorithmic learning in a random world}.
\newblock Springer, 2005.

\bibitem{deGrancey_2022_object}
F.~de~Grancey, J.-L. Adam, L.~Alecu, S.~Gerchinovitz, F.~Mamalet, and D.~Vigouroux, ``Object detection with probabilistic guarantees: A conformal prediction approach,'' in {\em Computer Safety, Reliability, and Security. SAFECOMP 2022 Workshops}, pp.~316--329, Springer, 2022.

\bibitem{Andéol_2023_confident}
L.~Andeol, T.~Fel, F.~de~Grancey, and L.~Mossina, ``Confident object detection via conformal prediction and conformal risk control: an application to railway signaling,'' in {\em Proceedings of the Twelfth Symposium on Conformal and Probabilistic Prediction with Applications}, vol.~204, pp.~36--55, PMLR, 2023.

\bibitem{Andéol_2023_conformal}
L.~And{\'e}ol, T.~Fel, F.~de~Grancey, and L.~Mossina, ``Conformal prediction for trustworthy detection of railway signals,'' {\em AI and Ethics}, vol.~4, no.~1, pp.~157--161, 2024.

\bibitem{Li_2022_towards}
S.~Li, S.~Park, X.~Ji, I.~Lee, and O.~Bastani, ``Towards pac multi-object detection and tracking,'' {\em arXiv preprint arXiv:2204.07482}, 2022.

\bibitem{Blot_2024_efficient}
V.~Blot, A.~L. de~Brionne, I.~Sellami, O.~Trassard, I.~Beau, C.~Sonigo, and N.~J. Brunel, ``Efficient precision control in object detection models for enhanced and reliable ovarian follicle counting,'' in {\em Uncertainty for Safe Utilization of Machine Learning in Medical Imaging-6th International Workshop}, 2024.

\bibitem{Angelopoulos_2022_gentle}
A.~N. Angelopoulos and S.~Bates, ``Conformal prediction: A gentle introduction,'' {\em Foundations and Trends® in Machine Learning}, vol.~16, no.~4, pp.~494--591, 2023.

\bibitem{Park_2020_PAC}
S.~Park, O.~Bastani, N.~Matni, and I.~Lee, ``Pac confidence sets for deep neural networks via calibrated prediction,'' in {\em International Conference on Learning Representations}, 2020.

\bibitem{Bates_2021_RCPS}
S.~Bates, A.~Angelopoulos, L.~Lei, J.~Malik, and M.~Jordan, ``Distribution-free, risk-controlling prediction sets,'' {\em J. ACM}, vol.~68, 9 2021.

\bibitem{angelopoulos2021learn}
A.~N. Angelopoulos, S.~Bates, E.~J. Cand{\`e}s, M.~I. Jordan, and L.~Lei, ``{Learn then test: Calibrating predictive algorithms to achieve risk control},'' {\em The Annals of Applied Statistics}, vol.~19, no.~2, pp.~1641 -- 1662, 2025.

\bibitem{mukama2024copula}
B.~C. Mukama, S.~Messoudi, S.~Rousseau, and S.~Destercke, ``Copula-based conformal prediction for object detection: a more efficient approach,'' {\em Proceedings of Machine Learning Research}, vol.~230, pp.~1--18, 2024.

\bibitem{timans2024adaptive}
A.~Timans, C.-N. Straehle, K.~Sakmann, and E.~Nalisnick, ``Adaptive bounding box uncertainties via two-step conformal prediction,'' in {\em Computer Vision--ECCV 2024}, pp.~363--398, Springer Nature Switzerland, 2025.

\bibitem{xu2024twoStageCRC}
Y.~Xu, M.~Ying, W.~Guo, and Z.~Wei, ``Two-stage conformal risk control with application to ranked retrieval,'' 2024.
\newblock Available as arXiv:2404.17769v3.

\bibitem{quach2023conformal}
V.~Quach, A.~Fisch, T.~Schuster, A.~Yala, J.~H. Sohn, T.~S. Jaakkola, and R.~Barzilay, ``Conformal language modeling,'' in {\em The Twelfth International Conference on Learning Representations}, 2024.

\bibitem{kuppers2022confidence}
F.~K{\"u}ppers, A.~Haselhoff, J.~Kronenberger, and J.~Schneider, ``Confidence calibration for object detection and segmentation,'' in {\em Deep Neural Networks and Data for Automated Driving: Robustness, Uncertainty Quantification, and Insights Towards Safety}, pp.~225--250, Springer, 2022.

\bibitem{riedlinger2023gradient}
T.~Riedlinger, M.~Rottmann, M.~Schubert, and H.~Gottschalk, ``Gradient-based quantification of epistemic uncertainty for deep object detectors,'' in {\em Proceedings of the IEEE/CVF Winter Conference on Applications of Computer Vision}, pp.~3921--3931, 2023.

\bibitem{riedlinger2022uncertainty}
T.~Riedlinger, M.~Schubert, K.~Kahl, and M.~Rottmann, ``Uncertainty quantification for object detection: output-and gradient-based approaches,'' in {\em Deep Neural Networks and Data for Automated Driving: Robustness, Uncertainty Quantification, and Insights Towards Safety}, pp.~251--275, Springer, 2022.

\bibitem{neal2012bayesian}
R.~M. Neal, {\em Bayesian learning for neural networks}, vol.~118.
\newblock Springer Science \& Business Media, 2012.

\bibitem{gal2015dropout}
Y.~Gal and Z.~Ghahramani, ``Dropout as a {B}ayesian approximation: Representing model uncertainty in deep learning,'' in {\em Proceedings of The 33rd International Conference on Machine Learning}, vol.~48, pp.~1050--1059, PMLR, 20--22 Jun 2016.

\bibitem{hinton1993keeping}
G.~E. Hinton and D.~Van~Camp, ``Keeping the neural networks simple by minimizing the description length of the weights,'' in {\em Proceedings of the sixth annual conference on Computational learning theory}, pp.~5--13, 1993.

\bibitem{kingma2013auto}
D.~P. Kingma, M.~Welling, {\em et~al.}, ``Auto-encoding variational bayes,'' 2013.

\bibitem{wilson2020bayesian}
A.~G. Wilson and P.~Izmailov, ``Bayesian deep learning and a probabilistic perspective of generalization,'' {\em Advances in neural information processing systems}, vol.~33, pp.~4697--4708, 2020.

\bibitem{li2016preconditioned}
C.~Li, C.~Chen, D.~Carlson, and L.~Carin, ``Preconditioned stochastic gradient langevin dynamics for deep neural networks,'' in {\em Proceedings of the AAAI conference on artificial intelligence}, vol.~30, 2016.

\bibitem{bodla2017soft}
N.~Bodla, B.~Singh, R.~Chellappa, and L.~S. Davis, ``Soft-nms--improving object detection with one line of code,'' in {\em Proceedings of the IEEE international conference on computer vision}, pp.~5561--5569, 2017.

\bibitem{hosang2017learning}
J.~Hosang, R.~Benenson, and B.~Schiele, ``Learning non-maximum suppression,'' in {\em Proceedings of the IEEE conference on computer vision and pattern recognition}, pp.~4507--4515, 2017.

\bibitem{Redmon_2016_YOLO}
J.~Redmon, S.~Divvala, R.~Girshick, and A.~Farhadi, ``You only look once: Unified, real-time object detection,'' in {\em Proceedings of the IEEE CVPR}, June 2016.

\bibitem{Jocher_2020_YOLOv5_by_Ultralytics}
G.~Jocher, ``{YOLOv5 by Ultralytics},'' May 2020.

\bibitem{Jocher_2023_YOLO_v8}
G.~Jocher, A.~Chaurasia, and J.~Qiu, ``{YOLO by Ultralytics},'' Jan. 2023.

\bibitem{ge2021yolox}
Z.~Ge, S.~Liu, F.~Wang, Z.~Li, and J.~Sun, ``Yolox: Exceeding yolo series in 2021,'' {\em arXiv preprint arXiv:2107.08430}, 2021.

\bibitem{ren2017fasterRCNN-tpami}
S.~Ren, K.~He, R.~Girshick, and J.~Sun, ``Faster {R-CNN}: Towards real-time object detection with region proposal networks,'' {\em IEEE Transactions on Pattern Analysis and Machine Intelligence}, vol.~39, no.~6, pp.~1137--1149, 2017.

\bibitem{Carion_2020_detr}
N.~Carion, F.~Massa, G.~Synnaeve, N.~Usunier, A.~Kirillov, and S.~Zagoruyko, ``End-to-end object detection with transformers,'' in {\em ECCV 2020}, pp.~213--229, Springer, 2020.

\bibitem{zheng2023less}
D.~Zheng, W.~Dong, H.~Hu, X.~Chen, and Y.~Wang, ``Less is more: Focus attention for efficient detr,'' in {\em Proceedings of the IEEE/CVF international conference on computer vision}, pp.~6674--6683, 2023.

\bibitem{zhao2024ms}
C.~Zhao, Y.~Sun, W.~Wang, Q.~Chen, E.~Ding, Y.~Yang, and J.~Wang, ``Ms-detr: Efficient detr training with mixed supervision,'' in {\em Proceedings of the IEEE/CVF Conference on Computer Vision and Pattern Recognition}, pp.~17027--17036, 2024.

\bibitem{Chen_2023_DiffusionDet}
S.~Chen, P.~Sun, Y.~Song, and P.~Luo, ``Diffusiondet: Diffusion model for object detection,'' in {\em Proceedings of the IEEE/CVF ICCV}, pp.~19830--19843, 2023.

\bibitem{anthony2009neural}
M.~Anthony and P.~L. Bartlett, {\em Neural network learning: Theoretical foundations}.
\newblock Cambridge University Press, 2009.

\bibitem{bartlett2021deep}
P.~L. Bartlett, A.~Montanari, and A.~Rakhlin, ``Deep learning: a statistical viewpoint,'' {\em Acta numerica}, vol.~30, pp.~87--201, 2021.

\bibitem{grohs2022mathematical}
P.~Grohs and G.~Kutyniok, {\em Mathematical aspects of deep learning}.
\newblock Cambridge University Press, 2022.

\bibitem{Papadopoulos_2002_inductive}
H.~Papadopoulos, K.~Proedrou, V.~Vovk, and A.~Gammerman, ``Inductive confidence machines for regression,'' in {\em Machine Learning: ECML 2002}, 2002.

\bibitem{Fontana_2023_conformal_review}
M.~Fontana, G.~Zeni, and S.~Vantini, ``{Conformal prediction: A unified review of theory and new challenges},'' {\em Bernoulli}, vol.~29, no.~1, pp.~1 -- 23, 2023.

\bibitem{Angelopoulos_2024_theoretical_CP}
A.~N. Angelopoulos, R.~F. Barber, and S.~Bates, ``Theoretical foundations of conformal prediction,'' {\em arXiv preprint arXiv:2411.11824}, 2024.

\bibitem{Angelopoulos_2022_CRC}
A.~N. Angelopoulos, S.~Bates, A.~Fisch, L.~Lei, and T.~Schuster, ``Conformal risk control,'' in {\em The Twelfth International Conference on Learning Representations}, 2024.

\bibitem{Gupta_2022_nested}
C.~Gupta, A.~K. Kuchibhotla, and A.~Ramdas, ``Nested conformal prediction and quantile out-of-bag ensemble methods,'' {\em Pattern Recognition}, vol.~127, p.~108496, 2022.

\bibitem{Vovk_2012_Conditional}
V.~Vovk, ``Conditional validity of inductive conformal predictors,'' in {\em Proceedings of the Asian Conference on Machine Learning}, pp.~475--490, PMLR, 2012.

\bibitem{Wang_2023_clcp}
D.~Wang, P.~Wang, Z.~Ji, X.~Yang, and H.~Li, ``Conformal loss-controlling prediction,'' 2023.

\bibitem{Sadinle_2019_lac_Least_Ambiguous}
M.~Sadinle, J.~Lei, and L.~Wasserman, ``Least ambiguous set-valued classifiers with bounded error levels,'' {\em Journal of the American Statistical Association}, vol.~114, no.~525, pp.~223--234, 2019.

\bibitem{Vovk_2013_cross_conformal}
V.~Vovk, ``Cross-conformal predictors,'' {\em Annals of Mathematics and Artificial Intelligence}, vol.~74, p.~9–28, July 2013.

\bibitem{barber2021predictive}
R.~F. Barber, E.~J. Cand{\`e}s, A.~Ramdas, and R.~J. Tibshirani, ``{Predictive inference with the jackknife+},'' {\em The Annals of Statistics}, vol.~49, no.~1, pp.~486 -- 507, 2021.

\bibitem{Kim_2020_jackknife_bootstrap}
B.~Kim, C.~Xu, and R.~Barber, ``Predictive inference is free with the jackknife+-after-bootstrap,'' in {\em Advances in Neural Information Processing Systems}, vol.~33, pp.~4138--4149, 2020.

\bibitem{rucklidge1997efficientlyLocatingObjectsHausdorff}
W.~J. Rucklidge, ``Efficiently locating objects using the {H}ausdorff distance,'' {\em International Journal of Computer Vision}, vol.~24, pp.~251--270, 1997.

\bibitem{rezatofighi2019generalized}
H.~Rezatofighi, N.~Tsoi, J.~Gwak, A.~Sadeghian, I.~Reid, and S.~Savarese, ``Generalized intersection over union: A metric and a loss for bounding box regression,'' in {\em Proceedings of the IEEE/CVF conference on computer vision and pattern recognition}, pp.~658--666, 2019.

\bibitem{Bostrom2017}
H.~Bostr\"{o}m, H.~Linusson, T.~L\"{o}fstr\"{o}m, and U.~Johansson, ``Accelerating difficulty estimation for conformal regression forests,'' {\em Annals of Mathematics and Artificial Intelligence}, vol.~81, p.~125–144, Mar. 2017.

\bibitem{Romano_2020_APS}
Y.~Romano, M.~Sesia, and E.~Candes, ``Classification with valid and adaptive coverage,'' in {\em Advances in Neural Information Processing Systems}, vol.~33, pp.~3581--3591, 2020.

\bibitem{Angelopoulos_2021_RAPS}
A.~N. Angelopoulos, S.~Bates, M.~Jordan, and J.~Malik, ``Uncertainty sets for image classifiers using conformal prediction,'' in {\em International Conference on Learning Representations}, 2021.

\bibitem{luo2024trustworthy}
R.~Luo and Z.~Zhou, ``Trustworthy classification through rank-based conformal prediction sets,'' {\em arXiv preprint arXiv:2407.04407}, 2024.

\bibitem{luo2024weighted}
R.~Luo and Z.~Zhou, ``Weighted aggregation of conformity scores for classification,'' {\em arXiv preprint arXiv:2407.10230}, 2024.

\bibitem{melki2024penalized}
P.~Melki, L.~Bombrun, B.~Diallo, J.~Dias, and J.-P. Da~Costa, ``The penalized inverse probability measure for conformal classification,'' in {\em Proceedings of the IEEE/CVF Conference on Computer Vision and Pattern Recognition}, pp.~3512--3521, 2024.

\bibitem{Lin_2014_COCO}
T.-Y. Lin, M.~Maire, S.~Belongie, J.~Hays, P.~Perona, D.~Ramanan, P.~Doll{\'a}r, and C.~L. Zitnick, ``Microsoft coco: Common objects in context,'' in {\em Computer Vision -- ECCV 2014}, pp.~740--755, Springer, 2014.

\bibitem{zouzou2025robustvisionbasedrunwaydetection}
A.~Zouzou, L.~Andéol, M.~Ducoffe, and R.~Boumazouza, ``Robust vision-based runway detection through conformal prediction and conformal map,'' 2025.

\bibitem{GineNickl2015mathematicalFoundations}
E.~Gin\'{e} and R.~Nickl, {\em Mathematical Foundations of Infinite-Dimensional Statistical Models}.
\newblock Cambridge Series in Statistical and Probabilistic Mathematics, Cambridge University Press, 2015.

\end{thebibliography}
}

\clearpage
\appendices
\section{Proofs}
\label{app:proof}

Throughout this section, for the sake of readability, we sometimes denote the test sample as the $(n+1)$-th sample, thus replacing $\Xtest, \Ytest, L_{\mathrm{test}}$ by $X_{n+1}, Y_{n+1}, L_{n+1}$.

For readability again, we also leave out a few measure-theoretic details, such as the fact that we use the cylindrical $\sigma$-algebra (e.g., \cite[p.16]{GineNickl2015mathematicalFoundations}) to properly define the random functions $\Lcnf_i$ or $\Lbul_i$, or the fact that $\lb^\mathrm{cnf}_{\pm},\lbcnfo,\lbbulp,\lbbulo$ can be made measurable (i.e., that they are well-defined random variables), as well as other quantities such as $\Lbul_i\left(\lbcnfo, \lbbulo\right)$.

\subsection{Main theorem}

The main theorem is reminded below.
Recall that the placeholder $\bullet \in \{\mathrm{loc},\mathrm{cls}\}$ means that we work both for the localization and classification tasks, and that $\bar{\lambda}^\mathrm{cnf} = \max \Lambda^\mathrm{cnf}$.

\begin{customthm}{2}
    \label{th:odcrc-bis}
    Suppose that Assumptions \ref{ass:cp-independence} and \ref{ass:loss-seq}
    hold true.\footnote{See Footnote~\ref{ft:iidlosses}.} 
    Let $\alphacnf \geq 0$ and $\alphabul \geq \alphacnf + \frac{B^\bullet}{n+1}$. 
    Then, $\lb^\mathrm{cnf}_{\pm}$ and 
        $\lbbulp$ introduced in \eqref{eq:seqcrc-lbcnf-plus}, \eqref{eq:seqcrc-lbcnf-min} and \eqref{eq:seqcrc-lbj} are well defined, and
    \begin{equation}
        \tag{7}
        \label{eq:thm-bis}
        \Ex{L_{n+1}^\bullet(\lbcnfp, \lbbulp)} \leq \alphabul \;.
    \end{equation}
    Under the additional assumption that $\Lcnf_i(\lbcnfbar)\leq\alpha^\mathrm{cnf}$ almost surely for all $i \in \{1,\ldots,n+1\}$, we also have
    \begin{equation}
        \tag{8}
        \label{eq:thm_cnf-bis}
        \Ex{L_{n+1}^\mathrm{cnf}(\lbcnfp)} \leq \alphacnf \;.
    \end{equation}
\end{customthm}

\begin{proof} 
The fact that $\lbcnf_\pm$ are well defined raises no specific issue (we use the convention $\inf \varnothing = \lbcnfbar$ in case the sets in \eqref{eq:seqcrc-lbcnf-plus} or \eqref{eq:seqcrc-lbcnf-min} are empty). Next we start by showing that $\lbbulp$ is well defined, and then prove the bounds~\eqref{eq:thm_cnf-bis} and~\eqref{eq:thm-bis}. \\[0.1cm]
    \textbf{Step 1: correct definition of $\lbbulp$.} 
    We show that the set in \eqref{eq:seqcrc-lbj} is nonempty, by distinguishing two cases. 
    If $\lbcnfm=\lbcnfbar$, then using the assumptions $\Lbul_i(\lbcnfbar,\bar{\lambda}^\bullet)=0$ and $\alpha^\bullet\geq\frac{B^\bullet}{n+1}$, we can see that the set
    \begin{equation}
        \label{eq:Lbulp-set}
        \mathcal{L}_+^\bullet = \left\{ \lbbul \in \Lbbul : \frac{n\Rbul_{n}(\lbcnfm,\lbbul)}{n+1} + \frac{B^\bullet}{n+1} \leq \alpha^\bullet\right\}
    \end{equation}
    contains $\bar{\lambda}^\bullet$ and is thus nonempty. (Recall that $\Rbul_{n}(\lbcnf,\lbbul) = \frac{1}{n}\sum_{i=1}^{n} \Lbul_i(\lbcnf, \lbbul)$.)
    
    On the other hand, if $\lbcnfm < \lbcnfbar$, then by definition of $\lbcnfm$,\footnote{In particular, when $\lbcnfm < \lbcnfbar$, the set in \eqref{eq:seqcrc-lbcnf-min} is nonempty.} by $\Rbul_n(\lbcnf,\bar{\lambda}^\bullet) \leq \Rtcnf_n(\lbcnf)$, and by right-continuity of $\Rbul_n$ in its first argument, we have
    $\frac{n}{n+1}\Rbul_n(\lbcnfm,\bar{\lambda}^\bullet)\leq \alpha^\text{cnf}$ almost surely. 
    Combining this with the assumption $\alpha^\bullet \geq \alpha^\text{cnf} + \frac{B^\bullet}{n+1}$ implies that $\bar{\lambda}^\bullet\in\mathcal{L}_+^\bullet$ as in the previous case.
    
    Therefore, in both cases, the (random) set $\mathcal{L}_+^\bullet$ is almost surely nonempty, so that $\lbbulp = \inf \mathcal{L}_+^\bullet$ is indeed well defined.
    
    \ \\
    \textbf{Step 2: proof of Eq.~\eqref{eq:thm_cnf-bis}.} Though it almost follows from \cite[Theorem 1]{Angelopoulos_2022_CRC}, we briefly recall the main arguments for the sake of completeness. 
    In this step, we make the additional assumption that $\Lcnf_i(\lbcnfbar)\leq\alpha^\mathrm{cnf}$ almost surely for all $i\in \{1,\dots,n+1\}$. 
    The proof proceeds by introducing the oracle
    \begin{equation}
    \label{eq:lbcnfo}
        \lbcnfo = \inf \left\{ \lbcnf \in \Lbcnf : \Rtcnf_{n+1}(\lbcnf)\leq \alpha^\text{cnf} \right\} \;,
    \end{equation}
    where 
    \begin{align}
    \Rtcnf_{n+1}(\lbcnf) & = \max \Big\{\Rcnf_{n+1}(\lbcnf), \Rloc_{n+1}(\lbcnf,\lblocbar),\nonumber \\
&\qquad\qquad\Rcls_{n+1}(\lbcnf,\lbclsbar)  \big\}\;.\label{eq:Rtnp1}
    \end{align}
First note that $\Rtcnf_{n+1}(\lbcnfbar) \leq \alpha^\mathrm{cnf}$ almost surely. (Indeed, by assumption, $\Rcnf_{n+1}(\lbcnfbar) = \frac{1}{n+1}\sum_{i=1}^{n+1}\Lcnf_i(\lbcnfbar)\leq\alpha^\mathrm{cnf}$ and $\Rbul_{n+1}(\lbcnfbar,\lbbulbar)=\frac{1}{n+1}\sum_{i=1}^{n+1}\Lbul_i(\lbcnfbar,\lbbulbar)=0$ a.s.) Therefore, the set in \eqref{eq:lbcnfo} is nonempty and $\lbcnfo$ is well defined.\footnote{In Step 3.a below, we will also handle cases in which this set is empty.}
Moreover, by \eqref{eq:lbcnfo}, by $\Rcnf_{n+1} \leq \Rtcnf_{n+1}$ and by right-continuity of  $\Rcnf_{n+1}$, we have
     $\Rcnf_{n+1}(\lbcnfo) \leq \alpha^\text{cnf}$ 
    almost surely. Taking the expectation and noting that $\lbcnfo$ is symmetric (i.e., invariant under any permutation of the $n+1$ calibration/test points; see Lemma \ref{lemma:symmetry}  in Appendix~\ref{sec:technicallemma} below) yields $\Ex{L_{n+1}^\mathrm{cnf}(\lbcnfo)} \leq \alphacnf$. The fact that $\lbcnfp \geq \lbcnfo$ (see details in Step 3.a below) and that $L_{n+1}^\mathrm{cnf}$ is non-increasing concludes the proof of~\eqref{eq:thm_cnf-bis}.
    \ \\
    \textbf{Step 3: proof of Eq.~\eqref{eq:thm-bis}.} The rest of the proof is novel. Though similar in spirit to  that of \cite[Theorem 1]{Angelopoulos_2022_CRC}, our multitask extension requires some additional care because $\lbbulp$ uses \emph{the same calibration data} as $\lbcnfp$. This is in contrast with, e.g., the analysis of \cite[Theorem~3 in Section~3.2]{xu2024twoStageCRC} that works for two independent calibration datasets (data-splitting trick).
    Note moreover that, contrary to Step 2, in all the sequel, we do not assume that $\Lcnf_i(\lbcnfbar)\leq\alpha^\mathrm{cnf}$ almost surely for all $i$.

    We proceed in three steps: we first compare $\lbcnfp$ and $\lbbulp$ with the oracles $\lbcnfo$ and $\lbbulo$ defined in~\eqref{eq:lbcnfo-new} and~\eqref{eq:lbbulo} below, which yields $\Ex{\Lbul_{n+1}(\lbcnfp,\lbbulp)}\leq \Ex{\Lbul_{n+1}(\lbcnfo, \lbbulo)}$. Then, in Step~3.c, we show that $\Ex{\Lbul_{n+1}(\lbcnfo, \lbbulo)}\leq\alpha^\bullet$.

    \ \\
    \textit{Step 3.a: we show that $\lbcnf_- \leq \lbcnf_* \leq \lbcnf_+$ almost surely.}
    To that end, recall that\footnote{Though not written explicitly to simplify notation, the parameter $\lbcnf$ runs over $\Lbcnf$ in each of the three sets below.}
        
    \begin{align}   
    \lbcnf_{+} & = \inf \left\{ \lbcnf : \frac{n\Rtcnf_{n}(\lbcnf)}{n+1} + \frac{\tilde{B}^\text{cnf}}{n+1} \leq \alpha^\text{cnf} \right\} \;,                  \nonumber \\[2ex] 
    \lbcnfo & = \inf \left\{ \lbcnf : \Rtcnf_{n+1}(\lbcnf) \leq \alpha^\text{cnf} \right\} \;, \label{eq:lbcnfo-new}   \\[2ex] 
    \lbcnf_{-} & = \inf \left\{ \lbcnf : \frac{n\Rtcnf_{n}(\lbcnf)}{n+1} + \frac{0}{n+1} \leq \alpha^\text{cnf} \right\} \;, \nonumber
    \end{align}   
    where $\inf\varnothing=\lbcnfbar$ by convention. (The definition of $\lbcnfo$ in \eqref{eq:lbcnfo-new} is slightly more general than in \eqref{eq:lbcnfo}, since empty sets are possible without the assumption on $\Lcnf_i(\lbcnfbar)$.) 
    
    Note that the three sets above, say, $A_+ \subset A_* \subset A_-$, are nested. (This is by definition of $\Rtcnf_{n+1}$ in \eqref{eq:Rtnp1} and because $\Lcnf_{n+1}$ and $\Lbul_{n+1}$ are bounded between $0$ and $\Tilde{B}^\mathrm{cnf}=\max \{ B^\mathrm{cnf}, B^\mathrm{loc}, B^\mathrm{cls} \}$ by Assumption~\ref{ass:loss-seq}.). Taking the infimum (with the convention $\inf\varnothing=\lbcnfbar$) yields $\lbcnf_- \leq \lbcnf_* \leq \lbcnf_+$ almost surely. 

    \ \\
    \textit{Step 3.b: we show that $\lbbulo \leq \lbbulp$ almost surely,} where the oracle $\lbbulo$ (which depends on the first-step oracle $\lbcnfo$) is defined by 
    \begin{equation}
    \label{eq:lbbulo}
        \lbbulo = \inf \left\{ \lbbul \in \Lbbul : \frac{1}{n+1}\sum_{i=1}^{n+1} \Lbul_i(\lbcnfo, \lbbul) \leq \alpha^\bullet \right\} \;.
    \end{equation}
 To see why $\lbbulo$ is bounded from above by $\lbbulp = \inf \mathcal{L}_+^\bullet$ (see \eqref{eq:seqcrc-lbj} and \eqref{eq:Lbulp-set}), recall that $\lbcnfm \leq \lbcnfo$ (Step 3.a), and that the $L_{i}^\bullet$ are non-increasing in $\lbcnf$ and bounded from above by $B^\bullet$ (by Assumption~\ref{ass:loss-seq}). Therefore, almost surely,  for all $\lbbul\in\Lbbul$,
    
    \begin{align*}
        & \frac{1}{n+1}\sum_{i=1}^{n} \Lbul_i(\lbcnfm, \lbbul) + \frac{B^\bullet}{n+1}   \\
        \geq & \frac{1}{n+1}\sum_{i=1}^{n} \Lbul_i(\lbcnfo, \lbbul) + \frac{\Lbul_{n+1}(\lbcnfo, \lbbul)}{n+1} \;.
    \end{align*}
    The set in~\eqref{eq:lbbulo} thus contains the set $\mathcal{L}_+^\bullet$ defined in~\eqref{eq:Lbulp-set}, so that $\lbbulo \leq \lbbulp$ almost surely. Importantly, the fact that $\mathcal{L}_+^\bullet$ is a.s. nonempty also entails that the set in~\eqref{eq:lbbulo} is a.s. nonempty.
    
    Putting Steps 3.a and 3.b together, and by monotonicity of $L_{n+1}^\bullet$ in both its variables, we obtain $\Lbul_{n+1}(\lbcnfp,\lbbulp)\leq \Lbul_{n+1}(\lbcnfo, \lbbulo)$ a.s., so that, taking expectations,
    \begin{equation}
        \label{eq:part1-seqrc-proof}
        \Ex{\Lbul_{n+1}(\lbcnfp,\lbbulp)}\leq \Ex{\Lbul_{n+1}(\lbcnfo, \lbbulo)} \;.
    \end{equation}

    \ \\
    \textit{Step 3.c: we show that $\Ex{\Lbul_{n+1}(\lbcnfo,\lbbulo)}\leq \alpha^\bullet$,} by using the fact that $\lbcnf_*$ and $\lbbulo$ are symmetric in the sense of Lemma \ref{lemma:symmetry}  in Appendix~\ref{sec:technicallemma} below.

More precisely, since the sequence $(L^\text{cnf}_i, L^\text{loc}_i, L^\text{cls}_i)_{i=1}^{n+1}$ is i.i.d. (by Assumption~\ref{ass:cp-independence}) and thus exchangeable, and since the oracles $\lbcnfo$ and $\lbbulo$ are symmetric (by Lemma \ref{lemma:symmetry}), we have, permuting $\Lbul_{n+1}$ with any $\Lbul_i$,
\begin{equation*}
    \begin{aligned}
        \Ex{\Lbul_{n+1}(\lbcnfo, \lbbulo)}      & =\frac{1}{n+1}\sum_{i=1}^{n+1}\Ex{\Lbul_i\left(\lbcnfo, \lbbulo\right)}\\
        & = \Ex{\frac{1}{n+1}\sum_{i=1}^{n+1}\Lbul_i\left(\lbcnfo, \lbbulo\right)} \leq \alpha^\bullet\;,
    \end{aligned}
\end{equation*}
where we used the fact that $\frac{1}{n+1}\sum_{i=1}^{n+1}\Lbul_i\left(\lbcnfo, \lbbulo\right) \leq \alpha^\bullet$ almost surely. (This follows from the fact that the set in \eqref{eq:lbbulo} is nonempty, with an infimum in $\lbbulo$, and by right-continuity of $\lbbul \mapsto \Lbul_i(\lbcnf_*, \lb^\bullet)$.)
    Combining \eqref{eq:part1-seqrc-proof} with the last inequality above concludes the proof of \eqref{eq:thm-bis}.
\end{proof}

\subsection{A simple symmetry property}
\label{sec:technicallemma}
We now state and prove a simple symmetry property that is key to obtaining Theorem~\ref{th:odcrc}. Put briefly, it means that the oracle parameters $\lbcnfo$ and $\lbbulo$, which are functions of $(L^\text{cnf}_1, L^\text{loc}_1, L^\text{cls}_1), \ldots, (L^\text{cnf}_{n+1}, L^\text{loc}_{n+1}, L^\text{cls}_{n+1})$, are invariant under any permutation of these loss $3$-tuples.  Recall that we use the placeholder $\bullet\in\{\text{loc}, \text{cls}\}$.

\begin{lemma}
    \label{lemma:symmetry}
        The oracle parameters $\lbcnfo$ and $\lbbulo$ defined in~\eqref{eq:lbcnfo-new} and~\eqref{eq:lbbulo} are symmetric functions of $(L^\text{cnf}_1, L^\text{loc}_1, L^\text{cls}_1), \ldots, (L^\text{cnf}_{n+1}, L^\text{loc}_{n+1}, L^\text{cls}_{n+1})$. 
        That is, for any permutation $\sigma$ of $\left\{1, \ldots, n+1\right\}$, we have
        
        \begin{nonumber}
        \begin{align}
            &\lbcnfo\left((L^\text{cnf}_1, L^\text{loc}_1, L^\text{cls}_1), \ldots, (L^\text{cnf}_{n+1}, L^\text{loc}_{n+1}, L^\text{cls}_{n+1}) \right) = \\
            &\lbcnfo \left(
                (L^\text{cnf}_{\sigma(1)}, L^\text{loc}_{\sigma(1)}, L^\text{cls}_{\sigma(1)}), \ldots, (L^\text{cnf}_{\sigma(n+1)}, L^\text{loc}_{\sigma(n+1)}, L^\text{cls}_{\sigma(n+1)})
            \right)
        \end{align}
        \end{nonumber}
        and
        \begin{nonumber}
        \begin{align}
            &\lbbulo\left((L^\text{cnf}_1, L^\text{loc}_1, L^\text{cls}_1), \ldots, (L^\text{cnf}_{n+1}, L^\text{loc}_{n+1}, L^\text{cls}_{n+1}) \right) = \\
            &\lbbulo \left(
                (L^\text{cnf}_{\sigma(1)}, L^\text{loc}_{\sigma(1)}, L^\text{cls}_{\sigma(1)}), \ldots, (L^\text{cnf}_{\sigma(n+1)}, L^\text{loc}_{\sigma(n+1)}, L^\text{cls}_{\sigma(n+1)})
            \right).
        \end{align}
        \end{nonumber}
\end{lemma}

    \begin{proof}
        The proof is rather straightforward, but we provide it for the sake of completeness. We start by showing that $\lbcnfo$ is symmetric.
        Recall from \eqref{eq:lbcnfo-new} that
        \[
        \lbcnfo = \inf \left\{ \lbcnf \in \Lbcnf :\Rtcnf_{n+1}(\lbcnf)\leq \alpha^\text{cnf} \right\}
        \]
        (with the convention $\inf\varnothing=\lbcnfbar$), 
        and note from \eqref{eq:Rtnp1} that $\Rtcnf_{n+1}$ is a maximum of three terms that are invariant under any permutation of the $3$-tuples $(L^\text{cnf}_1, L^\text{loc}_1, L^\text{cls}_1), \ldots, (L^\text{cnf}_{n+1}, L^\text{loc}_{n+1}, L^\text{cls}_{n+1})$. Therefore, $\lbcnfo$ is symmetric.
        
        We now prove the symmetry of $\lbbulo$, which is only a little more subtle. Indeed, note from~\eqref{eq:lbbulo} that the definition
        \[
        \lbbulo = \inf \left\{ \lbbul \in \Lbbul : \frac{1}{n+1}\sum_{i=1}^{n+1} \Lbul_i(\lbcnfo, \lbbul) \leq \alpha^\bullet \right\}
        \]
        involves the oracle parameter $\lbcnfo$. Since both the sum operator and $\lbcnfo$ (by the first part of the proof) are permutation-invariant, so is $\lbbulo$. This concludes the proof.
    \end{proof}

\begin{figure*}[t!] 
    \centering
    \subfloat[Successful detection of bears]{\includegraphics[width=0.40\linewidth]{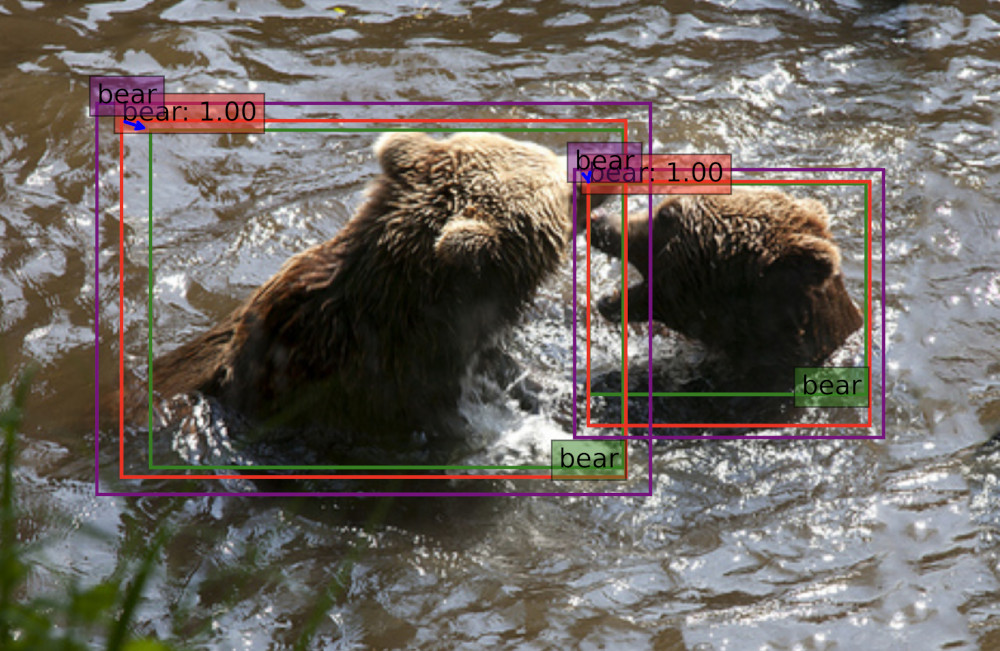}\label{fig:app-succ-subfig1}}
    \hfill 
    \subfloat[Successful detection of a clock tower]{\includegraphics[width=0.2\linewidth]{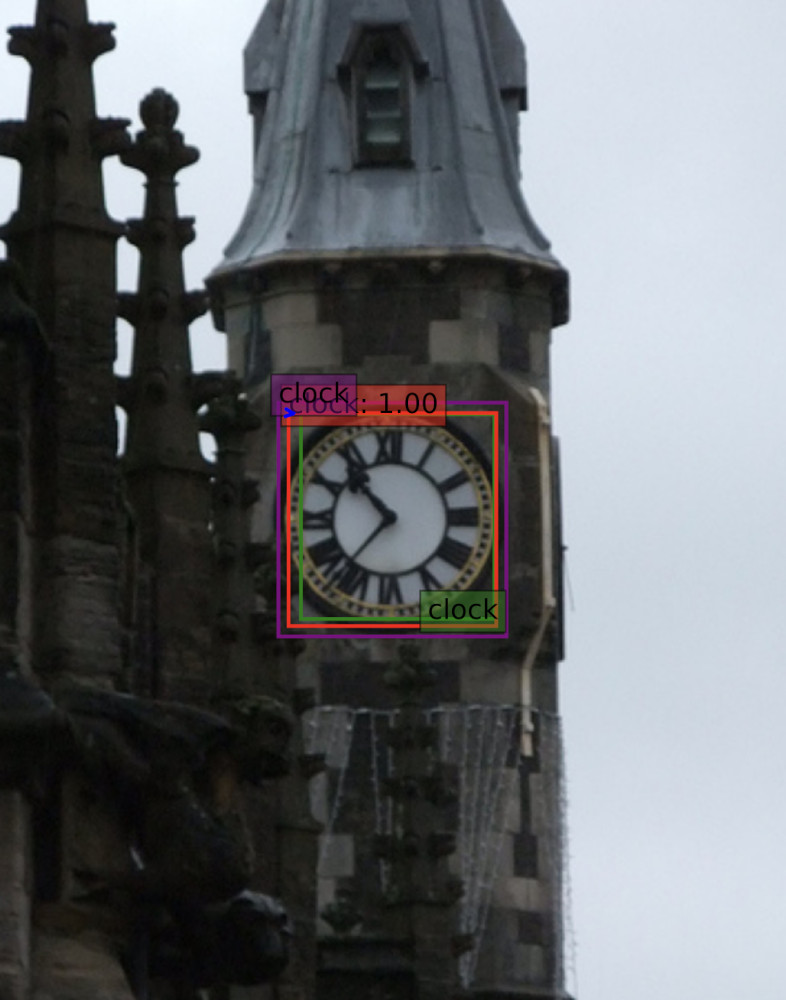}\label{fig:app-succ-subfig2}}
    \hfill 
    \subfloat[\emph{Technically} successful detection of an airplane]{\includegraphics[width=0.35\linewidth]{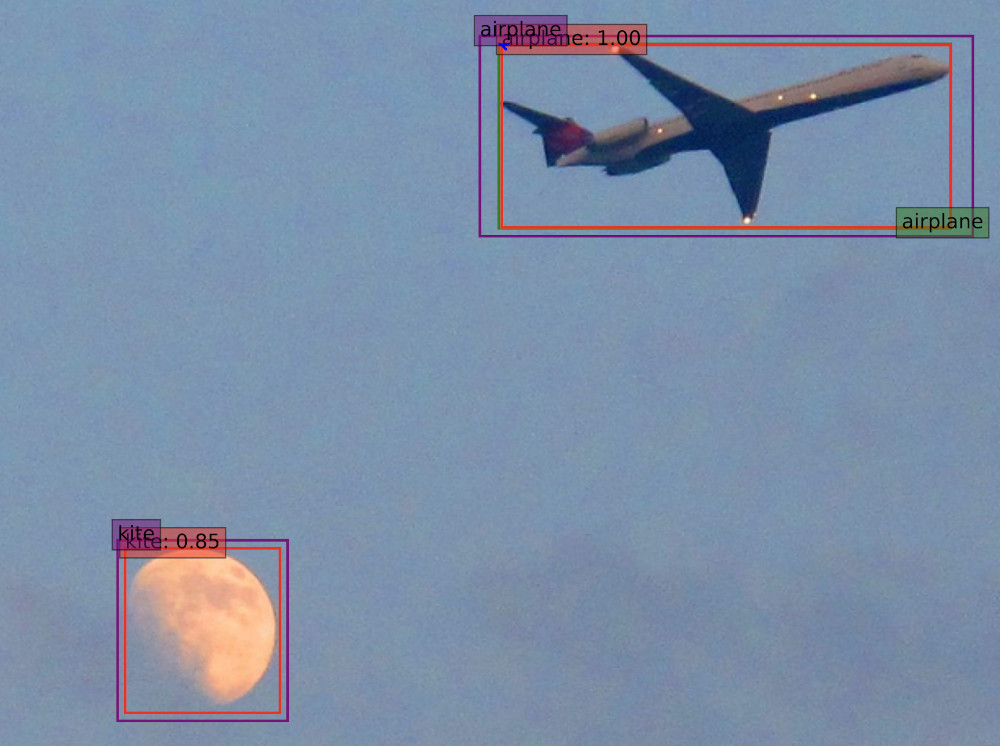}\label{fig:app-succ-subfig3}}
    \caption{Examples of two successful conformal predictions, and one \emph{technically successful}: this is a consequence of our recall-based guarantee.}
    \label{fig:app-succ}
\end{figure*}

\section{Additional Losses}
\label{app:additional-losses}
We here discuss losses not mentioned in the article, and generally not tested. They are presented for information, and potential future works when appropriate needs arise.

As discussed in details in Appendix~\ref{app:box_lvl}, many natural image-level losses are aggregations of object-level losses. While there are exceptions to that, the losses presented in this appendix are of this form. Formally, denote by $\Phi:S(\R)\mapsto\R$ an aggregation function, where $S(\R)$ correspond to the set of all the finite sequence of real values.
We introduce three main aggregation strategies.
\begin{enumerate}
    \item Average: $\Phi_\mathrm{avg}(S)=\frac{1}{\abs{S}}\sum_{i=1}^{\abs{S}}S_i$
    \item Max: $\Phi_\mathrm{max}(S)=\max_i S_i$
    \item Thresholded Average (at level $\tau$): $\Phi_\mathrm{thr}(S)=\mathbbm{1}_{\Phi_\mathrm{avg}(S)>\tau}$
\end{enumerate}

In the next sections, we therefore introduce object-level losses, than can be used with any of the aforementioned aggregation functions.

\subsection{Confidence}

The losses introduced in our work for confidence focus on the number of prediction, regardless of their location on the image. It is common for most confident boxes to be located close to true objects with well-performing predictors. We have indeed observed that accounting for the localization of the prediction at the confidence step was not necessary with our models and dataset. This may however, not be the case generally, and we therefore introduce the following losses, which account only for prediction sufficiently close to the ground truth.
These losses are presented for an image $x$, for a true box $b$, and the set of conformal bounding boxes $\Gacnf(x)$.
First, one can simply consider the inclusion of the ground truth:
\[
    \Lcnf_\mathrm{incl}(\lbcnf)=\begin{cases}
    0 &\text{ if } \exists \bhat\in\Gacnf(x),b \subseteq \bhat\\
    1 &\text{ otherwise}.
    \end{cases}
\]
This criterion is however very stringent and can result very quickly in a very low confidence threshold, i.e., numerous predictions being kept. We therefore recommend to use it only after applying a scaling or margin on the predictions $\bhat$ by a factor decided a priori (hyperparameter).

\begin{figure*}[t]
    \centering
    \includegraphics[width=\textwidth]{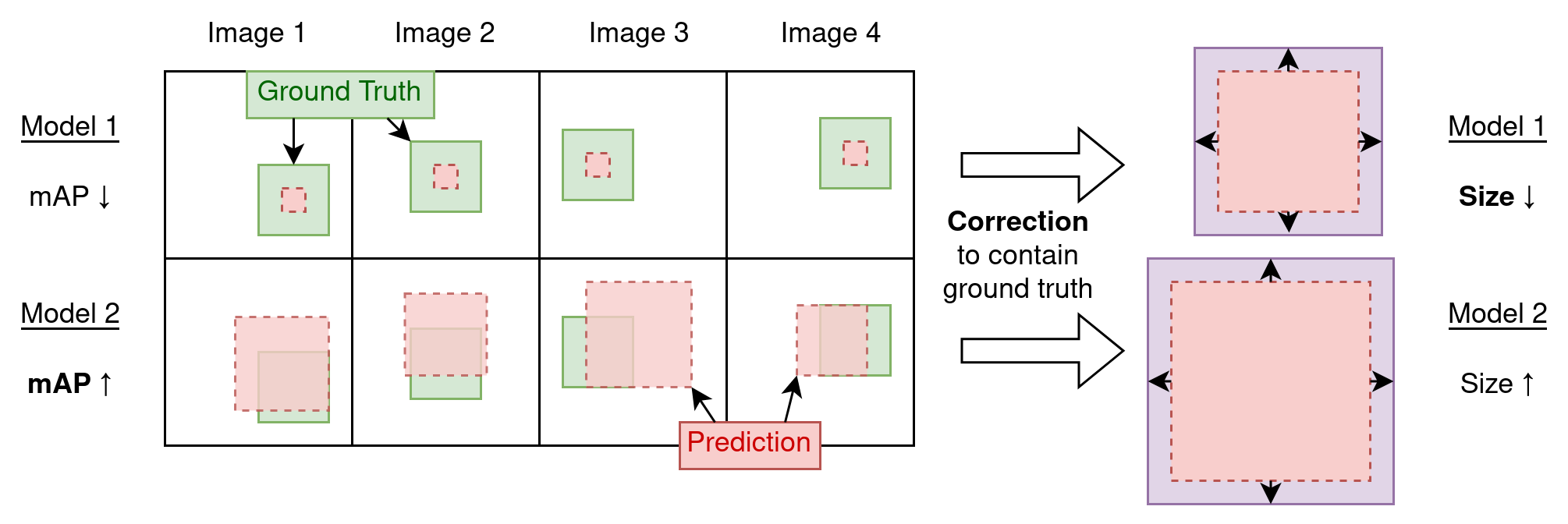}
    \caption{Higher 
    \cp performance (small correction prescribed) does not necessarily correlate to a higher mean Average Precision (mAP). 
    Considering a simple additive margin in \cp, an always-smaller but always-centered prediction performs worse in mAP, but requires significantly less conformal correction than a better mAP-performing model, which often overshoots the ground truth box, uniformly at random in each direction.}
    \label{fig:map_vs_conf}
\end{figure*}

An alternative is to replace the inclusion criterion $\subseteq$ by a distance, formalized by

\[
\Lcnf_{\mathrm{dst}}(\lbcnf) = \begin{cases}
    0 &\text{ if } \exists \bhat\in\Gacnf(x), d((b,c), (\bhat,\hat{c}))<\tau\\
    1 &\text{ otherwise},
\end{cases}
\]
where $d:(\mathcal{B}\times\mathcal{C})\times(\mathcal{B}\times\Sigma_{K-1})\rightarrow\R^{+}$ is any function measuring a distance (particularly, it does not need to be symmetric), and $\tau$ is a pre-defined threshold for a prediction to be considered sufficiently close to a ground truth.
Note that we define the loss to be $0$ when there are no ground truths on an image.
Note that the labels $c$ and $\hat{c}$ are also included, and may or not be used (such as to set arbitrarily high values to significantly different classes). 
We recommend for the sake of consistency to use the same distance as picked for matching, discussed in Section~\ref{sec:level-of-guarantee-and-matching}.

Alternatively, we can average the distance between the predictions and the ground truths. This is a non-binarized version of the loss function mentioned above.

\[
\Lcnf_{\mathrm{dist}}(\lbcnf) = 
    \min_{(\bhat,\hat{c})\in\Gacnf(x)} 
        d\left(
            ( b,c ),
            ( \bhat,\hat{c} ) 
        \right),
\]
where $d:(\mathcal{B}\times\mathcal{C})\times(\mathcal{B}\times\Sigma_{K-1})\rightarrow[0,1]$ is any \emph{normalized} distance (or divergence). Again, we recommend to use the same distance used for matching predictions and ground truths, in order to obtain more interpretable guarantees.

Nevertheless, an issue pertains: the prediction sets for confidence depend strongly on the confidence score and are not adaptive. Further work in this direction, such as building more complex, uncertainty-aware confidence scores, or combining them is warranted.

\subsection{Localization}

The prediction sets introduced in the main article are simple, but suffer from a lack of adaptivity. This is partly due to the lack of direct notion of uncertainty for predicted bounding boxes. 
A recent work, BayesOD~\cite{Harakeh_2020_bayesod} method replaces Non-Maximum Suppression by a Bayesian fusion, which results in uncertainty metrics on the coordinates of the corners of the predictions. 
Such information can be replace the width/height estimates in the multiplicative correction. We expect the resulting conformal boxes could grant better adaptivity:

\begin{equation}
    \Galoc(x)_k = \bhat_k+\begin{pmatrix}
    \hat{\sigma}^1_k & \hat{\sigma}^2_k & \hat{\sigma}^3_k & \hat{\sigma}^4_k
\end{pmatrix} \,,
\end{equation}
where the $\hat{\sigma}^i_k$ would quantify uncertainty of the $k$-th box coordinate, such as empirical standard deviations.

\subsection{Classification}

Classification is, unlike the two preceding ones, a fairly common task in conformal prediction. We therefore refer the reader to the large literature on the topic, including but not limited to \cite{Angelopoulos_2021_RAPS, luo2024trustworthy, luo2024weighted, melki2024penalized}.

\subsection{Note on Monotonicity}
\label{app:monotonicity}
While we mention several times throughout the article the necessity of the loss to be monotone in all its parameters. As mentioned in Remark \ref{rmk:nonmonotone-loc}, and introduced in ~\ref{sec:seqcrc-od-algos}, we apply a monotonization process that alter the considered losses to make them compatible with our theoretical results. Since our optimization algorithm, introduced in the next section, allows to control losses that are non-monotone by monotonizing them on the fly, any losses can be used. However, at the two extremes lie opposite results. On the one hand, when monotonicity is verified, our algorithm offers no penalty. On the other hand, when the loss is strongly non-monotone 

, the resulting prediction sets will be overly large and non-informative.

Additionally, we could in principle enforce monotonicity in $\lbcnf$ by computing a non-injective matching $\pi_x$ with a distance $d$ that mimics a given localization loss (so that adding boxes in $\Gacnf(x)$ could only decrease this loss). However, this would not fix the lack of monotonicity for classification losses (see Remark~\ref{rmk:nonmonotone-cls}), since we use the same matching $\pi_x$ for the two tasks.

\section{Types of guarantee \& Matching} 
\label{app:box_lvl}
\textbf{Types of guarantee:}

On a given image, there can a priori be any number of objects (and their associated bounding boxes). 
On the one hand, each of those individual objects may be considered as ground truth individually, i.e. $Y_i=(b_i, c_i)$. On the other hand, the ground truth can be defined as the set of all objects present on an image, i.e. $Y_i=(b_{i,j},c_{i,j})_{j=1}^{n_i}$, as introduced in this work.
Respectively, we refer to those as object-level and image-level guarantees following \cite{Andéol_2023_confident} as well as ~\cite{deGrancey_2022_object,Andéol_2023_conformal} where object-level guarantees have previously been referred to as box-level.
However, the exact setting of the latter varied, as in these previous works, conformal prediction was only applied at the object level to true-positive bounding boxes 
(i.e., ground-truth - prediction pairs, with only pairs with IoU above a certain threshold $\beta$ considered).

Defining an object, or the set of objects on an image as the ground truth $Y_i$ implies a different assumption on the data generating process, and thus a different guarantee. 
Indeed, the expectation $\Ex{L_\mathrm{test}(\lambda)}$ is on the joint distribution of the calibration set and the test pair $(X_\mathrm{test}, Y_\mathrm{test})$. 
Therefore, the guarantees hold ``on average'' over either image-object set pairs or over image-individual object pairs, depending on the previous choice.
Importantly, the difference in definition also implies a different assumption: the pairs $(X_i,Y_i)$ are assumed to be independent and identically distributed. 
While it is beyond the scope of this work to discuss the real word validity of this assumption, the image-level one is milder than the object-level one, which implicitly assumes independence between objects of a same image.

From a practical standpoint, if an image always contains exactly one object, there is no difference. 
More generally, when the image-wise loss is an aggregation of object-level losses, a connection can be made between the guarantees by simple reweighing.

The choice between object-level and image-level setting depends on practical requirements and should be reflected in the definition of the losses. 
An example is a binary loss: in the object-level setting, it entails that a proportion at least $1-\alpha$ of ground truth bounding boxes are correctly covered, while in the image-level setting, it entails that, on a proportion $1-\alpha$ of images at least, all bounding boxes (or a desired proportion) are covered.

\textbf{Matching:} 

As common in object detection~\cite{Jocher_2023_YOLO_v8, Carion_2020_detr}, a matching between the true bounding boxes and predicted ones is necessary, not only for evaluation here but also for calibration. 
We refer to by matching, a family of functions $\pi_i$ defined for each image $X_i$ that maps the index $j$ of an object on $X_i$ to the index of a predicted box among $\Gacnf(X_i)$. 
More formally, $\pi_i : \{1,\cdots,n_i\}\rightarrow \{1,\cdots,\abs{\Gacnf(X_i)}\}$ with $n_i$ the number of objects on image $X_i$ and $\Gacnf(X_i)$ the sets of bounding boxes likely to correspond to objects on image $X_i$. 
We do not make further assumption on the matching algorithm. 
Therefore it includes injective functions, i.e. each prediction being matched to at most one ground truth.
We can also consider functions where some inputs might have identic output, i.e. reusing predictions for multiple ground truth.
An example of an injective matching is the commonly used Hungarian algorithm.

Additionally, the above definition of matching fits most practical needs, but a more general approach with a one-to-many matching is also valid.
Intuitively, this setting corresponds to allowing multiple predictions to be matched to a single ground truth object. 
That general approach includes several settings, even ones that do not explicitly consider a matching, as in~\cite{Andéol_2023_confident}.
Finally, prediction sets here can be reused as-is in this general setting, and losses only have to be slightly modified in most cases.

\section{Algorithms}
\label{app:algo}

In this section, we introduce two algorithms to compute the first- and second-step parameters of our SeqCRC method.

Recall that we use a \emph{monotonization trick} to enforce monotonicity of $\Lloc$ and $\Lcls$ in $\lbcnf$. That is, we replace $\Lloc_i(\lbcnf,\lbloc)$ and $\Lcls_i(\lbcnf,\lbcls)$ with $\sup_{\lambda' \geq \lbcnf} \Lloc_i(\lambda',\lbloc)$ and $\sup_{\lambda' \geq \lbcnf} \Lcls_i(\lambda',\lbcls)$ respectively. For computational and memory efficiency purposes, these suprema are computed on the fly (i.e., while optimizing the underlying risks).

More precisely, the pseudo-code of Algorithm~\ref{alg:seqcrc-monotonizing-step1} can be used to compute $\lbcnfp$ and $\lbcnfm$ defined in \eqref{eq:seqcrc-lbcnf-plus} and \eqref{eq:seqcrc-lbcnf-min}, where $\Rloc_n(\lbcnf,\lblocbar)$ and $\Rcls_n(\lbcnf,\lbclsbar)$ (the localization and classification components of $\Rtcnf_n(\lbcnf)$, see  \eqref{eq:risk_cnf_n}) are obtained after monotonizing the localization and classification losses with respect to $\lbcnf$.\footnote{For broader applicability, in the actual code, we also monotonize the losses $\Lcnf_i(\lbcnf)$ within $\Rcnf_n(\lbcnf)$ in \eqref{eq:risk_cnf_n}, even if all $\Lcnf$ examples given in Section~\ref{sec:seqcrc-od-confidence} are already monotone.} To implement the monotonization trick, we use the key fact that the set $\Gacnf(x)$ defined in \eqref{eq:cnf-pred-set-sequence} is a piecewise-constant function of $\lbcnf$ (in a right-continuous manner), with discontinuity points given by the confidence scores $o(x)_{[k]}$ of predicted objects.

Similarly, the pseudo-code of Algorithm~\ref{alg:seqcrc-monotonizing-step2} shows how to compute $\lblocp$ and $\lbclsp$ defined in \eqref{eq:seqcrc-lbj}, where the risk $\Rbul_n(\lbcnfm, \lbbul)$ is obtained after monotonizing the localization or classification losses with respect to $\lbcnf$. More precisely, we output (slight) upper-approximations of $\lblocp$ and $\lbclsp$ using a binary search.

Both algorithms are written for $\Lbcnf = [0,1]$.

\begin{algorithm}
\caption{SeqCRC optimization for Step 1 (confidence), with monotonization trick}
\label{alg:seqcrc-monotonizing-step1}
\begin{algorithmic}[1]
\Require Loss functions $(\Lcnf_i)_{i=1}^n$, $(\Lloc_i)_{i=1}^n$ and $(\Lcls_i)_{i=1}^n$, error level $\alphacnf$, loss bound $B$, parameter bounds $\lblocbar,\lbclsbar$, confidence scores $C$ over whole dataset (list of $n$ sublists).
\Ensure Confidence parameter $\lbcnfp$ (if input $B$ equals $\Tilde{B}^\mathrm{cnf}$) or $\lbcnfm$ (if $B=0$); see \eqref{eq:seqcrc-lbcnf-plus} and \eqref{eq:seqcrc-lbcnf-min}.

\State Flatten the hierarchical list $C$ of confidence scores into a single array. Create an array $I$ of the same size, filled with the associated  image indices.
\State Sort both arrays simultaneously by increasing order of confidence scores, and denote the results by $\mathcal{C}$ and $\mathcal{I}$.
\State Shift $\mathcal{C}$'s values to the left, and replace the last (missing) element by $1$.
\State \% Next we search for the infimum $\lbcnf$ value in \eqref{eq:seqcrc-lbcnf-plus} or \eqref{eq:seqcrc-lbcnf-min}, starting from $\lbcnf=1$ downwards, while computing the losses $\Lcnf_i(\lbcnf)$ and the monotonized losses $\sup_{\lambda' \geq \lbcnf} \Lbul_{i}(\lambda', \lbbulbar)$ \emph{on the fly}.
\State $\lbcnf \gets 1$
\State Define three (monotonized) loss arrays: $\mathrm{Lcnf}$, $\mathrm{Lloc}$, $\mathrm{Lcls}$.
\For{$i = 1,\dots,n$}
    \State Compute losses with $\lbcnf$: $\mathrm{Lcnf}[i] \gets \Lcnf_{i}(\lbcnf)$, $\mathrm{Lloc}[i] \gets \Lloc_{i}(\lbcnf, \lblocbar)$, and $\mathrm{Lcls}[i] \gets \Lcls_{i}(\lbcnf, \lbclsbar)$.
\EndFor
\State Compute the three risks $R^\mathrm{cnf}=n^{-1}\sum_{i=1}^{n}\mathrm{Lcnf}[i]$, $R^\mathrm{loc}=n^{-1}\sum_{i=1}^{n}\mathrm{Lloc}[i]$ and $R^\mathrm{cls}=n^{-1}\sum_{i=1}^{n}\mathrm{Lcls}[i]$.
\State Compute the maximum risk $R=\max\{R^\mathrm{cnf},R^\mathrm{loc},R^\mathrm{cls}\}$.
\State \% We now iteratively decrease the value of $\lbcnf$.
\For{$c, i \in \mathcal{C}, \mathcal{I}$}
    \State Store previous value $\lbcnf {}^{(\mathrm{old})} \gets \lbcnf$.
    \State Set $\lbcnf \gets 1-c$.
    \State On image $i$, update confidence loss and  \emph{monotonized} localization/classification losses:\\
    $\mathrm{Lcnf}[i] \gets \Lcnf_{i}(\lbcnf)$\\
    $\mathrm{Lloc}[i] \gets \max\{\mathrm{Lloc}[i], \Lloc_{i}(\lbcnf, \lblocbar)\}$\\
    $\mathrm{Lcls}[i] \gets \max\{\mathrm{Lcls}[i], \Lcls_{i}(\lbcnf, \lbclsbar)\}$.
    \State Compute the three risks $R^\mathrm{cnf}=n^{-1}\sum_{i=1}^{n}\mathrm{Lcnf}[i]$, $R^\mathrm{loc}=n^{-1}\sum_{i=1}^{n}\mathrm{Lloc}[i]$ and $R^\mathrm{cls}=n^{-1}\sum_{i=1}^{n}\mathrm{Lcls}[i]$.
    \State Compute maximum risk $R=\max\{R^\mathrm{cnf},R^\mathrm{loc},R^\mathrm{cls}\}$.
    \State 
    \% Stopping condition:
    \If{$\frac{n}{n+1}R + \frac{B}{n+1}>\alphacnf$}\label{algo4-stopping-condition} \Return previous $\lbcnf {}^{(\mathrm{old})}$.
    \EndIf
\EndFor
\State \Return $0$ \%  if stopping condition is never satisfied
\end{algorithmic}
\end{algorithm}

\begin{algorithm}
\caption{SeqCRC optimization for Step 2 (localization/classification), with monotonization trick}
\label{alg:seqcrc-monotonizing-step2}
\begin{algorithmic}[1]
\Require Loss functions $(\Lbul_i)_{i=1}^n$, error level $\alphabul$, loss bound $B^\bullet$, confidence parameter $\lbcnfm$ from Step~1, confidence scores~$C$ over the whole dataset (list of $n$ sublists), a priori lower and upper bounds $l,u$ on $\lbbulp$, number $S$ of steps.
\Ensure Localization/classification parameter $\lbbulp$; see \eqref{eq:seqcrc-lbj}.
\State Create an empty array $\hat{\Lambda}$
\For{$k = 1,\ldots,S$}
    \State $\lbbul \gets \frac{l+u}{2}$
    \State Flatten the hierarchical list $C$ of confidence scores into a single array. Create an array $I$ of the same size, filled with the associated  image indices.
    \State Sort both arrays simultaneously by increasing order of confidence scores, and denote the results by $\mathcal{C}$ and $\mathcal{I}$.
    \State Shift $\mathcal{C}$'s values to the left, and replace the last (missing) element by $1$.
    \State \% Next we compute the monotonized losses $\sup_{\lambda' \geq \lbcnfm} \Lbul_{i}(\lambda', \lbbul)$, starting from $\lambda'=1$ downwards.
    \State Define a monotonized loss array $L$.
    \State $\lambda' \gets 1$ 
    \For{$i = 1,\dots,n$}
        \State Set $L[i] \gets \Lbul_i (\lambda', \lbbul)$
    \EndFor
    \For{$c, i \in \mathcal{C}, \mathcal{I}$}
        \State Set $\lambda' \gets 1-c$
        \State Update monotonized loss of image $i$ using confidence parameter $\lambda'$: $L[i] \gets \max \{L[i],\Lbul_{i}(\lambda', \lbbul)\}$
        \If{$\lambda' \leq \lbcnfm$}\label{algo-stopping-condition} \% Stopping condition
            \State \textbf{break}
        \EndIf
    \EndFor
    \State Compute the risk with monotonized losses:  $R \gets n^{-1}\sum_{i=1}^{n} L[i]$
    \If{$\frac{n}{n+1}R + \frac{B^\bullet}{n+1}\leq\alphabul$}
        \State Append $\lbbul$ to $\hat{\Lambda}$
        \State $u \gets \lbbul$
    \Else
        \State $l \gets \lbbul$
    \EndIf\\
    \Return the last element of $\hat{\Lambda}$ if nonempty; otherwise, raise an error
\EndFor
\end{algorithmic}
\end{algorithm}

\section{Details on Experiments}
In this section, we introduce the complete benchmark of our experiments, as well as additional experiments which include discussions on metrics in OD and conformal OD, localization margins and monotonicity.

\subsection{Complete benchmark}
\label{app:complete_benchmark}
We include the complete results of our experiments here in Tables~
\ref{tab:yolo_results_alpha_01_styled},  \ref{tab:yolo_results_alpha_02_styled},
\ref{tab:detr101_results_alpha_01_styled}, and \ref{tab:detr101_results_alpha_02_styled}. All experiments are conducted on a single NVIDIA RTX 4090. Each individual experiment takes approximately 20 minutes of runtime, most of it being taken by the monotonizing optimization algorithm.
We now introduce multiple additional visualizations of our best-performing experiment run, for additional discussions on the limitations of conformal object detection.
On Fig.~\ref{fig:app-succ}, we observe two successful detections, with singleton classification prediction sets and small localization margins. However, on the third image, the moon is predicted to be a kite. While this is incorrect due to the monotonicity constraints of \crc, we may only provide recall guarantees, i.e., ensure that we do not miss the ground truth. In that framework, additional predictions such as the moon predicted as a kite, are not penalized.
In Fig.~\ref{fig:app-fail}, we observe another instance of the previous mistake, since the boat is detected twice: once excluding its mast, and once including it. The next two images illustrate two failure cases of our guarantee, corresponding to previously discussed issues in terms of abuses of labeling. On the middle image, books are sometimes annotated individually, and sometimes as a block. On the right image, a statue is annotated as a person. The variability in the definition of ground truths leads to perceived failures, that are not necessarily actual failures. As visible on both image (b) and (c), the clock  (bottom-right) and statue (person, top) are respectively assigned a book and a clock predictions by our matching. In order to reach lower levels of error rates $\alpha$, these mistakes would have to be controlled, and would result in a need for significant margins at inference (or models that are able to learn this bias).

\subsection{Additional Experiments}
\label{app:further_experiments}
In this section, we discuss three additional experiments: one on the non-monotonicity of losses, one on the correlation between OD and CP metrics, the other on the number of margins used for conformal predicted boxes.

\subsubsection{mAP vs CP}
In the course of our exploratory experiments, we encountered an intriguing yet logically consistent phenomenon: the performance metrics commonly used in \odf do not necessarily align with the efficacy of \cpf techniques, particularly in terms of set size. 
This is exemplified in Figure \ref{fig:map_vs_conf}.

To clarify, consider two hypothetical models, Model 1 and Model 2, with Model 2 demonstrating superior performance in terms of mean Average Precision (mAP). This superiority is evidenced by its larger Intersection over Union (IoU) values with the ground truth across four sampled images, attributed in part to its tendency to predict larger bounding boxes. Notably, these models exhibit distinct error characteristics: Model 1 primarily errs in scaling, whereas Model 2's errors are predominantly cause by translations.
However, a critical insight emerges when considering the necessity to ensure that the ground truth is entirely captured within the predicted prediction's boundaries. Both models require the addition of a comparable margin to their predictions to meet this criterion. 
Consequently, due to Model 2's initially larger predictions, the resultant conformal boxes are significantly larger than those of Model 1.

In a sense, a model that fails in a consistent and predictable manner (``reliably'') can be more efficiently adjusted and at a lower ``cost'' in CP terms, as opposed to a model that, despite its higher performance metrics, fails in an unpredictable or ``random' fashion.

\begin{figure*}[t!] 
    \centering
    \subfloat[\emph{Technically} successful detection of a boat]{\includegraphics[width=0.5\linewidth]{figs/appendix/fail-boat.jpg}\label{fig:app-fail-subfig1}}
    \hfill 
    \subfloat[Unsuccessful prediction of a clock]{\includegraphics[width=0.23\linewidth]{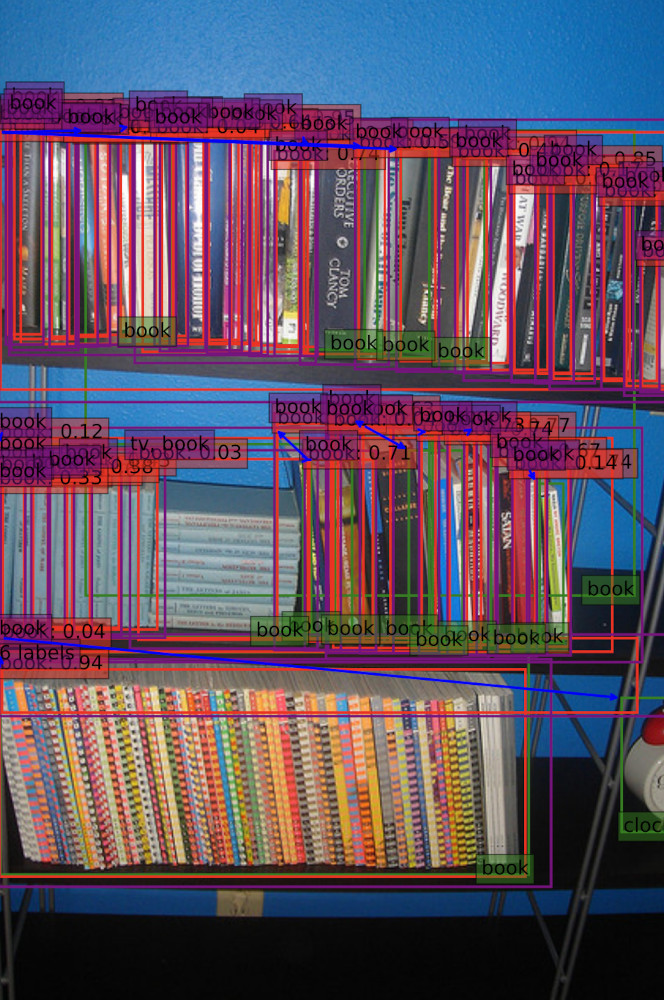}\label{fig:app-fail-subfig2}}
    \hfill 
    \subfloat[Unsuccessful prediction of a statue (person)]{\includegraphics[width=0.23\linewidth]{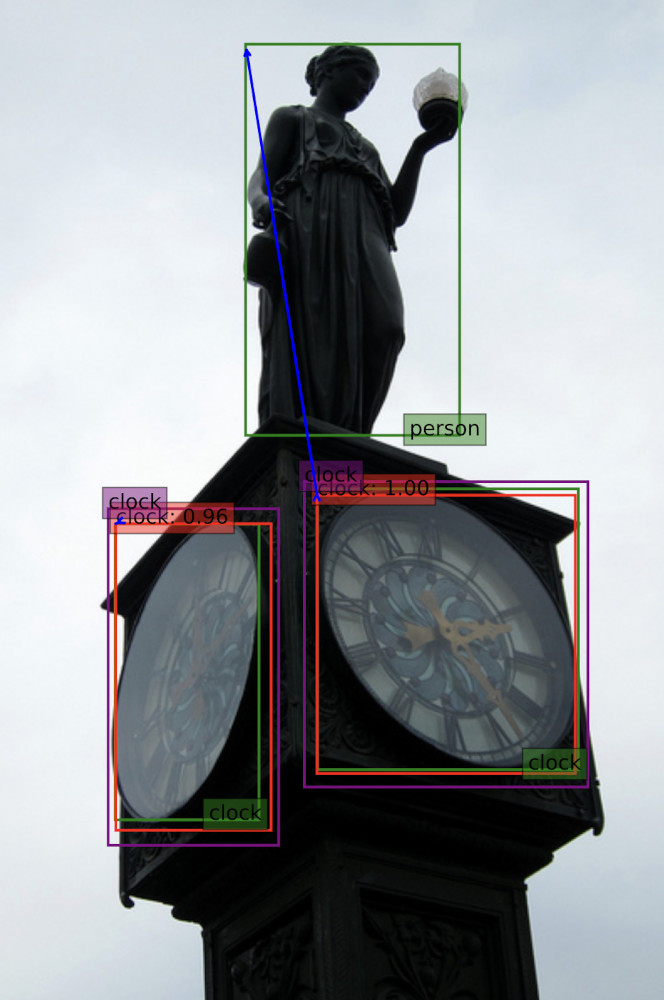}\label{fig:app-fail-subfig3}}
    \caption{Examples of one \emph{technically successful} conformal prediction (as in Fig.~\ref{fig:app-succ}) and two failures, due to a missing prediction.}
    \label{fig:app-fail}
\end{figure*}

\subsubsection{Margins}
In our research, we executed a targeted experiment to ascertain the optimal strategy for applying margins to the predicted bounding boxes in Object Detection. Our primary objective was to determine whether uniform margins should be applied across all four sides of a bounding box, whether separate margins should be used for height and width, or whether distinct margins should be applied to each side. It is important to note that while intuitively, the latter approach may seem advantageous in most scenarios, it inherently equates to conducting a multiple-hypothesis test, necessitating a correction for multiple comparisons.

\begin{table}[ht]
\centering
\small
\caption{Comparison of Margins in pixels, Coverages and Set sizes for considering \{1,2,4\} distinct values for margin corrections (with Bonferroni correction), for $\alpha=0.1$ (resp. $\alpha=0.05$, $\alpha=0.025$ with correction).}
\begin{tabular}{c@{\hskip 5pt} c@{\hskip 5pt}c@{\hskip 5pt}c@{\hskip 5pt}c@{\hskip 5pt} c@{\hskip 5pt}c@{\hskip 5pt}}
  {N. margins} 
& \multicolumn{4}{c}{Value of the Margin} 
& {Cov} 
& {Mean Set Size} \\
\cmidrule(lr){2-5}
& Left & Top & Right & Bottom & \\
\midrule
1 & 11.88 & 11.88 & 11.88 & 11.88 & 96.30\,\% & 142 \\
2 & 19.58 & 16.18 & 19.58 & 16.18 & 97.43\,\% & 145 \\
4 & 26.34 & 24.89 & 28.11 & 14.30 & 97.99\,\% & 151 \\
\end{tabular}

\label{tab:margin_exp}
\end{table}

In our analysis, we employed a Bonferroni correction to account for this multiplicity. However, the application of such a correction incurs a certain ``cost'' and may not always yield superior results. For instance, as evidenced in Table \ref{tab:margin_exp} of our study, we observed a trend wherein the more margins incorporated, the more conservative our procedure became. Consequently, in our specific case, it was more efficient to apply a single uniform margin in all directions.

It is crucial to underscore that this conclusion is not universally applicable across all datasets and models. For example, in a dataset where objects predominantly align along the same row, a more prudent approach might involve applying larger corrections horizontally and smaller vertically. This experimental approach is not resource-intensive and can be readily replicated on different datasets using our publicly available repository.

\begin{figure}[htb!] 
    \centering
    \includegraphics[width=\columnwidth]{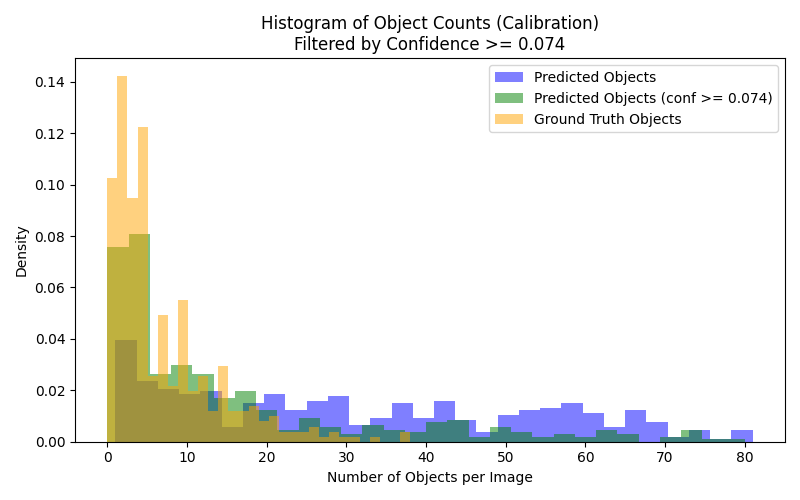}
    \caption{Histogram of the number of objects per image for Ground Truths (orange), Predictions with only pre-filtering (Blue), and Prediction after confidence threshold (with the Box Count Recall loss, $\alphacnf=0.02$ and DETR-50)}
    \label{fig:histo-set-size}
\end{figure}

\subsubsection{Distribution of Object Counts}
In this section we visualize on Fig.~\ref{fig:histo-set-size} the empirical distribution of the number of objects for the ground truths and predictions. The predictions are obtained using DETR-101 and are displayed in two variations. The  blue histogram corresponds to the number of objects only after filtering out those below $0.001$ confidence, and the green histogram using a confidence threshold obtained for the box count recall loss with $\alphacnf=0.02$. 
First, notice that the number of predicted objects can be quite large, and is heavy-tailed.
We then notice that while the obtained confidence threshold $0.074$ appears quite low, the distribution of counts seems to get closer to the ground truth distribution (plotted in orange).
Indeed, it remains conservative with regards to the ground truth object counts generally, while not being too loose in many cases. 
Additionally, while choosing the standard threshold $0.5$ leads to a distribution (not displayed here for clarity) matching the ground truth one even more, it is not guaranteed to contain, with high probability, enough objects.

\section{Effect of Object Detection and Conformal Settings}
\label{app:od-cp}

Obtaining statistically valid predictions in \odf is a multifaceted problem, dependent upon the desired level of guarantee and the specific configuration of object detectors. 
A pivotal issue concerns the ambiguity of the concept of ``ground truth''. 
This can be conceptualized as either a set of bounding boxes for an image or as individual bounding boxes. 

In this paper, we adopt the latter approach, the box-level setting, for its clarity in comparison to the set-based approach. 

Our study also touches upon the levels of guarantees in terms of confidence, localization, and classification within \od. 
Notably, classification and localization are inherently box-level concepts. 
Classification presupposes the correct identification of an object within a predicted bounding box, rendering predictions meaningless in the absence of an actual object. 
In our \cpf approach, we primarily control for the classification accuracy of the predicted bounding box nearest to the ground truth, not taking into account the classification accuracy of other predicted boxes that might have more accurately identified the object, albeit with less precise coordinates. 
Regarding confidence thresholds, formalization is challenging due to it not actually being a ground truth. In this context, we opted for an image-level guarantee, ensuring that, with a probability of at least $1-\alpha$, the number of predictions matches or exceeds the number of ground truths, irrespective of their spatial accuracy.

Lastly, regarding the $x_1y_1 x_2y_2$ (coordinates of two corners) versus $xywh$ (coordinates of the center + width and height) bounding box representations, while fundamentally equivalent, their application within \cp algorithms yields divergent results. 
The $xywh$ format inherently involves symmetric corrections and a translational adjustment, whereas the $xyxy$ format necessitates independent corrections for each side. However, these approaches can be reformulated to yield equivalent predictive sets.

\newpage
  
    \begin{table*}[htbp]
    \centering
    \tiny
    \caption{Experimental Results using the DETR-101 model for $\alphatot=0.1$}
    \label{tab:detr101_results_alpha_01_styled}
    \begin{tabular}{@{}lllll rrcrrcr@{}}
    \toprule
    \multicolumn{1}{c}{\multirow{2}{*}{\shortstack{ Matching\\Function}}} & 
    \multicolumn{1}{c}{\multirow{2}{*}{\shortstack{ Conf.\\Loss}}} & 
    \multicolumn{1}{c}{\multirow{2}{*}{\shortstack{ Loc.\\Loss}}} & 
    \multicolumn{1}{c}{\multirow{2}{*}{\shortstack{ Cls.\\Pred.\\Set}}} & 
    \multicolumn{1}{c}{\multirow{2}{*}{\shortstack{ Loc.\\Pred.\\Set.}}} & 
    \multicolumn{2}{c}{Confidence} & \multicolumn{2}{c}{Localization} & 
    \multicolumn{2}{c}{Classification} & \multicolumn{1}{c}{Global} \\
    \cmidrule(lr){6-7} \cmidrule(lr){8-9} \cmidrule(lr){10-11} \cmidrule(lr){12-12}
    & & & & & \multicolumn{1}{c}{Size} & \multicolumn{1}{c}{Risk} & 
    \multicolumn{1}{c}{Size} & \multicolumn{1}{c}{Risk} & 
    \multicolumn{1}{c}{Size} & \multicolumn{1}{c}{Risk} & 
    \multicolumn{1}{c}{Risk} \\
    \midrule
    mix & box\_count\_threshold & thresholded & lac & additive & 25.5884 & 0.022 & 4.5646 & 0.0472 & 12.7719 & 0.0449 & 0.0858 \\
mix & box\_count\_threshold & thresholded & lac & multiplicative & 25.5884 & 0.022 & 2.717 & 0.0372 & 12.7719 & 0.0449 & 0.0791 \\
mix & box\_count\_threshold & thresholded & aps & additive & 25.5884 & 0.022 & 4.5646 & 0.0416 & 17.9396 & 0.037 & 0.0727 \\
mix & box\_count\_threshold & thresholded & aps & multiplicative & 25.5884 & 0.022 & 2.2074 & 0.0404 & 11.1172 & 0.0459 & 0.0818 \\
mix & box\_count\_threshold & pixelwise & lac & additive & 25.5884 & 0.022 & 1.4788 & 0.0488 & 8.228 & 0.0464 & 0.083 \\
mix & box\_count\_threshold & pixelwise & lac & multiplicative & 25.5884 & 0.022 & 1.3337 & 0.0474 & 8.228 & 0.0464 & 0.0835 \\
mix & box\_count\_threshold & pixelwise & aps & additive & 25.5884 & 0.022 & 1.4788 & 0.0488 & 11.1172 & 0.0459 & 0.0829 \\
mix & box\_count\_threshold & pixelwise & aps & multiplicative & 25.5884 & 0.022 & 1.3337 & 0.0474 & 11.1171 & 0.0459 & 0.0835 \\
mix & box\_count\_threshold & boxwise & lac & additive & 25.5884 & 0.022 & 3.8942 & 0.0464 & 8.228 & 0.0464 & 0.0799 \\
mix & box\_count\_threshold & boxwise & lac & multiplicative & 25.5884 & 0.022 & 2.086 & 0.0468 & 8.228 & 0.0464 & 0.0838 \\
mix & box\_count\_threshold & boxwise & aps & additive & 25.5884 & 0.022 & 3.8942 & 0.0464 & 11.1172 & 0.0459 & 0.0794 \\
mix & box\_count\_threshold & boxwise & aps & multiplicative & 25.5884 & 0.022 & 2.086 & 0.0468 & 11.1171 & 0.0459 & 0.0837 \\
mix & box\_count\_recall & thresholded & lac & additive & 17.778 & 0.0185 & 4.7833 & 0.0468 & 12.0055 & 0.0471 & 0.0847 \\
mix & box\_count\_recall & thresholded & lac & multiplicative & 17.778 & 0.0185 & 2.6168 & 0.0324 & 12.0055 & 0.0471 & 0.0744 \\
mix & box\_count\_recall & thresholded & aps & additive & 17.778 & 0.0185 & 4.7833 & 0.0468 & 14.9545 & 0.0468 & 0.0844 \\
mix & box\_count\_recall & thresholded & aps & multiplicative & 17.778 & 0.0185 & 2.6168 & 0.0324 & 14.9545 & 0.0468 & 0.0743 \\
mix & box\_count\_recall & pixelwise & lac & additive & 17.778 & 0.0185 & 1.7924 & 0.0475 & 12.0055 & 0.0471 & 0.0807 \\
mix & box\_count\_recall & pixelwise & lac & multiplicative & 17.778 & 0.0185 & 1.5503 & 0.0443 & 12.0055 & 0.0471 & 0.0802 \\
mix & box\_count\_recall & pixelwise & aps & additive & 17.778 & 0.0185 & 1.7924 & 0.0475 & 14.9545 & 0.0468 & 0.0809 \\
mix & box\_count\_recall & pixelwise & aps & multiplicative & 17.778 & 0.0185 & 1.5503 & 0.0443 & 14.9545 & 0.0468 & 0.0805 \\
mix & box\_count\_recall & boxwise & lac & additive & 17.778 & 0.0185 & 4.1997 & 0.0476 & 12.0055 & 0.0471 & 0.0798 \\
mix & box\_count\_recall & boxwise & lac & multiplicative & 17.778 & 0.0185 & 2.4471 & 0.0432 & 12.0055 & 0.0471 & 0.0801 \\
mix & box\_count\_recall & boxwise & aps & additive & 17.778 & 0.0185 & 4.1997 & 0.0476 & 14.9545 & 0.0468 & 0.0799 \\
mix & box\_count\_recall & boxwise & aps & multiplicative & 17.778 & 0.0185 & 2.4471 & 0.0432 & 14.9545 & 0.0468 & 0.0804 \\
hausdorff & box\_count\_threshold & thresholded & lac & additive & 25.5884 & 0.022 & 3.1521 & 0.0492 & 42.5629 & 0.0539 & 0.1001 \\
hausdorff & box\_count\_threshold & thresholded & lac & multiplicative & 25.5884 & 0.022 & 1.5522 & 0.0456 & 42.5629 & 0.0539 & 0.0971 \\
hausdorff & box\_count\_threshold & thresholded & aps & additive & 25.5884 & 0.022 & 3.1521 & 0.0492 & 41.8463 & 0.0548 & 0.1009 \\
hausdorff & box\_count\_threshold & thresholded & aps & multiplicative & 25.5884 & 0.022 & 1.5522 & 0.0456 & 41.8461 & 0.0548 & 0.0978 \\
hausdorff & box\_count\_threshold & pixelwise & lac & additive & 25.5884 & 0.022 & 1.047 & 0.052 & 42.5629 & 0.0539 & 0.099 \\
hausdorff & box\_count\_threshold & pixelwise & lac & multiplicative & 25.5884 & 0.022 & 1.0425 & 0.0466 & 42.5629 & 0.0539 & 0.0947 \\
hausdorff & box\_count\_threshold & pixelwise & aps & additive & 25.5884 & 0.022 & 1.047 & 0.052 & 41.8463 & 0.0548 & 0.1 \\
hausdorff & box\_count\_threshold & pixelwise & aps & multiplicative & 25.5884 & 0.022 & 1.0425 & 0.0466 & 41.846 & 0.0548 & 0.0956 \\
hausdorff & box\_count\_threshold & boxwise & lac & additive & 25.5884 & 0.022 & 2.6066 & 0.0506 & 42.5629 & 0.0539 & 0.0968 \\
hausdorff & box\_count\_threshold & boxwise & lac & multiplicative & 25.5884 & 0.022 & 1.5036 & 0.049 & 42.5629 & 0.0539 & 0.0971 \\
hausdorff & box\_count\_threshold & boxwise & aps & additive & 25.5884 & 0.022 & 2.6066 & 0.0506 & 41.8463 & 0.0548 & 0.0977 \\
hausdorff & box\_count\_threshold & boxwise & aps & multiplicative & 25.5884 & 0.022 & 1.5036 & 0.049 & 41.846 & 0.0548 & 0.0981 \\
hausdorff & box\_count\_recall & thresholded & lac & additive & 17.778 & 0.0185 & 3.8291 & 0.0472 & 38.134 & 0.0515 & 0.0948 \\
hausdorff & box\_count\_recall & thresholded & lac & multiplicative & 17.778 & 0.0185 & 1.8169 & 0.0396 & 38.134 & 0.0515 & 0.0883 \\
hausdorff & box\_count\_recall & thresholded & aps & additive & 17.778 & 0.0185 & 3.8291 & 0.0472 & 37.16 & 0.0543 & 0.0978 \\
hausdorff & box\_count\_recall & thresholded & aps & multiplicative & 17.778 & 0.0185 & 1.8169 & 0.0396 & 37.1598 & 0.0543 & 0.0911 \\
hausdorff & box\_count\_recall & pixelwise & lac & additive & 17.778 & 0.0185 & 1.1577 & 0.051 & 38.134 & 0.0515 & 0.0946 \\
hausdorff & box\_count\_recall & pixelwise & lac & multiplicative & 17.778 & 0.0185 & 1.114 & 0.0506 & 38.134 & 0.0515 & 0.0953 \\
hausdorff & box\_count\_recall & pixelwise & aps & additive & 17.778 & 0.0185 & 1.1577 & 0.051 & 37.16 & 0.0543 & 0.0976 \\
hausdorff & box\_count\_recall & pixelwise & aps & multiplicative & 17.778 & 0.0185 & 1.114 & 0.0506 & 37.1598 & 0.0543 & 0.0982 \\
hausdorff & box\_count\_recall & boxwise & lac & additive & 17.778 & 0.0185 & 2.9928 & 0.0522 & 38.134 & 0.0515 & 0.0947 \\
hausdorff & box\_count\_recall & boxwise & lac & multiplicative & 17.778 & 0.0185 & 1.6957 & 0.0492 & 38.134 & 0.0515 & 0.0938 \\
hausdorff & box\_count\_recall & boxwise & aps & additive & 17.778 & 0.0185 & 2.9928 & 0.0522 & 37.16 & 0.0543 & 0.0978 \\
hausdorff & box\_count\_recall & boxwise & aps & multiplicative & 17.778 & 0.0185 & 1.6957 & 0.0492 & 37.1599 & 0.0543 & 0.0967 \\
lac & box\_count\_threshold & thresholded & lac & additive & 25.5884 & 0.022 & 21.6022 & 0.0408 & 0.9939 & 0.0508 & 0.0874 \\
lac & box\_count\_threshold & thresholded & lac & multiplicative & 25.5884 & 0.022 & 18.7336 & 0.0488 & 0.9939 & 0.0508 & 0.0942 \\
lac & box\_count\_threshold & thresholded & aps & additive & 25.5884 & 0.022 & 21.6022 & 0.0408 & 1.0066 & 0.0505 & 0.0871 \\
lac & box\_count\_threshold & thresholded & aps & multiplicative & 25.5884 & 0.022 & 18.7336 & 0.0488 & 1.0066 & 0.0505 & 0.0939 \\
lac & box\_count\_threshold & pixelwise & lac & additive & 25.5884 & 0.022 & 17.1966 & 0.0459 & 0.9939 & 0.0508 & 0.0855 \\
lac & box\_count\_threshold & pixelwise & lac & multiplicative & 25.5884 & 0.022 & 14.147 & 0.0501 & 0.9939 & 0.0508 & 0.0892 \\
lac & box\_count\_threshold & pixelwise & aps & additive & 25.5884 & 0.022 & 17.1966 & 0.0459 & 1.0066 & 0.0505 & 0.0852 \\
lac & box\_count\_threshold & pixelwise & aps & multiplicative & 25.5884 & 0.022 & 14.147 & 0.0501 & 1.0066 & 0.0505 & 0.0889 \\
lac & box\_count\_threshold & boxwise & lac & additive & 25.5884 & 0.022 & 20.4218 & 0.0431 & 0.9939 & 0.0508 & 0.082 \\
lac & box\_count\_threshold & boxwise & lac & multiplicative & 25.5884 & 0.022 & 16.2098 & 0.05 & 0.9939 & 0.0508 & 0.0878 \\
lac & box\_count\_threshold & boxwise & aps & additive & 25.5884 & 0.022 & 20.4218 & 0.0431 & 1.0066 & 0.0505 & 0.0816 \\
lac & box\_count\_threshold & boxwise & aps & multiplicative & 25.5884 & 0.022 & 16.2098 & 0.05 & 1.0066 & 0.0505 & 0.0876 \\
lac & box\_count\_recall & thresholded & lac & additive & 17.778 & 0.0185 & 20.0357 & 0.0352 & 1.2434 & 0.0482 & 0.0806 \\
lac & box\_count\_recall & thresholded & lac & multiplicative & 17.778 & 0.0185 & 20.0441 & 0.0404 & 1.2434 & 0.0482 & 0.0855 \\
lac & box\_count\_recall & thresholded & aps & additive & 17.778 & 0.0185 & 20.0357 & 0.0352 & 1.3256 & 0.0474 & 0.0797 \\
lac & box\_count\_recall & thresholded & aps & multiplicative & 17.778 & 0.0185 & 20.0441 & 0.0404 & 1.3256 & 0.0474 & 0.0846 \\
lac & box\_count\_recall & pixelwise & lac & additive & 17.778 & 0.0185 & 15.8575 & 0.0427 & 1.2434 & 0.0482 & 0.0826 \\
lac & box\_count\_recall & pixelwise & lac & multiplicative & 17.778 & 0.0185 & 15.1964 & 0.0457 & 1.2434 & 0.0482 & 0.0852 \\
lac & box\_count\_recall & pixelwise & aps & additive & 17.778 & 0.0185 & 15.8575 & 0.0427 & 1.3256 & 0.0474 & 0.082 \\
lac & box\_count\_recall & pixelwise & aps & multiplicative & 17.778 & 0.0185 & 15.1964 & 0.0457 & 1.3256 & 0.0474 & 0.0849 \\
lac & box\_count\_recall & boxwise & lac & additive & 17.778 & 0.0185 & 18.7786 & 0.0395 & 1.2434 & 0.0482 & 0.0785 \\
lac & box\_count\_recall & boxwise & lac & multiplicative & 17.778 & 0.0185 & 17.596 & 0.044 & 1.2434 & 0.0482 & 0.0822 \\
lac & box\_count\_recall & boxwise & aps & additive & 17.778 & 0.0185 & 18.7786 & 0.0395 & 1.3256 & 0.0474 & 0.078 \\
lac & box\_count\_recall & boxwise & aps & multiplicative & 17.778 & 0.0185 & 17.596 & 0.044 & 1.3256 & 0.0474 & 0.0819 \\
giou & box\_count\_threshold & thresholded & lac & additive & 25.5884 & 0.022 & 35.5329 & 0.0224 & 44.5444 & 0.0482 & 0.0697 \\
giou & box\_count\_threshold & thresholded & aps & additive & 25.5884 & 0.022 & 35.5329 & 0.0224 & 43.6363 & 0.0491 & 0.0707 \\
giou & box\_count\_threshold & pixelwise & lac & additive & 25.5884 & 0.022 & 31.4165 & 0.0359 & 44.5444 & 0.0482 & 0.0815 \\
giou & box\_count\_threshold & pixelwise & lac & multiplicative & 25.5884 & 0.022 & 158.2964 & 0.0263 & 44.5444 & 0.0482 & 0.0738 \\
giou & box\_count\_threshold & pixelwise & aps & additive & 25.5884 & 0.022 & 31.4165 & 0.0359 & 43.6364 & 0.0491 & 0.0822 \\
giou & box\_count\_threshold & pixelwise & aps & multiplicative & 25.5884 & 0.022 & 158.2964 & 0.0263 & 43.6363 & 0.0491 & 0.0748 \\
giou & box\_count\_threshold & boxwise & lac & additive & 25.5884 & 0.022 & 35.0271 & 0.0305 & 44.5444 & 0.0482 & 0.0764 \\
giou & box\_count\_threshold & boxwise & lac & multiplicative & 25.5884 & 0.022 & 183.4134 & 0.0239 & 44.5444 & 0.0482 & 0.0715 \\
giou & box\_count\_threshold & boxwise & aps & additive & 25.5884 & 0.022 & 35.0271 & 0.0305 & 43.6364 & 0.0491 & 0.0772 \\
giou & box\_count\_threshold & boxwise & aps & multiplicative & 25.5884 & 0.022 & 183.4134 & 0.0239 & 43.6363 & 0.0491 & 0.0724 \\
giou & box\_count\_recall & thresholded & lac & additive & 17.778 & 0.0185 & 31.8864 & 0.0144 & 44.6575 & 0.0458 & 0.0596 \\
giou & box\_count\_recall & thresholded & aps & additive & 17.778 & 0.0185 & 31.8864 & 0.0144 & 44.4713 & 0.0456 & 0.0593 \\
giou & box\_count\_recall & pixelwise & lac & additive & 17.778 & 0.0185 & 28.2407 & 0.0279 & 44.6575 & 0.0458 & 0.0716 \\
giou & box\_count\_recall & pixelwise & lac & multiplicative & 17.778 & 0.0185 & 165.6919 & 0.0114 & 44.6575 & 0.0458 & 0.0566 \\
giou & box\_count\_recall & pixelwise & aps & additive & 17.778 & 0.0185 & 28.2407 & 0.0279 & 44.4707 & 0.0456 & 0.0712 \\
giou & box\_count\_recall & pixelwise & aps & multiplicative & 17.778 & 0.0185 & 165.6919 & 0.0114 & 44.4713 & 0.0456 & 0.0565 \\
giou & box\_count\_recall & boxwise & lac & additive & 17.778 & 0.0185 & 31.4174 & 0.0204 & 44.6575 & 0.0458 & 0.0646 \\
giou & box\_count\_recall & boxwise & lac & multiplicative & 17.778 & 0.0185 & 192.0147 & 0.01 & 44.6575 & 0.0458 & 0.0553 \\
giou & box\_count\_recall & boxwise & aps & additive & 17.778 & 0.0185 & 31.4174 & 0.0204 & 44.4707 & 0.0456 & 0.0642 \\
giou & box\_count\_recall & boxwise & aps & multiplicative & 17.778 & 0.0185 & 192.0147 & 0.01 & 44.4713 & 0.0456 & 0.0551 \\
    \bottomrule
    \end{tabular}
    \end{table*}

    \begin{table*}[htbp]
    \centering
    \tiny
    \caption{Experimental Results using the DETR-101 model for $\alphatot=0.2$}
    \label{tab:detr101_results_alpha_02_styled}
    \begin{tabular}{@{}lllll rrcrrcr@{}}
    \toprule
    \multicolumn{1}{c}{\multirow{2}{*}{\shortstack{ Matching\\Function}}} & 
    \multicolumn{1}{c}{\multirow{2}{*}{\shortstack{ Conf.\\Loss}}} & 
    \multicolumn{1}{c}{\multirow{2}{*}{\shortstack{ Loc.\\Loss}}} & 
    \multicolumn{1}{c}{\multirow{2}{*}{\shortstack{ Cls.\\Pred.\\Set}}} & 
    \multicolumn{1}{c}{\multirow{2}{*}{\shortstack{ Loc.\\Pred.\\Set.}}} & 
    \multicolumn{2}{c}{Confidence} & \multicolumn{2}{c}{Localization} & 
    \multicolumn{2}{c}{Classification} & \multicolumn{1}{c}{Global} \\
    \cmidrule(lr){6-7} \cmidrule(lr){8-9} \cmidrule(lr){10-11} \cmidrule(lr){12-12}
    & & & & & \multicolumn{1}{c}{Size} & \multicolumn{1}{c}{Risk} & 
    \multicolumn{1}{c}{Size} & \multicolumn{1}{c}{Risk} & 
    \multicolumn{1}{c}{Size} & \multicolumn{1}{c}{Risk} & 
    \multicolumn{1}{c}{Risk} \\
    \midrule
    mix & box\_count\_threshold & thresholded & lac & additive & 22.6572 & 0.0268 & 2.5801 & 0.0976 & 0.9313 & 0.0878 & 0.1663 \\
mix & box\_count\_threshold & thresholded & lac & multiplicative & 22.6572 & 0.0268 & 1.7213 & 0.0884 & 0.9313 & 0.0878 & 0.1628 \\
mix & box\_count\_threshold & thresholded & aps & additive & 22.6572 & 0.0268 & 2.5801 & 0.0976 & 0.9936 & 0.0813 & 0.1609 \\
mix & box\_count\_threshold & thresholded & aps & multiplicative & 22.6572 & 0.0268 & 1.7213 & 0.0884 & 0.9936 & 0.0813 & 0.1582 \\
mix & box\_count\_threshold & pixelwise & lac & additive & 22.6572 & 0.0268 & 1.03 & 0.0981 & 0.9313 & 0.0878 & 0.1522 \\
mix & box\_count\_threshold & pixelwise & lac & multiplicative & 22.6572 & 0.0268 & 1.0179 & 0.0977 & 0.9313 & 0.0878 & 0.1524 \\
mix & box\_count\_threshold & pixelwise & aps & additive & 22.6572 & 0.0268 & 1.03 & 0.0981 & 0.9936 & 0.0813 & 0.149 \\
mix & box\_count\_threshold & pixelwise & aps & multiplicative & 22.6572 & 0.0268 & 1.0179 & 0.0977 & 0.9936 & 0.0813 & 0.1492 \\
mix & box\_count\_threshold & boxwise & lac & additive & 22.6572 & 0.0268 & 2.1558 & 0.0957 & 0.9313 & 0.0878 & 0.1477 \\
mix & box\_count\_threshold & boxwise & lac & multiplicative & 22.6572 & 0.0268 & 1.5515 & 0.0938 & 0.9313 & 0.0878 & 0.151 \\
mix & box\_count\_threshold & boxwise & aps & additive & 22.6572 & 0.0268 & 2.1558 & 0.0957 & 0.9936 & 0.0813 & 0.1439 \\
mix & box\_count\_threshold & boxwise & aps & multiplicative & 22.6572 & 0.0268 & 1.5515 & 0.0938 & 0.9936 & 0.0813 & 0.1477 \\
mix & box\_count\_recall & thresholded & lac & additive & 14.0464 & 0.0264 & 2.744 & 0.0988 & 1.064 & 0.0918 & 0.1675 \\
mix & box\_count\_recall & thresholded & lac & multiplicative & 14.0464 & 0.0264 & 1.9139 & 0.0768 & 1.064 & 0.0918 & 0.1538 \\
mix & box\_count\_recall & thresholded & aps & additive & 14.0464 & 0.0264 & 2.744 & 0.0988 & 1.0787 & 0.0918 & 0.1674 \\
mix & box\_count\_recall & thresholded & aps & multiplicative & 14.0464 & 0.0264 & 1.9139 & 0.0768 & 1.0787 & 0.0918 & 0.1537 \\
mix & box\_count\_recall & pixelwise & lac & additive & 14.0464 & 0.0264 & 1.0842 & 0.095 & 1.064 & 0.0918 & 0.1508 \\
mix & box\_count\_recall & pixelwise & lac & multiplicative & 14.0464 & 0.0264 & 1.0655 & 0.0967 & 1.064 & 0.0918 & 0.1538 \\
mix & box\_count\_recall & pixelwise & aps & additive & 14.0464 & 0.0264 & 1.0842 & 0.095 & 1.0787 & 0.0918 & 0.1508 \\
mix & box\_count\_recall & pixelwise & aps & multiplicative & 14.0464 & 0.0264 & 1.0655 & 0.0967 & 1.0787 & 0.0918 & 0.1538 \\
mix & box\_count\_recall & boxwise & lac & additive & 14.0464 & 0.0264 & 2.3195 & 0.0972 & 1.064 & 0.0918 & 0.1498 \\
mix & box\_count\_recall & boxwise & lac & multiplicative & 14.0464 & 0.0264 & 1.6957 & 0.0899 & 1.064 & 0.0918 & 0.1508 \\
mix & box\_count\_recall & boxwise & aps & additive & 14.0464 & 0.0264 & 2.3195 & 0.0972 & 1.0787 & 0.0918 & 0.1498 \\
mix & box\_count\_recall & boxwise & aps & multiplicative & 14.0464 & 0.0264 & 1.6957 & 0.0899 & 1.0787 & 0.0918 & 0.1509 \\
hausdorff & box\_count\_threshold & thresholded & lac & additive & 22.6572 & 0.0268 & 1.833 & 0.1024 & 25.3486 & 0.1012 & 0.19 \\
hausdorff & box\_count\_threshold & thresholded & lac & multiplicative & 22.6572 & 0.0268 & 1.3817 & 0.1112 & 25.3486 & 0.1012 & 0.2023 \\
hausdorff & box\_count\_threshold & thresholded & aps & additive & 22.6572 & 0.0268 & 1.833 & 0.1024 & 27.6101 & 0.1012 & 0.1901 \\
hausdorff & box\_count\_threshold & thresholded & aps & multiplicative & 22.6572 & 0.0268 & 1.3817 & 0.1112 & 27.6102 & 0.1012 & 0.2028 \\
hausdorff & box\_count\_threshold & pixelwise & lac & additive & 22.6572 & 0.0268 & 1.0009 & 0.0704 & 25.3486 & 0.1012 & 0.1536 \\
hausdorff & box\_count\_threshold & pixelwise & lac & multiplicative & 22.6572 & 0.0268 & 1.0179 & 0.0611 & 25.3486 & 0.1012 & 0.1463 \\
hausdorff & box\_count\_threshold & pixelwise & aps & additive & 22.6572 & 0.0268 & 1.0009 & 0.0704 & 27.6101 & 0.1012 & 0.1548 \\
hausdorff & box\_count\_threshold & pixelwise & aps & multiplicative & 22.6572 & 0.0268 & 1.0179 & 0.0611 & 27.6102 & 0.1012 & 0.1474 \\
hausdorff & box\_count\_threshold & boxwise & lac & additive & 22.6572 & 0.0268 & 1.6399 & 0.099 & 25.3486 & 0.1012 & 0.1742 \\
hausdorff & box\_count\_threshold & boxwise & lac & multiplicative & 22.6572 & 0.0268 & 1.2847 & 0.1014 & 25.3486 & 0.1012 & 0.1822 \\
hausdorff & box\_count\_threshold & boxwise & aps & additive & 22.6572 & 0.0268 & 1.6399 & 0.099 & 27.6101 & 0.1012 & 0.1752 \\
hausdorff & box\_count\_threshold & boxwise & aps & multiplicative & 22.6572 & 0.0268 & 1.2847 & 0.1014 & 27.6102 & 0.1012 & 0.1834 \\
hausdorff & box\_count\_recall & thresholded & lac & additive & 14.0464 & 0.0264 & 2.1505 & 0.1024 & 22.0349 & 0.0986 & 0.1871 \\
hausdorff & box\_count\_recall & thresholded & lac & multiplicative & 14.0464 & 0.0264 & 1.526 & 0.096 & 22.0349 & 0.0986 & 0.1852 \\
hausdorff & box\_count\_recall & thresholded & aps & additive & 14.0464 & 0.0264 & 2.1505 & 0.1024 & 23.6444 & 0.0989 & 0.1881 \\
hausdorff & box\_count\_recall & thresholded & aps & multiplicative & 14.0464 & 0.0264 & 1.526 & 0.096 & 23.6444 & 0.0989 & 0.1859 \\
hausdorff & box\_count\_recall & pixelwise & lac & additive & 14.0464 & 0.0264 & 0.999 & 0.0914 & 22.0349 & 0.0986 & 0.1677 \\
hausdorff & box\_count\_recall & pixelwise & lac & multiplicative & 14.0464 & 0.0264 & 1.017 & 0.0813 & 22.0349 & 0.0986 & 0.1597 \\
hausdorff & box\_count\_recall & pixelwise & aps & additive & 14.0464 & 0.0264 & 0.999 & 0.0914 & 23.6444 & 0.0989 & 0.169 \\
hausdorff & box\_count\_recall & pixelwise & aps & multiplicative & 14.0464 & 0.0264 & 1.017 & 0.0813 & 23.6444 & 0.0989 & 0.161 \\
hausdorff & box\_count\_recall & boxwise & lac & additive & 14.0464 & 0.0264 & 1.7532 & 0.1016 & 22.0349 & 0.0986 & 0.1739 \\
hausdorff & box\_count\_recall & boxwise & lac & multiplicative & 14.0464 & 0.0264 & 1.4049 & 0.1 & 22.0349 & 0.0986 & 0.1763 \\
hausdorff & box\_count\_recall & boxwise & aps & additive & 14.0464 & 0.0264 & 1.7532 & 0.1016 & 23.6444 & 0.0989 & 0.1748 \\
hausdorff & box\_count\_recall & boxwise & aps & multiplicative & 14.0464 & 0.0264 & 1.4049 & 0.1 & 23.6444 & 0.0989 & 0.1778 \\
lac & box\_count\_threshold & thresholded & lac & additive & 22.6572 & 0.0268 & 17.2748 & 0.0824 & 0.6533 & 0.0989 & 0.1694 \\
lac & box\_count\_threshold & thresholded & lac & multiplicative & 22.6572 & 0.0268 & 12.0794 & 0.1056 & 0.6533 & 0.0989 & 0.1872 \\
lac & box\_count\_threshold & thresholded & aps & additive & 22.6572 & 0.0268 & 17.2748 & 0.0824 & 0.9936 & 0.0548 & 0.1292 \\
lac & box\_count\_threshold & thresholded & aps & multiplicative & 22.6572 & 0.0268 & 12.0794 & 0.1056 & 0.9936 & 0.0548 & 0.1515 \\
lac & box\_count\_threshold & pixelwise & lac & additive & 22.6572 & 0.0268 & 12.277 & 0.0919 & 0.6533 & 0.0989 & 0.1594 \\
lac & box\_count\_threshold & pixelwise & lac & multiplicative & 22.6572 & 0.0268 & 7.7858 & 0.0999 & 0.6533 & 0.0989 & 0.1656 \\
lac & box\_count\_threshold & pixelwise & aps & additive & 22.6572 & 0.0268 & 12.277 & 0.0919 & 0.9936 & 0.0548 & 0.1268 \\
lac & box\_count\_threshold & pixelwise & aps & multiplicative & 22.6572 & 0.0268 & 7.7858 & 0.0999 & 0.9936 & 0.0548 & 0.1355 \\
lac & box\_count\_threshold & boxwise & lac & additive & 22.6572 & 0.0268 & 15.4127 & 0.0901 & 0.6533 & 0.0989 & 0.1565 \\
lac & box\_count\_threshold & boxwise & lac & multiplicative & 22.6572 & 0.0268 & 9.1685 & 0.0984 & 0.6533 & 0.0989 & 0.1631 \\
lac & box\_count\_threshold & boxwise & aps & additive & 22.6572 & 0.0268 & 15.4127 & 0.0901 & 0.9936 & 0.0548 & 0.1232 \\
lac & box\_count\_threshold & boxwise & aps & multiplicative & 22.6572 & 0.0268 & 9.1685 & 0.0984 & 0.9936 & 0.0548 & 0.1328 \\
lac & box\_count\_recall & thresholded & lac & additive & 14.0464 & 0.0264 & 15.1958 & 0.0796 & 0.817 & 0.1014 & 0.1689 \\
lac & box\_count\_recall & thresholded & lac & multiplicative & 14.0464 & 0.0264 & 12.9423 & 0.0956 & 0.817 & 0.1014 & 0.1805 \\
lac & box\_count\_recall & thresholded & aps & additive & 14.0464 & 0.0264 & 15.1958 & 0.0796 & 0.9928 & 0.0699 & 0.1399 \\
lac & box\_count\_recall & thresholded & aps & multiplicative & 14.0464 & 0.0264 & 12.9423 & 0.0956 & 0.9928 & 0.0699 & 0.1537 \\
lac & box\_count\_recall & pixelwise & lac & additive & 14.0464 & 0.0264 & 10.9039 & 0.0887 & 0.817 & 0.1014 & 0.1586 \\
lac & box\_count\_recall & pixelwise & lac & multiplicative & 14.0464 & 0.0264 & 8.2643 & 0.0943 & 0.817 & 0.1014 & 0.1641 \\
lac & box\_count\_recall & pixelwise & aps & additive & 14.0464 & 0.0264 & 10.9039 & 0.0887 & 0.9928 & 0.0699 & 0.1356 \\
lac & box\_count\_recall & pixelwise & aps & multiplicative & 14.0464 & 0.0264 & 8.2643 & 0.0943 & 0.9928 & 0.0699 & 0.1415 \\
lac & box\_count\_recall & boxwise & lac & additive & 14.0464 & 0.0264 & 13.6096 & 0.0852 & 0.817 & 0.1014 & 0.155 \\
lac & box\_count\_recall & boxwise & lac & multiplicative & 14.0464 & 0.0264 & 9.8398 & 0.0912 & 0.817 & 0.1014 & 0.1602 \\
lac & box\_count\_recall & boxwise & aps & additive & 14.0464 & 0.0264 & 13.6096 & 0.0852 & 0.9928 & 0.0699 & 0.1314 \\
lac & box\_count\_recall & boxwise & aps & multiplicative & 14.0464 & 0.0264 & 9.8398 & 0.0912 & 0.9928 & 0.0699 & 0.1376 \\
giou & box\_count\_threshold & thresholded & lac & additive & 22.6572 & 0.0268 & 32.7365 & 0.0524 & 34.0559 & 0.0876 & 0.1366 \\
giou & box\_count\_threshold & thresholded & lac & multiplicative & 22.6572 & 0.0268 & 147.1712 & 0.0468 & 34.0559 & 0.0876 & 0.1337 \\
giou & box\_count\_threshold & thresholded & aps & additive & 22.6572 & 0.0268 & 32.7365 & 0.0524 & 33.8632 & 0.0898 & 0.1389 \\
giou & box\_count\_threshold & thresholded & aps & multiplicative & 22.6572 & 0.0268 & 147.1712 & 0.0468 & 33.8633 & 0.0898 & 0.1359 \\
giou & box\_count\_threshold & pixelwise & lac & additive & 22.6572 & 0.0268 & 27.477 & 0.0681 & 34.0559 & 0.0876 & 0.1469 \\
giou & box\_count\_threshold & pixelwise & lac & multiplicative & 22.6572 & 0.0268 & 101.6635 & 0.0525 & 34.0559 & 0.0876 & 0.1374 \\
giou & box\_count\_threshold & pixelwise & aps & additive & 22.6572 & 0.0268 & 27.477 & 0.0681 & 33.8632 & 0.0898 & 0.1491 \\
giou & box\_count\_threshold & pixelwise & aps & multiplicative & 22.6572 & 0.0268 & 101.6635 & 0.0525 & 33.8633 & 0.0898 & 0.1394 \\
giou & box\_count\_threshold & boxwise & lac & additive & 22.6572 & 0.0268 & 31.6264 & 0.0573 & 34.0559 & 0.0876 & 0.136 \\
giou & box\_count\_threshold & boxwise & lac & multiplicative & 22.6572 & 0.0268 & 116.7034 & 0.05 & 34.0559 & 0.0876 & 0.1351 \\
giou & box\_count\_threshold & boxwise & aps & additive & 22.6572 & 0.0268 & 31.6264 & 0.0573 & 33.8632 & 0.0898 & 0.1384 \\
giou & box\_count\_threshold & boxwise & aps & multiplicative & 22.6572 & 0.0268 & 116.7034 & 0.05 & 33.8633 & 0.0898 & 0.1371 \\
giou & box\_count\_recall & thresholded & lac & additive & 14.0464 & 0.0264 & 28.1762 & 0.0312 & 32.523 & 0.0827 & 0.1122 \\
giou & box\_count\_recall & thresholded & lac & multiplicative & 14.0464 & 0.0264 & 153.597 & 0.0172 & 32.523 & 0.0827 & 0.0994 \\
giou & box\_count\_recall & thresholded & aps & additive & 14.0464 & 0.0264 & 28.1762 & 0.0312 & 32.335 & 0.0819 & 0.1113 \\
giou & box\_count\_recall & thresholded & aps & multiplicative & 14.0464 & 0.0264 & 153.597 & 0.0172 & 32.3353 & 0.0819 & 0.0986 \\
giou & box\_count\_recall & pixelwise & lac & additive & 14.0464 & 0.0264 & 23.6903 & 0.0541 & 32.523 & 0.0827 & 0.1297 \\
giou & box\_count\_recall & pixelwise & lac & multiplicative & 14.0464 & 0.0264 & 104.6357 & 0.0243 & 32.523 & 0.0827 & 0.1056 \\
giou & box\_count\_recall & pixelwise & aps & additive & 14.0464 & 0.0264 & 23.6903 & 0.0541 & 32.335 & 0.0819 & 0.1289 \\
giou & box\_count\_recall & pixelwise & aps & multiplicative & 14.0464 & 0.0264 & 104.6357 & 0.0243 & 32.3353 & 0.0819 & 0.1048 \\
giou & box\_count\_recall & boxwise & lac & additive & 14.0464 & 0.0264 & 27.2076 & 0.0418 & 32.523 & 0.0827 & 0.1187 \\
giou & box\_count\_recall & boxwise & lac & multiplicative & 14.0464 & 0.0264 & 120.4633 & 0.0215 & 32.523 & 0.0827 & 0.1029 \\
giou & box\_count\_recall & boxwise & aps & additive & 14.0464 & 0.0264 & 27.2076 & 0.0418 & 32.335 & 0.0819 & 0.1177 \\
giou & box\_count\_recall & boxwise & aps & multiplicative & 14.0464 & 0.0264 & 120.4633 & 0.0215 & 32.3353 & 0.0819 & 0.102 \\
    \bottomrule
    \end{tabular}
    \end{table*}

    \begin{table*}[htbp]
    \centering
    \tiny
    \caption{Experimental Results using the YOLOv8x model for $\alphatot=0.1$}
    \label{tab:yolo_results_alpha_01_styled}
    \begin{tabular}{@{}lllll rrcrrcr@{}}
    \toprule
    \multicolumn{1}{c}{\multirow{2}{*}{\shortstack{ Matching\\Function}}} & 
    \multicolumn{1}{c}{\multirow{2}{*}{\shortstack{ Conf.\\Loss}}} & 
    \multicolumn{1}{c}{\multirow{2}{*}{\shortstack{ Loc.\\Loss}}} & 
    \multicolumn{1}{c}{\multirow{2}{*}{\shortstack{ Cls.\\Pred.\\Set}}} & 
    \multicolumn{1}{c}{\multirow{2}{*}{\shortstack{ Loc.\\Pred.\\Set.}}} & 
    \multicolumn{2}{c}{Confidence} & \multicolumn{2}{c}{Localization} & 
    \multicolumn{2}{c}{Classification} & \multicolumn{1}{c}{Global} \\
    \cmidrule(lr){6-7} \cmidrule(lr){8-9} \cmidrule(lr){10-11} \cmidrule(lr){12-12}
    & & & & & \multicolumn{1}{c}{Size} & \multicolumn{1}{c}{Risk} & 
    \multicolumn{1}{c}{Size} & \multicolumn{1}{c}{Risk} & 
    \multicolumn{1}{c}{Size} & \multicolumn{1}{c}{Risk} & 
    \multicolumn{1}{c}{Risk} \\
    \midrule
    mix & box\_count\_threshold & thresholded & lac & additive & 18.8552 & 0.0124 & 6.2473 & 0.0464 & 69.8821 & 0.0517 & 0.0943 \\
mix & box\_count\_threshold & thresholded & lac & multiplicative & 18.8552 & 0.0124 & 6.6374 & 0.052 & 69.8821 & 0.0517 & 0.1021 \\
mix & box\_count\_threshold & thresholded & aps & additive & 18.8552 & 0.0124 & 6.2473 & 0.0464 & 40.7516 & 0.0418 & 0.0852 \\
mix & box\_count\_threshold & thresholded & aps & multiplicative & 18.8552 & 0.0124 & 6.6374 & 0.052 & 40.7516 & 0.0418 & 0.0929 \\
mix & box\_count\_threshold & pixelwise & lac & additive & 18.8552 & 0.0124 & 4.3435 & 0.0477 & 69.8823 & 0.0517 & 0.092 \\
mix & box\_count\_threshold & pixelwise & lac & multiplicative & 18.8552 & 0.0124 & 4.1814 & 0.0517 & 69.8823 & 0.0517 & 0.098 \\
mix & box\_count\_threshold & pixelwise & aps & additive & 18.8552 & 0.0124 & 4.3435 & 0.0477 & 40.7516 & 0.0418 & 0.0816 \\
mix & box\_count\_threshold & pixelwise & aps & multiplicative & 18.8552 & 0.0124 & 4.1814 & 0.0517 & 40.7516 & 0.0418 & 0.0885 \\
mix & box\_count\_threshold & boxwise & lac & additive & 18.8552 & 0.0124 & 6.1211 & 0.0432 & 69.8823 & 0.0517 & 0.0883 \\
mix & box\_count\_threshold & boxwise & lac & multiplicative & 18.8552 & 0.0124 & 5.8593 & 0.0495 & 69.8823 & 0.0517 & 0.0961 \\
mix & box\_count\_threshold & boxwise & aps & additive & 18.8552 & 0.0124 & 6.1211 & 0.0432 & 40.7516 & 0.0418 & 0.0781 \\
mix & box\_count\_threshold & boxwise & aps & multiplicative & 18.8552 & 0.0124 & 5.8593 & 0.0495 & 40.7516 & 0.0418 & 0.0869 \\
mix & box\_count\_recall & thresholded & lac & additive & 11.7104 & 0.0189 & 5.4582 & 0.0456 & 67.5809 & 0.0531 & 0.0936 \\
mix & box\_count\_recall & thresholded & lac & multiplicative & 11.7104 & 0.0189 & 6.961 & 0.0544 & 67.5809 & 0.0531 & 0.1061 \\
mix & box\_count\_recall & thresholded & aps & additive & 11.7104 & 0.0189 & 5.4582 & 0.0456 & 41.6644 & 0.0416 & 0.0835 \\
mix & box\_count\_recall & thresholded & aps & multiplicative & 11.7104 & 0.0189 & 6.961 & 0.0544 & 41.6644 & 0.0416 & 0.0948 \\
mix & box\_count\_recall & pixelwise & lac & additive & 11.7104 & 0.0189 & 3.867 & 0.0487 & 67.5809 & 0.0531 & 0.0935 \\
mix & box\_count\_recall & pixelwise & lac & multiplicative & 11.7104 & 0.0189 & 4.5352 & 0.0492 & 67.5809 & 0.0531 & 0.0976 \\
mix & box\_count\_recall & pixelwise & aps & additive & 11.7104 & 0.0189 & 3.867 & 0.0487 & 41.6644 & 0.0416 & 0.0814 \\
mix & box\_count\_recall & pixelwise & aps & multiplicative & 11.7104 & 0.0189 & 4.5352 & 0.0492 & 41.6644 & 0.0416 & 0.0859 \\
mix & box\_count\_recall & boxwise & lac & additive & 11.7104 & 0.0189 & 5.2321 & 0.0453 & 67.5809 & 0.0531 & 0.0902 \\
mix & box\_count\_recall & boxwise & lac & multiplicative & 11.7104 & 0.0189 & 6.2333 & 0.0477 & 67.5809 & 0.0531 & 0.0965 \\
mix & box\_count\_recall & boxwise & aps & additive & 11.7104 & 0.0189 & 5.2321 & 0.0453 & 41.6644 & 0.0416 & 0.0791 \\
mix & box\_count\_recall & boxwise & aps & multiplicative & 11.7104 & 0.0189 & 6.2333 & 0.0477 & 41.6644 & 0.0416 & 0.085 \\
hausdorff & box\_count\_threshold & thresholded & lac & additive & 18.8552 & 0.0124 & 6.2473 & 0.046 & 70.2243 & 0.0523 & 0.0945 \\
hausdorff & box\_count\_threshold & thresholded & lac & multiplicative & 18.8552 & 0.0124 & 7.4398 & 0.0464 & 70.2243 & 0.0523 & 0.0975 \\
hausdorff & box\_count\_threshold & thresholded & aps & additive & 18.8552 & 0.0124 & 6.2473 & 0.046 & 40.9013 & 0.0434 & 0.0864 \\
hausdorff & box\_count\_threshold & thresholded & aps & multiplicative & 18.8552 & 0.0124 & 7.4398 & 0.0464 & 40.9012 & 0.0434 & 0.089 \\
hausdorff & box\_count\_threshold & pixelwise & lac & additive & 18.8552 & 0.0124 & 4.3435 & 0.0477 & 70.2243 & 0.0523 & 0.0926 \\
hausdorff & box\_count\_threshold & pixelwise & lac & multiplicative & 18.8552 & 0.0124 & 4.6191 & 0.0473 & 70.2243 & 0.0523 & 0.0952 \\
hausdorff & box\_count\_threshold & pixelwise & aps & additive & 18.8552 & 0.0124 & 4.3435 & 0.0477 & 40.9013 & 0.0434 & 0.0829 \\
hausdorff & box\_count\_threshold & pixelwise & aps & multiplicative & 18.8552 & 0.0124 & 4.6191 & 0.0473 & 40.9013 & 0.0434 & 0.0862 \\
hausdorff & box\_count\_threshold & boxwise & lac & additive & 18.8552 & 0.0124 & 6.1211 & 0.0431 & 70.2243 & 0.0523 & 0.0886 \\
hausdorff & box\_count\_threshold & boxwise & lac & multiplicative & 18.8552 & 0.0124 & 6.5888 & 0.0452 & 70.2243 & 0.0523 & 0.0934 \\
hausdorff & box\_count\_threshold & boxwise & aps & additive & 18.8552 & 0.0124 & 6.1211 & 0.0431 & 40.9013 & 0.0434 & 0.0796 \\
hausdorff & box\_count\_threshold & boxwise & aps & multiplicative & 18.8552 & 0.0124 & 6.5888 & 0.0452 & 40.9012 & 0.0434 & 0.0846 \\
hausdorff & box\_count\_recall & thresholded & lac & additive & 11.7104 & 0.0189 & 5.4513 & 0.0468 & 67.9788 & 0.0539 & 0.0957 \\
hausdorff & box\_count\_recall & thresholded & lac & multiplicative & 11.7104 & 0.0189 & 8.0526 & 0.044 & 67.9788 & 0.0539 & 0.0967 \\
hausdorff & box\_count\_recall & thresholded & aps & additive & 11.7104 & 0.0189 & 5.4513 & 0.0468 & 41.7312 & 0.0428 & 0.0858 \\
hausdorff & box\_count\_recall & thresholded & aps & multiplicative & 11.7104 & 0.0189 & 8.0526 & 0.044 & 41.7312 & 0.0428 & 0.0857 \\
hausdorff & box\_count\_recall & pixelwise & lac & additive & 11.7104 & 0.0189 & 3.867 & 0.0487 & 67.9788 & 0.0539 & 0.0945 \\
hausdorff & box\_count\_recall & pixelwise & lac & multiplicative & 11.7104 & 0.0189 & 4.9961 & 0.0438 & 67.9788 & 0.0539 & 0.0939 \\
hausdorff & box\_count\_recall & pixelwise & aps & additive & 11.7104 & 0.0189 & 3.867 & 0.0487 & 41.7312 & 0.0428 & 0.0824 \\
hausdorff & box\_count\_recall & pixelwise & aps & multiplicative & 11.7104 & 0.0189 & 4.9961 & 0.0438 & 41.7312 & 0.0428 & 0.0821 \\
hausdorff & box\_count\_recall & boxwise & lac & additive & 11.7104 & 0.0189 & 5.2252 & 0.0457 & 67.9788 & 0.0539 & 0.0918 \\
hausdorff & box\_count\_recall & boxwise & lac & multiplicative & 11.7104 & 0.0189 & 6.9853 & 0.0416 & 67.9788 & 0.0539 & 0.0921 \\
hausdorff & box\_count\_recall & boxwise & aps & additive & 11.7104 & 0.0189 & 5.2252 & 0.0457 & 41.7312 & 0.0428 & 0.0808 \\
hausdorff & box\_count\_recall & boxwise & aps & multiplicative & 11.7104 & 0.0189 & 6.9853 & 0.0416 & 41.7312 & 0.0428 & 0.0807 \\
lac & box\_count\_threshold & thresholded & lac & additive & 18.8552 & 0.0124 & 25.0369 & 0.0504 & 0.7168 & 0.0489 & 0.0946 \\
lac & box\_count\_threshold & thresholded & lac & multiplicative & 18.8552 & 0.0124 & 21.6649 & 0.0428 & 0.7168 & 0.0489 & 0.0869 \\
lac & box\_count\_threshold & thresholded & aps & additive & 18.8552 & 0.0124 & 25.0369 & 0.0504 & 0.996 & 0.0428 & 0.0893 \\
lac & box\_count\_threshold & thresholded & aps & multiplicative & 18.8552 & 0.0124 & 21.6649 & 0.0428 & 0.996 & 0.0428 & 0.0817 \\
lac & box\_count\_threshold & pixelwise & lac & additive & 18.8552 & 0.0124 & 20.6405 & 0.0467 & 0.7168 & 0.0489 & 0.086 \\
lac & box\_count\_threshold & pixelwise & lac & multiplicative & 18.8552 & 0.0124 & 15.8047 & 0.0491 & 0.7168 & 0.0489 & 0.0853 \\
lac & box\_count\_threshold & pixelwise & aps & additive & 18.8552 & 0.0124 & 20.6405 & 0.0467 & 0.996 & 0.0428 & 0.0817 \\
lac & box\_count\_threshold & pixelwise & aps & multiplicative & 18.8552 & 0.0124 & 15.8047 & 0.0491 & 0.996 & 0.0428 & 0.0812 \\
lac & box\_count\_threshold & boxwise & lac & additive & 18.8552 & 0.0124 & 24.1811 & 0.0459 & 0.7168 & 0.0489 & 0.0828 \\
lac & box\_count\_threshold & boxwise & lac & multiplicative & 18.8552 & 0.0124 & 17.8473 & 0.0486 & 0.7168 & 0.0489 & 0.0838 \\
lac & box\_count\_threshold & boxwise & aps & additive & 18.8552 & 0.0124 & 24.1811 & 0.0459 & 0.996 & 0.0428 & 0.079 \\
lac & box\_count\_threshold & boxwise & aps & multiplicative & 18.8552 & 0.0124 & 17.8473 & 0.0486 & 0.996 & 0.0428 & 0.0797 \\
lac & box\_count\_recall & thresholded & lac & additive & 11.7104 & 0.0189 & 21.8857 & 0.0388 & 1.0285 & 0.0489 & 0.083 \\
lac & box\_count\_recall & thresholded & lac & multiplicative & 11.7104 & 0.0189 & 25.4212 & 0.0348 & 1.0285 & 0.0489 & 0.079 \\
lac & box\_count\_recall & thresholded & aps & additive & 11.7104 & 0.0189 & 21.8857 & 0.0388 & 1.3236 & 0.0478 & 0.0822 \\
lac & box\_count\_recall & thresholded & aps & multiplicative & 11.7104 & 0.0189 & 25.4212 & 0.0348 & 1.3236 & 0.0478 & 0.0786 \\
lac & box\_count\_recall & pixelwise & lac & additive & 11.7104 & 0.0189 & 17.6658 & 0.0437 & 1.0285 & 0.0489 & 0.0823 \\
lac & box\_count\_recall & pixelwise & lac & multiplicative & 11.7104 & 0.0189 & 17.7072 & 0.0426 & 1.0285 & 0.0489 & 0.0792 \\
lac & box\_count\_recall & pixelwise & aps & additive & 11.7104 & 0.0189 & 17.6658 & 0.0437 & 1.3236 & 0.0478 & 0.0822 \\
lac & box\_count\_recall & pixelwise & aps & multiplicative & 11.7104 & 0.0189 & 17.7072 & 0.0426 & 1.3236 & 0.0478 & 0.0795 \\
lac & box\_count\_recall & boxwise & lac & additive & 11.7104 & 0.0189 & 20.4723 & 0.0425 & 1.0285 & 0.0489 & 0.0797 \\
lac & box\_count\_recall & boxwise & lac & multiplicative & 11.7104 & 0.0189 & 20.036 & 0.0411 & 1.0285 & 0.0489 & 0.077 \\
lac & box\_count\_recall & boxwise & aps & additive & 11.7104 & 0.0189 & 20.4723 & 0.0425 & 1.3236 & 0.0478 & 0.0796 \\
lac & box\_count\_recall & boxwise & aps & multiplicative & 11.7104 & 0.0189 & 20.036 & 0.0411 & 1.3236 & 0.0478 & 0.0772 \\
giou & box\_count\_threshold & thresholded & lac & additive & 18.8552 & 0.0124 & 43.6338 & 0.016 & 65.9061 & 0.0401 & 0.0559 \\
giou & box\_count\_threshold & thresholded & aps & additive & 18.8552 & 0.0124 & 43.6338 & 0.016 & 65.4388 & 0.0268 & 0.0421 \\
giou & box\_count\_threshold & pixelwise & lac & additive & 18.8552 & 0.0124 & 39.4903 & 0.0258 & 65.906 & 0.0401 & 0.0655 \\
giou & box\_count\_threshold & pixelwise & lac & multiplicative & 18.8552 & 0.0124 & 171.0406 & 0.0154 & 65.906 & 0.0401 & 0.0553 \\
giou & box\_count\_threshold & pixelwise & aps & additive & 18.8552 & 0.0124 & 39.4903 & 0.0258 & 65.4388 & 0.0268 & 0.0506 \\
giou & box\_count\_threshold & pixelwise & aps & multiplicative & 18.8552 & 0.0124 & 171.0406 & 0.0154 & 65.4388 & 0.0268 & 0.0417 \\
giou & box\_count\_threshold & boxwise & lac & additive & 18.8552 & 0.0124 & 43.3738 & 0.0198 & 65.906 & 0.0401 & 0.0595 \\
giou & box\_count\_threshold & boxwise & lac & multiplicative & 18.8552 & 0.0124 & 195.5759 & 0.0136 & 65.906 & 0.0401 & 0.0534 \\
giou & box\_count\_threshold & boxwise & aps & additive & 18.8552 & 0.0124 & 43.3738 & 0.0198 & 65.4388 & 0.0268 & 0.0448 \\
giou & box\_count\_threshold & boxwise & aps & multiplicative & 18.8552 & 0.0124 & 195.5759 & 0.0136 & 65.4388 & 0.0268 & 0.0399 \\
giou & box\_count\_recall & thresholded & lac & additive & 11.7104 & 0.0189 & 35.7372 & 0.0092 & 67.048 & 0.0417 & 0.0504 \\
giou & box\_count\_recall & thresholded & aps & additive & 11.7104 & 0.0189 & 35.7372 & 0.0092 & 66.5712 & 0.024 & 0.0325 \\
giou & box\_count\_recall & pixelwise & lac & additive & 11.7104 & 0.0189 & 32.3787 & 0.0192 & 67.0478 & 0.0417 & 0.0602 \\
giou & box\_count\_recall & pixelwise & lac & multiplicative & 11.7104 & 0.0189 & 171.6958 & 0.0073 & 67.0478 & 0.0417 & 0.0486 \\
giou & box\_count\_recall & pixelwise & aps & additive & 11.7104 & 0.0189 & 32.3787 & 0.0192 & 66.5712 & 0.024 & 0.041 \\
giou & box\_count\_recall & pixelwise & aps & multiplicative & 11.7104 & 0.0189 & 171.6958 & 0.0073 & 66.5712 & 0.024 & 0.0306 \\
giou & box\_count\_recall & boxwise & lac & additive & 11.7104 & 0.0189 & 35.5256 & 0.0128 & 67.0478 & 0.0417 & 0.0538 \\
giou & box\_count\_recall & boxwise & lac & multiplicative & 11.7104 & 0.0189 & 196.7056 & 0.0067 & 67.0478 & 0.0417 & 0.048 \\
giou & box\_count\_recall & boxwise & aps & additive & 11.7104 & 0.0189 & 35.5256 & 0.0128 & 66.5712 & 0.024 & 0.0349 \\
giou & box\_count\_recall & boxwise & aps & multiplicative & 11.7104 & 0.0189 & 196.7056 & 0.0067 & 66.5712 & 0.024 & 0.0299 \\
    \bottomrule
    \end{tabular}
    \end{table*}

    \begin{table*}[htbp]
    \centering
    \tiny
    \caption{Experimental Results using the YOLOv8x model for $\alphatot=0.2$}
    \label{tab:yolo_results_alpha_02_styled}
    \begin{tabular}{@{}lllll rrcrrcr@{}}
    \toprule
    \multicolumn{1}{c}{\multirow{2}{*}{\shortstack{ Matching\\Function}}} & 
    \multicolumn{1}{c}{\multirow{2}{*}{\shortstack{ Conf.\\Loss}}} & 
    \multicolumn{1}{c}{\multirow{2}{*}{\shortstack{ Loc.\\Loss}}} & 
    \multicolumn{1}{c}{\multirow{2}{*}{\shortstack{ Cls.\\Pred.\\Set}}} & 
    \multicolumn{1}{c}{\multirow{2}{*}{\shortstack{ Loc.\\Pred.\\Set.}}} & 
    \multicolumn{2}{c}{Confidence} & \multicolumn{2}{c}{Localization} & 
    \multicolumn{2}{c}{Classification} & \multicolumn{1}{c}{Global} \\
    \cmidrule(lr){6-7} \cmidrule(lr){8-9} \cmidrule(lr){10-11} \cmidrule(lr){12-12}
    & & & & & \multicolumn{1}{c}{Size} & \multicolumn{1}{c}{Risk} & 
    \multicolumn{1}{c}{Size} & \multicolumn{1}{c}{Risk} & 
    \multicolumn{1}{c}{Size} & \multicolumn{1}{c}{Risk} & 
    \multicolumn{1}{c}{Risk} \\
    \midrule
    mix & box\_count\_threshold & thresholded & lac & additive & 15.7624 & 0.0212 & 5.6266 & 0.0928 & 57.0939 & 0.1017 & 0.1811 \\
mix & box\_count\_threshold & thresholded & lac & multiplicative & 15.7624 & 0.0212 & 4.8604 & 0.0944 & 57.0939 & 0.1017 & 0.1916 \\
mix & box\_count\_threshold & thresholded & aps & additive & 15.7624 & 0.0212 & 5.6266 & 0.0928 & 22.5433 & 0.0856 & 0.169 \\
mix & box\_count\_threshold & thresholded & aps & multiplicative & 15.7624 & 0.0212 & 4.8604 & 0.0944 & 22.5433 & 0.0856 & 0.1756 \\
mix & box\_count\_threshold & pixelwise & lac & additive & 15.7624 & 0.0212 & 3.2599 & 0.1012 & 57.0939 & 0.1017 & 0.1753 \\
mix & box\_count\_threshold & pixelwise & lac & multiplicative & 15.7624 & 0.0212 & 2.5269 & 0.1009 & 57.0939 & 0.1017 & 0.1824 \\
mix & box\_count\_threshold & pixelwise & aps & additive & 15.7624 & 0.0212 & 3.2599 & 0.1012 & 22.5433 & 0.0856 & 0.1626 \\
mix & box\_count\_threshold & pixelwise & aps & multiplicative & 15.7624 & 0.0212 & 2.5269 & 0.1009 & 22.5433 & 0.0856 & 0.1691 \\
mix & box\_count\_threshold & boxwise & lac & additive & 15.7624 & 0.0212 & 5.4937 & 0.0945 & 57.0939 & 0.1017 & 0.1717 \\
mix & box\_count\_threshold & boxwise & lac & multiplicative & 15.7624 & 0.0212 & 3.8395 & 0.0988 & 57.0939 & 0.1017 & 0.1825 \\
mix & box\_count\_threshold & boxwise & aps & additive & 15.7624 & 0.0212 & 5.4937 & 0.0945 & 22.5433 & 0.0856 & 0.1593 \\
mix & box\_count\_threshold & boxwise & aps & multiplicative & 15.7624 & 0.0212 & 3.8395 & 0.0988 & 22.5433 & 0.0856 & 0.1686 \\
mix & box\_count\_recall & thresholded & lac & additive & 10.1584 & 0.0313 & 4.7828 & 0.0896 & 54.0995 & 0.0994 & 0.1754 \\
mix & box\_count\_recall & thresholded & lac & multiplicative & 10.1584 & 0.0313 & 5.1396 & 0.0936 & 54.0995 & 0.0994 & 0.1883 \\
mix & box\_count\_recall & thresholded & aps & additive & 10.1584 & 0.0313 & 4.7828 & 0.0896 & 24.5079 & 0.0818 & 0.16 \\
mix & box\_count\_recall & thresholded & aps & multiplicative & 10.1584 & 0.0313 & 5.1396 & 0.0936 & 24.5079 & 0.0818 & 0.1695 \\
mix & box\_count\_recall & pixelwise & lac & additive & 10.1584 & 0.0313 & 2.9593 & 0.1007 & 54.0995 & 0.0994 & 0.172 \\
mix & box\_count\_recall & pixelwise & lac & multiplicative & 10.1584 & 0.0313 & 2.6906 & 0.0984 & 54.0995 & 0.0994 & 0.1787 \\
mix & box\_count\_recall & pixelwise & aps & additive & 10.1584 & 0.0313 & 2.9593 & 0.1007 & 24.5079 & 0.0818 & 0.1562 \\
mix & box\_count\_recall & pixelwise & aps & multiplicative & 10.1584 & 0.0313 & 2.6906 & 0.0984 & 24.5079 & 0.0818 & 0.1623 \\
mix & box\_count\_recall & boxwise & lac & additive & 10.1584 & 0.0313 & 4.7128 & 0.0904 & 54.0995 & 0.0994 & 0.1645 \\
mix & box\_count\_recall & boxwise & lac & multiplicative & 10.1584 & 0.0313 & 4.0727 & 0.0942 & 54.0995 & 0.0994 & 0.1767 \\
mix & box\_count\_recall & boxwise & aps & additive & 10.1584 & 0.0313 & 4.7128 & 0.0904 & 24.5079 & 0.0818 & 0.1497 \\
mix & box\_count\_recall & boxwise & aps & multiplicative & 10.1584 & 0.0313 & 4.0727 & 0.0942 & 24.5079 & 0.0818 & 0.1596 \\
hausdorff & box\_count\_threshold & thresholded & lac & additive & 15.7624 & 0.0212 & 5.6266 & 0.0908 & 58.4583 & 0.1018 & 0.1803 \\
hausdorff & box\_count\_threshold & thresholded & lac & multiplicative & 15.7624 & 0.0212 & 5.3222 & 0.0892 & 58.4583 & 0.1018 & 0.1866 \\
hausdorff & box\_count\_threshold & thresholded & aps & additive & 15.7624 & 0.0212 & 5.6266 & 0.0908 & 23.2279 & 0.0871 & 0.1684 \\
hausdorff & box\_count\_threshold & thresholded & aps & multiplicative & 15.7624 & 0.0212 & 5.3222 & 0.0892 & 23.2279 & 0.0871 & 0.1724 \\
hausdorff & box\_count\_threshold & pixelwise & lac & additive & 15.7624 & 0.0212 & 3.2677 & 0.1007 & 58.4583 & 0.1018 & 0.1752 \\
hausdorff & box\_count\_threshold & pixelwise & lac & multiplicative & 15.7624 & 0.0212 & 2.6485 & 0.0966 & 58.4583 & 0.1018 & 0.1797 \\
hausdorff & box\_count\_threshold & pixelwise & aps & additive & 15.7624 & 0.0212 & 3.2677 & 0.1007 & 23.2279 & 0.0871 & 0.1632 \\
hausdorff & box\_count\_threshold & pixelwise & aps & multiplicative & 15.7624 & 0.0212 & 2.6485 & 0.0966 & 23.2279 & 0.0871 & 0.167 \\
hausdorff & box\_count\_threshold & boxwise & lac & additive & 15.7624 & 0.0212 & 5.4858 & 0.0953 & 58.4583 & 0.1018 & 0.1724 \\
hausdorff & box\_count\_threshold & boxwise & lac & multiplicative & 15.7624 & 0.0212 & 4.1312 & 0.0918 & 58.4583 & 0.1018 & 0.1768 \\
hausdorff & box\_count\_threshold & boxwise & aps & additive & 15.7624 & 0.0212 & 5.4858 & 0.0953 & 23.2279 & 0.0871 & 0.1612 \\
hausdorff & box\_count\_threshold & boxwise & aps & multiplicative & 15.7624 & 0.0212 & 4.1312 & 0.0918 & 23.2279 & 0.0871 & 0.164 \\
hausdorff & box\_count\_recall & thresholded & lac & additive & 10.1584 & 0.0313 & 4.7828 & 0.088 & 55.4691 & 0.0992 & 0.1739 \\
hausdorff & box\_count\_recall & thresholded & lac & multiplicative & 10.1584 & 0.0313 & 5.6488 & 0.0824 & 55.4691 & 0.0992 & 0.1777 \\
hausdorff & box\_count\_recall & thresholded & aps & additive & 10.1584 & 0.0313 & 4.7828 & 0.088 & 24.8851 & 0.0835 & 0.16 \\
hausdorff & box\_count\_recall & thresholded & aps & multiplicative & 10.1584 & 0.0313 & 5.6488 & 0.0824 & 24.885 & 0.0835 & 0.1608 \\
hausdorff & box\_count\_recall & pixelwise & lac & additive & 10.1584 & 0.0313 & 2.9721 & 0.0996 & 55.4692 & 0.0992 & 0.171 \\
hausdorff & box\_count\_recall & pixelwise & lac & multiplicative & 10.1584 & 0.0313 & 2.8361 & 0.0924 & 55.4692 & 0.0992 & 0.1739 \\
hausdorff & box\_count\_recall & pixelwise & aps & additive & 10.1584 & 0.0313 & 2.9721 & 0.0996 & 24.8851 & 0.0835 & 0.1566 \\
hausdorff & box\_count\_recall & pixelwise & aps & multiplicative & 10.1584 & 0.0313 & 2.8361 & 0.0924 & 24.885 & 0.0835 & 0.1589 \\
hausdorff & box\_count\_recall & boxwise & lac & additive & 10.1584 & 0.0313 & 4.7128 & 0.0898 & 55.4691 & 0.0992 & 0.1642 \\
hausdorff & box\_count\_recall & boxwise & lac & multiplicative & 10.1584 & 0.0313 & 4.3879 & 0.0845 & 55.4691 & 0.0992 & 0.1685 \\
hausdorff & box\_count\_recall & boxwise & aps & additive & 10.1584 & 0.0313 & 4.7128 & 0.0898 & 24.8851 & 0.0835 & 0.1505 \\
hausdorff & box\_count\_recall & boxwise & aps & multiplicative & 10.1584 & 0.0313 & 4.3879 & 0.0845 & 24.885 & 0.0835 & 0.1532 \\
lac & box\_count\_threshold & thresholded & lac & additive & 15.7624 & 0.0212 & 19.6105 & 0.09 & 0.5361 & 0.106 & 0.1837 \\
lac & box\_count\_threshold & thresholded & lac & multiplicative & 15.7624 & 0.0212 & 13.3677 & 0.1008 & 0.5361 & 0.106 & 0.1881 \\
lac & box\_count\_threshold & thresholded & aps & additive & 15.7624 & 0.0212 & 19.6105 & 0.09 & 0.9956 & 0.0461 & 0.1288 \\
lac & box\_count\_threshold & thresholded & aps & multiplicative & 15.7624 & 0.0212 & 13.3677 & 0.1008 & 0.9956 & 0.0461 & 0.1378 \\
lac & box\_count\_threshold & pixelwise & lac & additive & 15.7624 & 0.0212 & 14.1831 & 0.0913 & 0.5361 & 0.106 & 0.1686 \\
lac & box\_count\_threshold & pixelwise & lac & multiplicative & 15.7624 & 0.0212 & 8.3605 & 0.1051 & 0.5361 & 0.106 & 0.174 \\
lac & box\_count\_threshold & pixelwise & aps & additive & 15.7624 & 0.0212 & 14.1831 & 0.0913 & 0.9956 & 0.0461 & 0.1218 \\
lac & box\_count\_threshold & pixelwise & aps & multiplicative & 15.7624 & 0.0212 & 8.3605 & 0.1051 & 0.9956 & 0.0461 & 0.133 \\
lac & box\_count\_threshold & boxwise & lac & additive & 15.7624 & 0.0212 & 17.3806 & 0.0919 & 0.5361 & 0.106 & 0.1667 \\
lac & box\_count\_threshold & boxwise & lac & multiplicative & 15.7624 & 0.0212 & 9.7946 & 0.1068 & 0.5361 & 0.106 & 0.1728 \\
lac & box\_count\_threshold & boxwise & aps & additive & 15.7624 & 0.0212 & 17.3806 & 0.0919 & 0.9956 & 0.0461 & 0.1203 \\
lac & box\_count\_threshold & boxwise & aps & multiplicative & 15.7624 & 0.0212 & 9.7946 & 0.1068 & 0.9956 & 0.0461 & 0.1324 \\
lac & box\_count\_recall & thresholded & lac & additive & 10.1584 & 0.0313 & 16.6423 & 0.0844 & 0.7161 & 0.106 & 0.1788 \\
lac & box\_count\_recall & thresholded & lac & multiplicative & 10.1584 & 0.0313 & 14.4994 & 0.0904 & 0.7161 & 0.106 & 0.1794 \\
lac & box\_count\_recall & thresholded & aps & additive & 10.1584 & 0.0313 & 16.6423 & 0.0844 & 0.9932 & 0.0611 & 0.1369 \\
lac & box\_count\_recall & thresholded & aps & multiplicative & 10.1584 & 0.0313 & 14.4994 & 0.0904 & 0.9932 & 0.0611 & 0.1414 \\
lac & box\_count\_recall & pixelwise & lac & additive & 10.1584 & 0.0313 & 12.236 & 0.0859 & 0.7161 & 0.106 & 0.1631 \\
lac & box\_count\_recall & pixelwise & lac & multiplicative & 10.1584 & 0.0313 & 8.9708 & 0.0974 & 0.7161 & 0.106 & 0.1675 \\
lac & box\_count\_recall & pixelwise & aps & additive & 10.1584 & 0.0313 & 12.236 & 0.0859 & 0.9932 & 0.0611 & 0.1263 \\
lac & box\_count\_recall & pixelwise & aps & multiplicative & 10.1584 & 0.0313 & 8.9708 & 0.0974 & 0.9932 & 0.0611 & 0.1357 \\
lac & box\_count\_recall & boxwise & lac & additive & 10.1584 & 0.0313 & 14.9561 & 0.0845 & 0.7161 & 0.106 & 0.1608 \\
lac & box\_count\_recall & boxwise & lac & multiplicative & 10.1584 & 0.0313 & 10.5469 & 0.0966 & 0.7161 & 0.106 & 0.1659 \\
lac & box\_count\_recall & boxwise & aps & additive & 10.1584 & 0.0313 & 14.9561 & 0.0845 & 0.9932 & 0.0611 & 0.1238 \\
lac & box\_count\_recall & boxwise & aps & multiplicative & 10.1584 & 0.0313 & 10.5469 & 0.0966 & 0.9932 & 0.0611 & 0.1338 \\
giou & box\_count\_threshold & thresholded & lac & additive & 15.7624 & 0.0212 & 39.4871 & 0.0348 & 57.1519 & 0.0891 & 0.123 \\
giou & box\_count\_threshold & thresholded & lac & multiplicative & 15.7624 & 0.0212 & 179.9897 & 0.0196 & 57.1519 & 0.0891 & 0.1084 \\
giou & box\_count\_threshold & thresholded & aps & additive & 15.7624 & 0.0212 & 39.4871 & 0.0348 & 57.7448 & 0.0459 & 0.0792 \\
giou & box\_count\_threshold & thresholded & aps & multiplicative & 15.7624 & 0.0212 & 179.9897 & 0.0196 & 57.7448 & 0.0459 & 0.0651 \\
giou & box\_count\_threshold & pixelwise & lac & additive & 15.7624 & 0.0212 & 34.0626 & 0.0495 & 57.1519 & 0.0891 & 0.1346 \\
giou & box\_count\_threshold & pixelwise & lac & multiplicative & 15.7624 & 0.0212 & 140.1268 & 0.0222 & 57.1519 & 0.0891 & 0.11 \\
giou & box\_count\_threshold & pixelwise & aps & additive & 15.7624 & 0.0212 & 34.0626 & 0.0495 & 57.7448 & 0.0459 & 0.0896 \\
giou & box\_count\_threshold & pixelwise & aps & multiplicative & 15.7624 & 0.0212 & 140.1268 & 0.0222 & 57.7448 & 0.0459 & 0.0667 \\
giou & box\_count\_threshold & boxwise & lac & additive & 15.7624 & 0.0212 & 38.7389 & 0.0397 & 57.1519 & 0.0891 & 0.1257 \\
giou & box\_count\_threshold & boxwise & lac & multiplicative & 15.7624 & 0.0212 & 161.3222 & 0.0192 & 57.1519 & 0.0891 & 0.1072 \\
giou & box\_count\_threshold & boxwise & aps & additive & 15.7624 & 0.0212 & 38.7389 & 0.0397 & 57.7448 & 0.0459 & 0.0803 \\
giou & box\_count\_threshold & boxwise & aps & multiplicative & 15.7624 & 0.0212 & 161.3222 & 0.0192 & 57.7448 & 0.0459 & 0.0639 \\
giou & box\_count\_recall & thresholded & lac & additive & 10.1584 & 0.0313 & 32.3011 & 0.022 & 58.219 & 0.0858 & 0.1066 \\
giou & box\_count\_recall & thresholded & lac & multiplicative & 10.1584 & 0.0313 & 179.9438 & 0.0108 & 58.219 & 0.0858 & 0.096 \\
giou & box\_count\_recall & thresholded & aps & additive & 10.1584 & 0.0313 & 32.3011 & 0.022 & 58.5988 & 0.0432 & 0.064 \\
giou & box\_count\_recall & thresholded & aps & multiplicative & 10.1584 & 0.0313 & 179.9438 & 0.0108 & 58.5988 & 0.0432 & 0.0534 \\
giou & box\_count\_recall & pixelwise & lac & additive & 10.1584 & 0.0313 & 27.9418 & 0.0375 & 58.219 & 0.0858 & 0.121 \\
giou & box\_count\_recall & pixelwise & lac & multiplicative & 10.1584 & 0.0313 & 140.4195 & 0.0123 & 58.219 & 0.0858 & 0.0974 \\
giou & box\_count\_recall & pixelwise & aps & additive & 10.1584 & 0.0313 & 27.9418 & 0.0375 & 58.5988 & 0.0432 & 0.0759 \\
giou & box\_count\_recall & pixelwise & aps & multiplicative & 10.1584 & 0.0313 & 140.4195 & 0.0123 & 58.5988 & 0.0432 & 0.0546 \\
giou & box\_count\_recall & boxwise & lac & additive & 10.1584 & 0.0313 & 31.7182 & 0.0274 & 58.2189 & 0.0858 & 0.1108 \\
giou & box\_count\_recall & boxwise & lac & multiplicative & 10.1584 & 0.0313 & 162.1215 & 0.0113 & 58.2189 & 0.0858 & 0.0964 \\
giou & box\_count\_recall & boxwise & aps & additive & 10.1584 & 0.0313 & 31.7182 & 0.0274 & 58.5988 & 0.0432 & 0.0665 \\
giou & box\_count\_recall & boxwise & aps & multiplicative & 10.1584 & 0.0313 & 162.1215 & 0.0113 & 58.5988 & 0.0432 & 0.0536 \\
    \bottomrule
    \end{tabular}
    \end{table*}
     
\end{document}